\documentclass{article}

\usepackage{microtype}
\usepackage{graphicx}
\usepackage{subcaption}  
\usepackage{booktabs} 

\usepackage{hyperref}

\usepackage{algorithm}
\usepackage{algorithmic}
\usepackage{algpseudocode}
\usepackage{amsmath,amsthm}

\usepackage[preprint]{icml2026}

\usepackage{amsmath}
\usepackage{amssymb}
\usepackage{mathtools}
\usepackage{amsthm}

\usepackage[capitalize,noabbrev]{cleveref}

\usepackage[utf8]{inputenc} 
\usepackage[T1]{fontenc} 
\usepackage{url} 
\usepackage{amsfonts} 
\usepackage{nicefrac} 
\usepackage{xcolor} 
\usepackage{mathrsfs}
\usepackage{enumerate}
\usepackage{wrapfig}
\usepackage{float}
\usepackage{bbding}
\usepackage{multirow}
\usepackage[misc]{ifsym}
\usepackage{placeins}
\usepackage{import}
\usepackage{tikz}
\usetikzlibrary{graphs}
\usetikzlibrary{positioning}

\theoremstyle{plain}

\newtheorem{proposition}{Proposition}

\theoremstyle{definition}
\newtheorem{definition}{Definition}
\theoremstyle{remark}

\numberwithin{theorem}{section}
\numberwithin{proposition}{section}
\numberwithin{lemma}{section}
\numberwithin{corollary}{section}
\numberwithin{definition}{section}
\numberwithin{remark}{section}

\newcommand{\CommentGray}[1]{\textcolor[gray]{0.6}{\Comment{#1}}}

\usepackage{amsmath}
\usepackage{amsfonts}
\usepackage{bm}
\usepackage{esint}

\def\eqref#1{equation~\ref{#1}}

\def\1{\bm{1}}

\def\rvb{{\mathbf{b}}}
\def\rvc{{\mathbf{c}}}

\def\rvh{{\mathbf{h}}}

\def\rvv{{\mathbf{v}}}
\def\rvw{{\mathbf{w}}}
\def\rvx{{\mathbf{x}}}
\def\rvy{{\mathbf{y}}}

\def\vtheta{{\bm{\theta}}}

\def\vz{{\bm{z}}}

\def\mA{{\bm{A}}}
\def\mB{{\bm{B}}}

\def\mH{{\bm{H}}}
\def\mI{{\bm{I}}}

\def\mK{{\bm{K}}}

\def\mQ{{\bm{Q}}}

\def\mV{{\bm{V}}}
\def\mW{{\bm{W}}}
\def\mX{{\bm{X}}}

\DeclareMathAlphabet{\mathsfit}{\encodingdefault}{\sfdefault}{m}{sl}
\SetMathAlphabet{\mathsfit}{bold}{\encodingdefault}{\sfdefault}{bx}{n}

\newcommand{\Eqn}{Eqn.~}

\usepackage{amsmath,amsfonts,bm}

\def\eqref#1{equation~\ref{#1}}

\def\1{\bm{1}}

\def\rvb{{\mathbf{b}}}
\def\rvc{{\mathbf{c}}}

\def\rvh{{\mathbf{h}}}

\def\rvv{{\mathbf{v}}}
\def\rvw{{\mathbf{w}}}
\def\rvx{{\mathbf{x}}}
\def\rvy{{\mathbf{y}}}

\def\vtheta{{\bm{\theta}}}

\def\vz{{\bm{z}}}

\def\mA{{\bm{A}}}
\def\mB{{\bm{B}}}

\def\mH{{\bm{H}}}
\def\mI{{\bm{I}}}

\def\mK{{\bm{K}}}

\def\mQ{{\bm{Q}}}

\def\mV{{\bm{V}}}
\def\mW{{\bm{W}}}
\def\mX{{\bm{X}}}

\DeclareMathAlphabet{\mathsfit}{\encodingdefault}{\sfdefault}{m}{sl}
\SetMathAlphabet{\mathsfit}{bold}{\encodingdefault}{\sfdefault}{bx}{n}

\newcommand{\R}{\mathbb{R}}

\icmltitlerunning{Enjoy Your Layer Normalization with the Computational Efficiency of RMS\-Norm}

\begin{document}
	
	\twocolumn[
	\icmltitle{Enjoy Your Layer Normalization \\ with the Computational Efficiency of RMS\-Norm}
	
	
	
	\icmlsetsymbol{equal}{*}
	
	\begin{icmlauthorlist}

        \icmlauthor{Yuxin Guo}{inst1}
		\icmlauthor{Yihao Yue}{inst1}
        \icmlauthor{Yunhao Ni}{inst1}
		\icmlauthor{Yizhou Ruan}{inst1}
        \icmlauthor{Jie Luo}{inst1}
		\icmlauthor{Wenjun Wu}{inst1,inst2}
		\icmlauthor{Lei Huang~$^\textrm{\Letter}$}{inst1,inst2}
	\end{icmlauthorlist}
	
	\icmlaffiliation{inst1}{State Key Laboratory of Complex and Critical Software Environment (SKLCCSE),  Beihang University, Beijing, China}
	\icmlaffiliation{inst2}{Hangzhou International Innovation Institute, Beihang University, Hangzhou, China}
	
	\icmlcorrespondingauthor{Lei Huang}{huangleiAI@buaa.edu.cn}
	
	\icmlkeywords{Machine Learning, ICML, Layer Normalization, Universal Approximation}
	\vskip 0.3in
	]
	


\printAffiliationsAndNotice{}  
	
\begin{abstract}
Layer normalization (LN) is a fundamental component in modern deep learning, but its per-sample centering and scaling introduce non-negligible inference overhead. RMSNorm improves efficiency by removing the centering operation, yet this may discard benefits associated with centering. This paper propose a framework to determine whether an LN in an arbitrary DNN can be replaced by RMSNorm without changing the model function. The key idea is to fold LN’s centering operation into upstream general linear layers by enforcing zero-mean outputs through the column-centered constraint (CCC) and column-based weight centering (CBWC). We extend the analysis to arbitrary DNNs, define such LNs as foldable LNs, and develop a graph-based detection algorithm. Our analysis shows that many LNs in widely used architectures are foldable, enabling exact inference-time conversion and end-to-end acceleration of 2\% to 12\% without changing model predictions. Experiments across multiple task families further show that, when exact equivalence is partially broken in practical training settings, our method remains competitive with vanilla LN while improving efficiency.
\end{abstract}

\section{Introduction}
Normalization techniques are widely used in deep neural networks (DNNs) to stabilize and accelerate training~\citep{2023_TPAMI_Huang}. As a seminal work, Batch Normalization (BN)~\citep{2015_ICML_Ioffe} improves the training stability and the optimization efficiency of DNN by standardizing (centering and scaling) the activations of intermediate DNN layers within a mini-batch of data during training. During inference, BN uses population statistics calculated in training for normalization, and this operation can be folded into upstream linear layers~\citep{2018_CVPR_Jacob}, avoiding additional computational cost. Despite many merits, BN suffers from the train-inference inconsistent problem, leading to significantly degenerate performance in the scenarios of small batch sizes training and domain shifted distributions~\citep{2023_TPAMI_Huang}. 

Layer Normalization (LN)~\citep{2016_LN_Ba} addresses BN's train-inference inconsistency by standardizing the inputs of layers across neurons for each individual sample. It has become a key component of Transformer~\citep{2017_NIPS_Vaswani} and its variants~\citep{2019_ACL_Dai,2020_ICML_Xiong,2021_ICLR_Dosovitskiy}, spreading from the Natural Language Processing (NLP)~\citep{Radford2018GPT1,2019_ACL_Devlin,2020_JMLR_Raffel} to Computer Vision (CV)~\citep{2021_ICLR_Dosovitskiy,2020_ECCV_Carion,2022_CVPR_Cheng} communities. LN plays a foundational role~\citep{2023_TPAMI_Huang} in the evolution of neural architectures and is a standard component in many foundation models~\citep{Brown2020GPT3,2022_NIPS_Flamingo,2023_ICCV_Kirillov}. However, during inference, unlike BN and other normalization-free counterparts, LN must perform per-sample normalization for each sample which introduces non-negligible additional computational cost.

To alleviate the computational burden of LN, RMS\-Norm \citep{2019_NIPS_Zhang} was proposed as a scaling-only normalization method. According to the experiments in~\citep{2019_NIPS_Zhang}, RMS\-Norm is reported to reduce the running time of LN by 7\%--64\% across different models. Owing to its great potential in practice for computational efficiency, RMS\-Norm has been widely adopted in various architectures~\citep{2024_Tinyllama_Zhang,2024_Gemma,2024_OpenELM}. However, by removing the centering operation, RMS\-Norm may lose some of the conditioning-related benefits commonly associated with centering in prior work~\citep{1990_NIPS_LeCun,1998_Schraudolph,2012_NN_Gregoire,2016_LN_Ba,2017_ICCV_Huang}. This raises the question of how we can retain the theoretical advantages of LN while enjoying the computational cost of RMS\-Norm with mathematical equivalence. 

A previous study~\citep{jiang2023pre} showed that LN can be reduced to RMS\-Norm in pre-norm Transformers by removing redundant mean information in the main branch, which arises from preprocessing and residual connections. However, their analysis and resulting modification are restricted to the specific structures of pre-norm Transformers, leaving open the question of whether a more general framework can be developed for arbitrary DNNs.

In this work, we propose a general framework for determining whether an LN layer can be replaced with RMS\-Norm without changing model predictions at inference time or optimization dynamics during training. We first study the case where LN follows a linear layer. In this setting, we formally characterize when the centering operation of an LN becomes redundant and can therefore be removed. To this end, we introduce the column-centered constraint (CCC) (Def.~\ref{def:CCC}) on the immediately preceding linear layer, enforced through column-based weight centering (CBWC) (Def.~\ref{def:cb}), a reparameterization of the weight matrix. We show that the CBWC+RMS\-Norm scheme is mathematically equivalent to the vanilla model in both inference and training.

We then extend the analysis to general neural networks by introducing the notion of a foldable LN (Def.~\ref{def:foldable}) and its associated zero-mean graph (Def.~\ref{def:zero-mean-graph}). These concepts characterize, from a graph perspective, the structural conditions under which the centering operation of LN can be folded into upstream general linear layers, thereby allowing LN to be replaced by RMSNorm. Specifically, an LN is foldable when all leaf nodes in its zero-mean graph correspond to general linear layers or layers that already guarantee zero-mean output, so that the zero-mean property required by LN can be enforced through CBWC on the corresponding upstream layers. This analysis also clarifies practical cases in which the centering operation of some LNs can be removed through auxiliary centering rather than through CBWC.

Finally, we present an algorithm for detecting foldable LNs and identifying the corresponding upstream layers requiring CBWC in arbitrary DNNs. Our analysis shows that most LNs in widely used architectures are foldable, enabling straightforward inference acceleration and yielding end-to-end inference acceleration of 2\% to 12\%. In addition, experiments across multiple task families show that even when exact folding conditions do not fully hold in practical training settings, CBWC+RMSNorm achieves performance comparable to LN while improving efficiency especially in long-sequence tasks. Our code is available at \url{https://github.com/BobYue-01/Enjoy-LN}.

\section{Notation and Preliminaries}
\label{sec:notation}
We use $x \in \mathbb{R}$, $\rvx \in \mathbb{R}^n$, and $\mX \in \mathbb{R}^{m \times n}$ to denote a scalar, a vector, and a matrix, respectively, where $\mathbb{R}$ is the set of real numbers and $m,n \in \mathbb{N}_{+}$. We use $\boldsymbol{1}_n$ to denote the $n$-dimensional all-one column vector.

\paragraph{\textbf{Neural Network.}} 
We model a neural network as a function $f(\rvx; \vtheta)$, where $\rvx$ is the input and $\vtheta \in \Theta$ denotes the set of all learnable parameters, partitioned layer-wise as $\vtheta = \{ \vtheta^{(k)} \}_{k=1}^L$. For an $L$-layer MLP, the forward pass consists of alternating linear transformations,\footnote{Following deep learning conventions, we do not distinguish between linear and affine transformations. The bias term is omitted here for simplicity; its inclusion is discussed in Appendix~\ref{apx:bias}.} normalization, and activation:
\begin{eqnarray}
 		\label{eqn:Linear}
 		\rvh^{(k)}&=&\mW(\vtheta^{(k)})\  \rvx^{(k-1)},   \\
        \widehat{\rvh}^{(k)}&=&\mathrm{Norm}(\rvh^{(k)}) \\
 		\label{eqn:Nonlinear}
 		\rvx^{(k)}&=&\phi(\widehat\rvh^{(k)}), \quad k=1,\dots, L,
\end{eqnarray}
where $\rvx^{(0)} = \rvx$ is the input, and $\mW(\vtheta^{(k)})\in \mathbb R^{m\times n}$ is the weight matrix of layer $k$.\footnote{For convenience of notation, we simplify $\vtheta^{(k)}$ to $\vtheta_k$ in this paper.}

\paragraph{\textbf{Layer Normalization.}} Layer Normalization (LN) is a fundamental component in modern deep neural networks. Given a layer input vector $\rvh = [h_1, h_2, \dots, h_n]^\top \in \mathbb{R}^n$, LN normalizes $\rvh$ across its $n$ elements (neurons) through centering and scaling: \footnote{In practice, LN typically includes a learnable affine transformation after normalization, which is omitted here for simplicity. The term $\epsilon$ prevents division by zero.}
{
\setlength{\abovedisplayskip}{1ex} 
\setlength{\belowdisplayskip}{0ex} 
\begin{align}
    \label{eqn:centering}
  \mathrm{Centering:} \quad &\widetilde{h}_j = h_j - \mu, \quad j = 1, 2, \dots, n,\\
  \label{eqn:scaling}
  \mathrm{Scaling:} \quad &\widehat{h}_j = \frac{\widetilde{h}_j}{\sqrt{\sigma^2 + \epsilon}}, \quad j = 1, 2, \dots, n,
\end{align}
}
where $\mu=\frac{1}{n}\sum_{j=1}^{n} h_j$ is the mean of $\rvh$ and $\sigma^2 = \frac{1}{n} \sum_{j=1}^{n} \widetilde{h}_j^2$ is the second moment of the centered vector $\widetilde{\rvh}=[\widetilde{h}_1,\widetilde{h}_2,\cdots,\widetilde{h}_n]^\top$. \textit{Centering} enforces zero mean across elements, and the subsequent \textit{scaling} step ensures unit second moment, corresponding to the normalization retained in RMS\-Norm.

By definition, RMS\-Norm is computationally simpler than LN because it removes the centering step. However, although RMS\-Norm enjoys faster speed, it does not explicitly constrain the sample mean, which can lead to a less stable output range and larger variation in activations across layers (see Section~\ref{exp:ln} and Appendix~\ref{apx:ablation} for details).

In this paper, we aim to reduce the computational overhead of LN by replacing it with RMS\-Norm. However, a direct substitution may disrupt training dynamics and degrade performance. In Section~\ref{sec:ln-after-linear}, we present a mathematically equivalent method for replacing LN with RMS\-Norm when LN directly follows a linear layer. In Section~\ref{sec:framework}, we extend this analysis to arbitrary DNNs and propose an algorithm for automatically detecting and replacing foldable LNs. Finally, in Section~\ref{sec:training}, we conduct comprehensive experiments to validate the effectiveness and efficiency of the proposed approach.

\section{Equivalent Replacement of LN after Linear Layers}
\label{sec:ln-after-linear}

\begin{figure}[t]
    \centering
    \includegraphics[height=30ex]{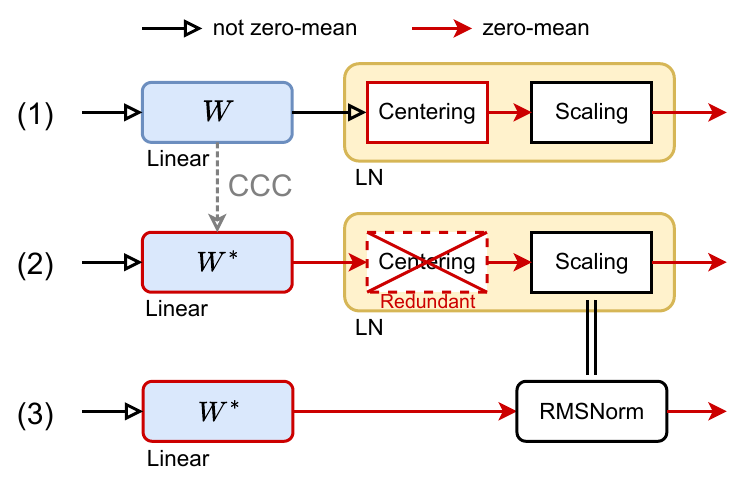}
    \caption{Overview of the method. $W$ denotes the weight matrix of a general linear layer, and $W^*$ is applied with CBWC which satisfies CCC.}
    \label{fig:overview}
    \vspace{-2ex}
\end{figure}

In this section, we present a theoretical analysis of the case where LN is applied directly after a linear layer, and show how it can be replaced with RMS\-Norm in a mathematically equivalent way. We introduce the \textit{column-centered constraint} (CCC), which enforces zero-mean output for the linear layer, as shown in Figure~\ref{fig:overview}. We further propose the \textit{column-based weight centering} (CBWC), a reparameterization that ensures the linear layer satisfies CCC during training. As a result, the CBWC+RMS\-Norm scheme is mathematically equivalent to the original LN-based model.

\subsection{Column-Centered Constraint}
\label{sec:ccc}
We begin with the common case in which LN is applied immediately after a linear layer. Intuitively, the centering step of LN becomes redundant if the output of the linear layer has zero mean across neurons. Since this zero-mean property must hold for all samples, it can be enforced directly through a suitable constraint on the parameters of the linear layer.

We write the linear layer as a linear transformation $\rvh(\rvx; \vtheta_k) = \mW(\vtheta_k)\; \rvx$, and denote the columns of $\mW(\vtheta_k)$ by $\rvw_i(\vtheta_k) \in \mathbb{R}^m$, i.e., $\mW(\vtheta_k) = \big[\rvw_1(\vtheta_k), \rvw_2(\vtheta_k), \dots, \rvw_n(\vtheta_k)\big]$. Under this notation, we define the \textit{column-centered constraint} as follows.

\begin{definition}[Column-Centered Constraint (CCC)]
    \label{def:CCC} 
    The parameter $\vtheta_k\in \Theta^{(k)}$ of a linear layer satisfies the \emph{column-centered constraint} if:
    {
    \setlength{\abovedisplayskip}{1ex}  
    \setlength{\belowdisplayskip}{1ex} 
    \begin{equation}
    \label{eqn:ccc-general}
        \vtheta_k \in\Theta^* = \left\{\vtheta : \rvw_i(\vtheta)^\top \boldsymbol{1}_m =0,\ i=1,2,\dots,n\ \right\}\subseteq \Theta^{(k)},
    \end{equation}
    }
    i.e., the elements in each column vector $\rvw_i(\vtheta_k)$ sum to zero. 
\end{definition}

\begin{proposition}[Zero-mean Output under CCC]
    \label{prop:zero-mean-ccc}
    If the parameters of a linear layer $\vtheta_k$ satisfy CCC, then for any input $\rvx \in \R^n$, the output 
    $\rvh = \mW(\vtheta_k)\; \rvx$
    has zero mean across its elements, i.e., 
    $\mu_h = \frac{1}{m} \boldsymbol{1}_m^\top \rvh = 0.$
\end{proposition}

Prop.~\ref{prop:zero-mean-ccc} follows directly from Def.~\ref{def:CCC}; the full proof is given in Appendix~\ref{apx:proof-ccc}. Therefore, enforcing CCC on the linear layer folds the parameter space $\Theta^{(k)}$ into a subspace $\Theta^*$. Accordingly, the linear layer effectively performs the centering step in advance, making the subsequent use of LN or RMS\-Norm equivalent.

\paragraph{General Linear Layers.}
To go beyond standard linear layers, we regard any layer that applies a linear transformation to its input as a general linear layer. Typical examples include recurrent layers with shared weights in RNNs and convolution layers in CNNs. Their parameters can be written in the unified form of \Eqn\ref{eqn:Linear}, and the corresponding CCCs are derived in Appendix~\ref{apx:proof-of-ccc}.

\subsection{Column-Based Weight Centering}
\label{sec:cbwc}

To enforce CCC during training, we introduce a reparameterization technique \footnote{Reparameterization refers to transforming a model's parameter space to enforce desired constraints while preserving model capacity~\cite{Nowlan_Oja_Amari}.} called \textit{column-based weight centering}, illustrated in Figure~\ref{fig:cbwc-sketch} and formally defined below.

\begin{definition}[Column-Based Weight Centering (CBWC)]
\label{def:cb}
    \emph{Column-based weight centering} introduces a proxy parameter matrix $\mW$ by the effective weight matrix $\mV$ via the transformation
    {
    \setlength{\abovedisplayskip}{1ex}  
    \setlength{\belowdisplayskip}{1ex} 
    \begin{equation}
    \label{eqn:cbwc-fw}
        \mV = \varphi(\mW) = \left( \mI - \frac{1}{m} \boldsymbol{1}_m \boldsymbol{1}_m^\top \right) \mW,
    \end{equation}
    }
    where $m$ is the number of output neurons. During backpropagation, the gradient with respect to $W$ is given by
    {
    \setlength{\abovedisplayskip}{1ex}
    \setlength{\belowdisplayskip}{1ex} 
    \begin{equation}
        \frac{\partial \mathcal{L}}{\partial \mW} = \phi \left(\frac{\partial \mathcal{L}}{\partial \mV}\right) = \left(\mI - \frac{1}{m} \boldsymbol{1}_m \boldsymbol{1}_m^\top \right)^\top \frac{\partial \mathcal{L}}{\partial \mV}.
    \end{equation}}
\end{definition}

Note that CCC can be enforced in many ways. For example, the degenerate subspace $\Theta^* = \{\boldsymbol{0}\}$ satisfies Eqn.~\ref{eqn:ccc-general},  but it destroys the representation capacity of the layer. By contrast, CBWC is a structured reparameterization that enforces CCC while preserving mathematical equivalence with the original model.

\begin{proposition}[Equivalent Optimization Process]
    \label{prop:opt}
    Consider a model containing a general linear layer followed by LN. Replacing that general linear layer with its CBWC parameterization and replacing LN with RMS\-Norm yields an optimization process equivalent to that of the original model.
\end{proposition}

\begin{figure}[t]
    \vspace{-0.05in}
    \centering
    \begin{subfigure}[t]{0.18\textwidth}
        \centering
        \includegraphics[height=24ex]{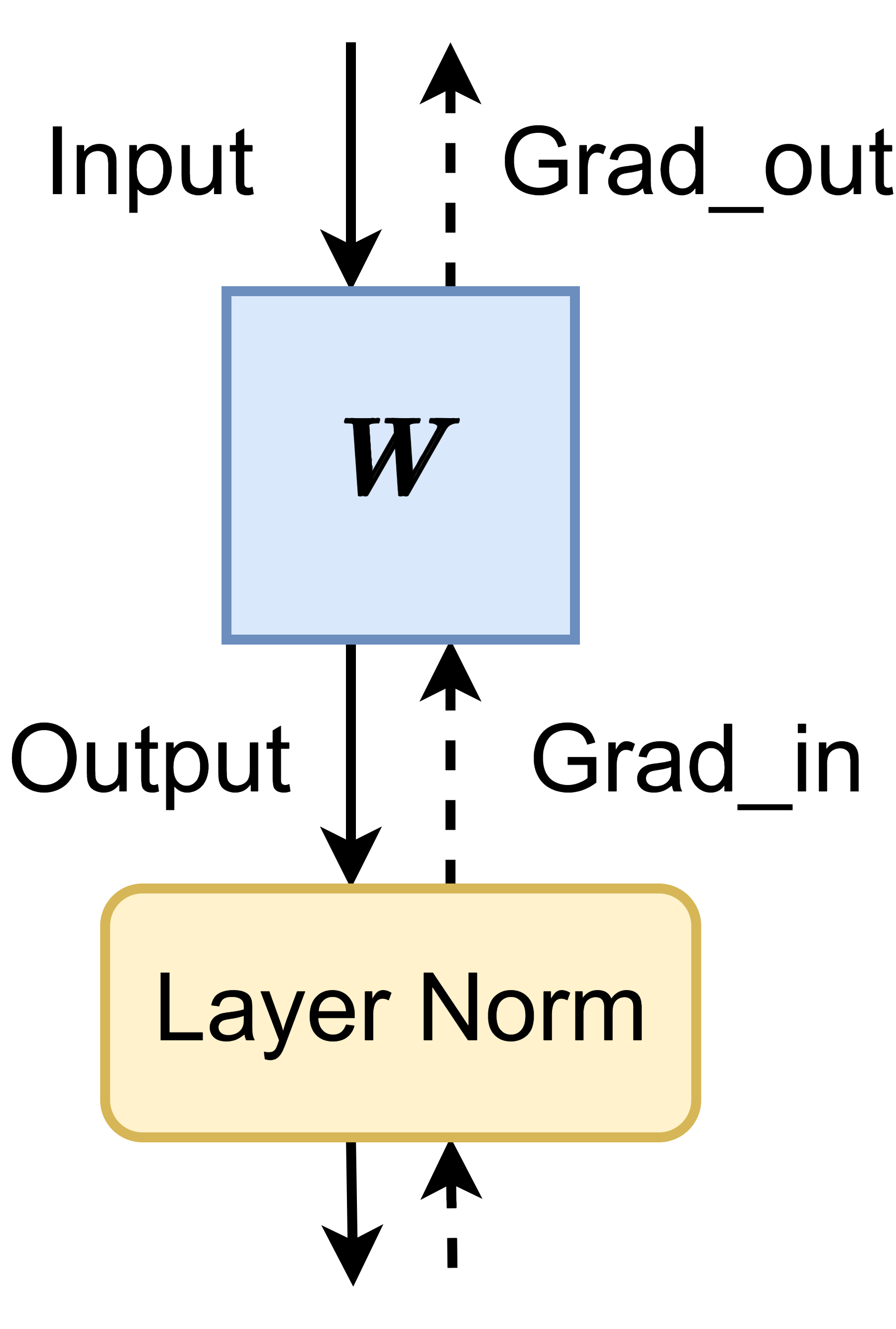}
        \caption{Origin.}
    \end{subfigure}
    \hspace{0.05in}	
    \begin{subfigure}[t]{0.23\textwidth}
        \centering
        \includegraphics[height=24ex]{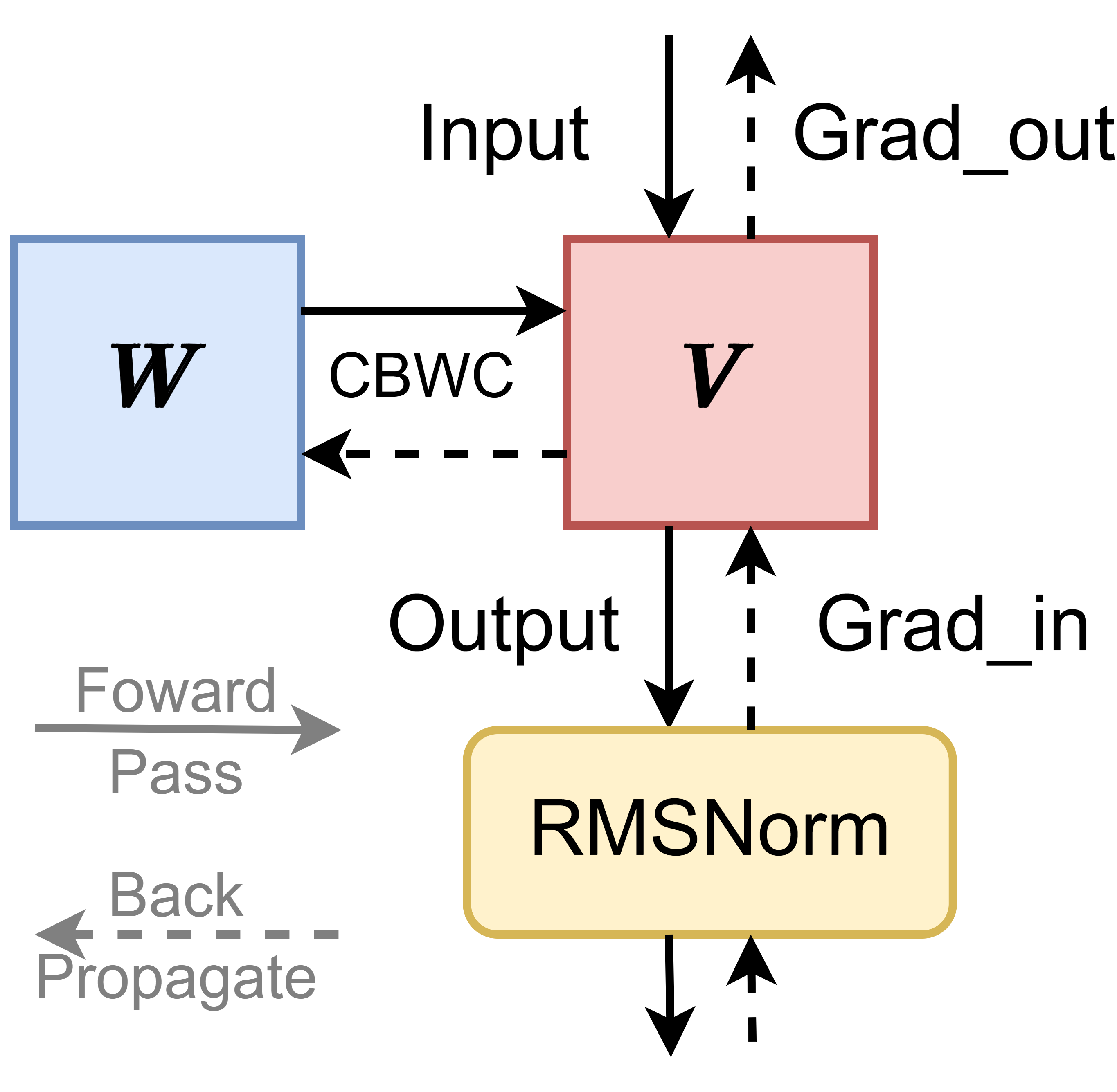}
        \caption{CBWC+RMS.}
    \end{subfigure}
    
    \caption{Sketch map of the two training scheme.}
            \label{fig:cbwc-sketch}
    \vspace{-0.2in}
\end{figure}

The proof of Prop.~\ref{prop:opt} is given in Appendix~\ref{apx:proof-equivalent-opt}. Once CBWC is applied to the upstream general linear layer, RMS\-Norm can replace the LN while producing identical outputs and gradients, yet at lower computational cost. The resulting CBWC+RMS\-Norm scheme therefore yields parameters that are mathematically equivalent to those of the original LN-based model, which we further discuss in Section~\ref{sec:fine-tune}.

\paragraph{Compatibility with Optimization Techniques.}
From the viewpoint of the optimizer, learning-rate scheduler, gradient clipping, and weight-decay implementation applied on the weight matrix, CBWC+RMS\-Norm is operationally indistinguishable from the original LN-based model.

\paragraph{Extension to Group Normalization.}
CCC and CBWC on general linear layers can be extended to Group Normalization (GN)~\citep{Wu_2018_ECCV}, which subsumes LN. The resulting grouped-CCC and grouped-CBWC are derived in Appendix~\ref{apx:gccc}.

\section{A Framework for Folding LN in Arbitrary DNNs}
\label{sec:framework}
In this section, we extend the analysis from Section~\ref{sec:ln-after-linear} to arbitrary DNNs and develop a general framework for determining whether an LN can be replaced with RMS\-Norm. We first introduce the notion of a foldable LN and analyze the structural conditions for LN folding in simple but common architectures, including sequential models and models with parallel connections. We then generalize this analysis to arbitrary neural networks from a graph perspective. Finally, we propose an algorithm for automatically detecting foldable LNs and identifying the corresponding upstream general linear layers requiring CBWC.

\subsection{Foldable LN in Simple but Common Cases}
\label{sec:foldable-ln}

We begin by studying when an LN in a DNN can be replaced with RMS\-Norm. Intuitively, an LN is foldable if replacing it with RMS\-Norm yields identical outputs for all inputs under an appropriate parameter constraint. We formalize this notion as follows.
\begin{definition}[Foldable LN]
\label{def:foldable}
Let $f(\cdot;\vtheta)$ denote a subnetwork that directly precedes an LN layer, and let $\Theta$ be the parameter space of this subnetwork. For a non-empty subset $\Theta^* \subseteq \Theta$, we say that this LN is \emph{foldable} with respect to $\Theta^*$ if, for all $\vtheta^* \in \Theta^*$ and all inputs $\rvx \in \R^n$, it holds that
\begin{equation}
    \mathrm{RMSNorm}(f(\rvx;\vtheta^*)) = \mathrm{LN}(f(\rvx;\vtheta^*)).
\end{equation}
\end{definition}
Although Def.~\ref{def:foldable} is stated for an arbitrary subset $\Theta^* \subseteq \Theta$, in the following analysis we focus on the case where $\Theta^*$ is induced by CBWC. This is because Prop.~\ref{prop:opt} establishes the mathematical equivalence between CBWC+RMS\-Norm and the original LN model in both inference and training. Therefore, restricting $\Theta^*$ allows us to keep the analysis concrete.

Extending the analysis from the simple setting in Section~\ref{sec:ln-after-linear}, whether an LN is foldable is determined by the structure of the model, namely, whether the input to the LN is produced through layers that preserve or enforce the zero-mean property. Therefore, to determine whether an LN is foldable, we backtrack the data flow path and analyze the layers that contribute to its input.

\paragraph{Sequential models.} 
Take MLPs as an example. In this case, each LN is foldable because it directly follows a linear layer, independently of previous layers or input activations. In more complex settings, additional layers may lie between the general linear layer and the target LN. Notably, layer involving only scalar operations (i.e., $\rvx\mapsto a\rvx$ for some $a\in\R$) preserve the zero-mean property of the sample. These include scale layers (e.g., dropout layer in inference mode) and constant multiplication layers (e.g., temperature scaling). Therefore, the centering operation in LN can still be folded into the parameters of the preceding linear transformation, even when the scalar operations are interspersed between the general linear layer and the LN. 

\paragraph{Parallel Connections.}
Modern neural networks often contain more complex connections, including parallel branches. Among them, the residual connection sums multiple branches. To eliminate the mean of the output of a residual connection, it is sufficient to eliminate the mean of each branch separately by linearity. Hence, the downstream LN can be folded by applying CBWC to the general linear layers on each branch. In contrast, other parallel connections, such as concatenation, break the required mathematical equivalence and therefore prevent the downstream LN from being foldable. A more detailed proof is provided in Appendix~\ref{proof:res}.

To systematically determine whether an LN is foldable in an arbitrary neural network, we next develop a graph-based analysis of how the original parameter space $\Theta$ can be effectively folded into a target parameter space $\Theta^*$ through CBWC.

\subsection{Analysis of Neural Networks from a Graph Perspective}
\label{sec:graph}

\begin{figure}[t]
    \centering
    \begin{subfigure}[t]{0.23\textwidth}
        \centering
        \includegraphics[height=7ex]{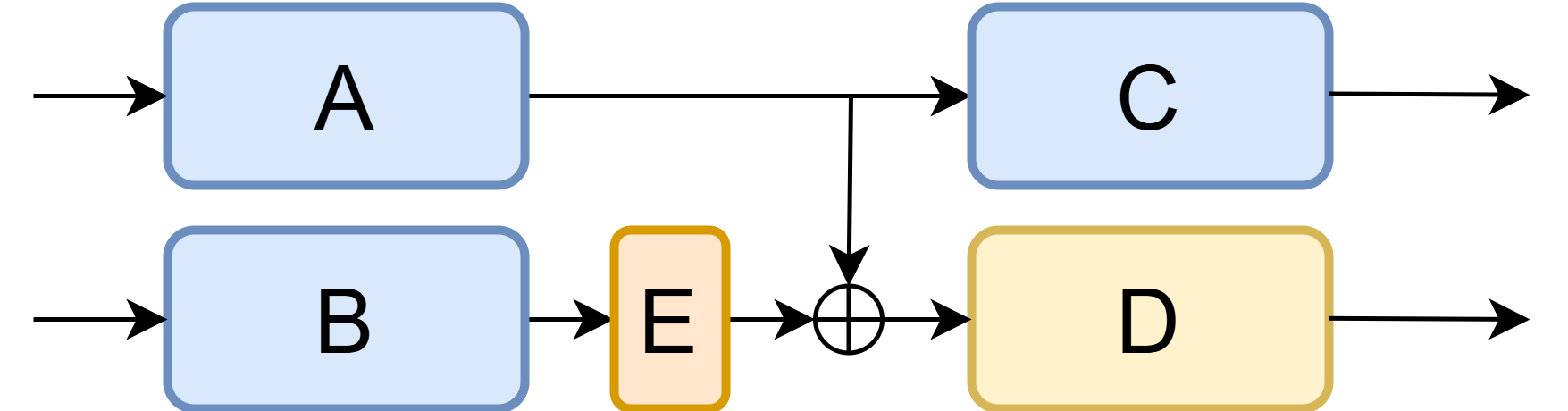}
        \caption{Original neural network.}
    \end{subfigure}
    \hspace{0.05in}	
    \begin{subfigure}[t]{0.23\textwidth}
        \centering
        \includegraphics[height=7ex]{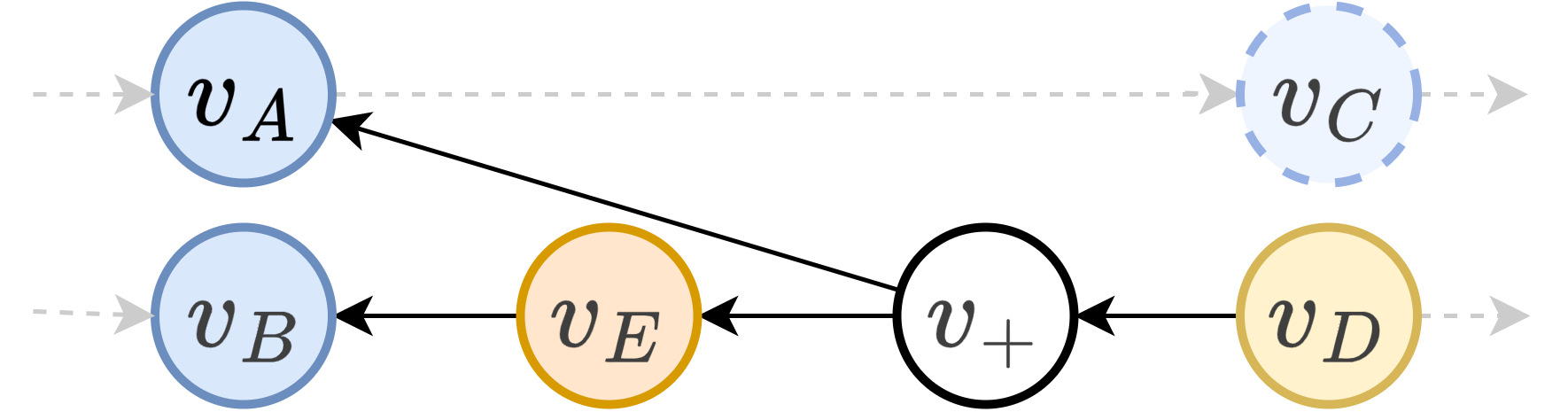}
        \caption{Zero-mean graph.}
    \end{subfigure}
    
    \caption{An example network and its associated zero-mean graph.  
    Layer D is an LN, layer E is a scalar operation, and $\oplus$ denotes residual addition. Layer D is foldable if and only if $v_A, v_B \in V_l$.}
    \label{fig:net-to-graph}
    \vspace{-0.16in}
\end{figure}

We model the network as a directed graph $G=(V,E)$. Each vertex $v\in V$ corresponds to a layer, represented as a function $f:\R^m\rightarrow\R^n$, for some $m,n \in \mathbb{N}_+$. Each directed edge $\langle u,v\rangle \in E \subseteq V\times V$ indicates that the output of layer $u$ is used as an input to layer $v$.

We denote by $V_l \subseteq V$ the subset of vertices corresponding to general linear layers. We further define $V_s := \left\{ f\mid\forall \rvx \in \mathbb{R}^m,\ f(\rvx) = a \cdot \rvx,\ \forall a \in \mathbb{R} \right\}$ as the set of scalar operations, and $V_r := \left\{ f \mid \forall \rvx_1, \rvx_2 \in \mathbb{R}^m,\ f(\rvx_1, \rvx_2) = \rvx_1 + \rvx_2,\ \right\}$ as the set of residual addition operations. For layers whose outputs are intrinsically zero-mean, we define $V_c:=\{f\mid\forall \rvx \in \mathbb{R}^m , f(\rvx)^\top \boldsymbol{1}_n=0\}$. All remaining operations are collected into $V_- := V \setminus (V_l \cup V_{s} \cup V_{r} \cup V_c)$. These typically include nonlinear operations such as ReLU or softmax, which in general disrupt the zero-mean property and hence prevent the downstream LN from being foldable.

To determine whether an LN is foldable, we verify whether its input maintains zero mean due to parameter constraints imposed on upstream layers. Starting from the vertex $v_{\mathrm{LN}}$ corresponding to the LN, we backtrack along the graph according to the following rules:
\begin{itemize}
    \item If the predecessor vertex belongs to $V_l$, zero-mean output can be ensured by applying CBWC.
    \item If it belongs to $V_s$, we continue backtracking through its predecessor.
    \item If it belongs to $V_r$, we backtrack through all of its input branches.
    \item If it belongs to $V_c$, the zero-mean property is already guaranteed.
    \item If it belongs to $V_-$, zero mean is not guaranteed, and thus LN folding fails.
\end{itemize}

\begin{algorithm}[ht]
\label{alg:detect-F-LN}
\caption{Detect foldable LNs and the corresponding upstream layers requiring CBWC.}
\begin{algorithmic}[1]
\State \textbf{Input:} Model $\mathcal{M}$ with input tensor $T_0^\mathrm{in}$
\State \textbf{Output:} Set $\mathcal{S}$ of foldable LNs and set $\mathcal{C}$ of corresponding upstream layers requiring CBWC
\State $T_0^\mathrm{in}.\textit{centered} \gets \texttt{False}$  \CommentGray{Initial tensor state.}
\State $\mathcal{S},\ \mathcal{C} \gets \emptyset$  \CommentGray{Initialize the output sets.} 

\For{each step $t$, layer $M_t \in \mathcal{M}$} 
 \Statex \CommentGray{Process the current layer.}
    \If{$T_t^\mathrm{in} \gets \sum_{i=1}^n T_i$} 
     \Statex\CommentGray{$T_t^\mathrm{in}$ corresponds to a residual connection.}
        \State $T_t^\mathrm{in}.\textit{centered} \gets \bigwedge_{i=1}^n T_i.\textit{centered}$ 
        \State $T_t^\mathrm{in}.\textit{corresponding} \gets \bigcup_{i=1}^n T_i.\textit{corresponding}$
        \Statex\CommentGray{Update the state and corresponding layer set.}
    \EndIf

    \If{$M_t = \text{LN} \text{ and } T_t^\mathrm{in}.\textit{centered} = \texttt{True}$}
        \State $\mathcal{S} \gets \mathcal{S} \cup \{M_t\}$
        \State $\mathcal{C} \gets \mathcal{C} \cup \{T_t^\mathrm{in}.\textit{corresponding}\}$
        \Statex\CommentGray{Record the foldable LN and its corresponding layers.}
    \EndIf
    
    \State $T_t^\mathrm{out} \gets M_t(T_t^\mathrm{in})$ \CommentGray{Compute the output tensor.}
    
    \If{$M_t \in V_s$}
        \State $T_t^\mathrm{out}.\textit{centered} \gets T_t^\mathrm{in}.\textit{centered}$ 
        \State $T_t^\mathrm{out}.\textit{corresponding} \gets T_t^\mathrm{in}.\textit{corresponding}$
        \Statex\CommentGray{Preserve the previous state and set.}
    \ElsIf{$M_t \in V_l \cup V_c $}
        \State $T_t^\mathrm{out}.\textit{centered} \gets \texttt{True}$
        \State $T_t^\mathrm{out}.\textit{corresponding} \gets M_t$
        \Statex\CommentGray{Update the state and corresponding layer.}
    \Else
        \State $T_t^\mathrm{out}.\textit{centered} \gets \texttt{False}$
        \State $T_t^\mathrm{out}.\textit{corresponding} \gets \emptyset$
        \Statex\CommentGray{Reset the state and corresponding layer set.}
    \EndIf

\EndFor
\State \textbf{Return:} $\mathcal{S}$ and $\mathcal{C}$
\end{algorithmic}
\end{algorithm}

\begin{definition}[Zero-Mean Graph]
\label{def:zero-mean-graph}
Given a neural network with computation graph $G=(V,E)$ and a target LN vertex $v_{\mathrm{LN}} \in V$, we recursively construct the corresponding \emph{zero-mean graph} $G_z=(V_z,E_z)$ as follows. Initialize $V_0=\{u\mid \langle u,v_{\mathrm{LN}}\rangle \in E\}$ and $E_0=\{\langle v_{\mathrm{LN}}, u\rangle \mid \langle u,v_{\mathrm{LN}}\rangle \in E\}$.
For each iteration $k=0,1,2,\dots$:
\vspace{-1ex}
\begin{enumerate}
    \item Let $A_k=\{v\in V_k\mid v\in(V_r\cup V_s)\}$, i.e., the subset of current leaf vertices that require further backtracking.
    \item Define $V_{k+1}=\bigcup_{v\in A_k}\{u\mid\langle u,v\rangle \in E\}$ by backtracking from every vertex in $A_k$. If $V_{k+1}=\emptyset$, the iteration terminates.
    \item Define $E_{k+1}=\bigcup_{v\in A_k}\{\langle v,u\rangle \mid\langle u,v\rangle \in E\}$ by adding the corresponding reversed edges to the graph.
\end{enumerate}
Define $V_z = \bigcup_{k=0}^\infty V_k, \ E_z = \bigcup_{k=0}^\infty E_k$. 
\end{definition}

The leaf nodes of a zero-mean graph belong to either $V_l$, $V_c$, or $V_-$. Therefore, the root LN is foldable if all leaf nodes correspond to general linear layers or layers that already guarantee zero-mean output. In that case, the LN can be replaced with RMS\-Norm by applying CBWC to the corresponding upstream general linear layers. 

For the remaining case in which a leaf node belongs to $V_-$, zero-mean output can still be enforced by inserting an explicit auxiliary centering operation after that layer, which we discuss as a practical extension. This makes the subsequent LN foldable, provided that the inserted centering does not violate the stricter criterion discussed in Appendix~\ref{sec:stricter-critieria}. Although such auxiliary centering introduces extra computational cost, it is still beneficial when the number of resulting foldable LNs exceeds the number of inserted centering operations.
Under this practical extension, all LNs in pre-norm and post-norm Transformers~\citep{2020_ICML_Xiong} can be folded (see Appendix~\ref{sec:gpt}).

\subsection{Algorithm for Detecting Foldable LNs}
\label{sec:alg}
Based on the above analysis, we propose an algorithm for detecting all foldable LNs in arbitrary neural networks and identifying the corresponding upstream layers requiring CBWC. Rather than explicitly backtracking from each LN to form its zero-mean graph, we track the state of each tensor together with its corresponding upstream layers requiring CBWC during a forward traversal. We present Algorithm~1.

Our algorithm automatically detects foldable LNs in arbitrary neural networks and identifies the corresponding upstream general linear layers requiring CBWC to enable mathematically equivalent replacement of LN with RMS\-Norm. We empirically evaluated Algorithm 1 on mainstream architectures including GPT-2~\citep{Radford2019GPT2}, BERT~\citep{2019_ACL_Devlin}, ViT~\citep{2021_ICLR_Dosovitskiy}, Phi~\citep{2023_Phi_Textbooks_Gunasekar}, BLOOM~\citep{2022_BLOOM_Scao}, and OPT~\citep{zhang2022opt}. Remarkably, all LNs in these models can be folded. We also fold most LNs in VMamba~\citep{liu2024vmamba}. Other widely used models such as LLaMA~\citep{touvron2023llama} and Mamba2~\citep{dao2024transformers} employ RMS\-Norm instead of LN by design. 
We note that the runtime and memory overhead of the folding procedure is minimal and can be amortized after only a few inference steps.
Further details are provided in Appendix~\ref{apx:f-ln-table}.

\section{Experiments}
\label{sec:training}

In this section, we empirically evaluate the effectiveness and efficiency of the proposed method. Since our goal is to retain the advantages of LN while reducing its computational cost, we first examine whether the centering operation in LN provides practical benefits over RMS\-Norm. We then study the performance of CBWC+RMS\-Norm in training and fine-tuning settings.

\begin{figure}[t]
\vspace{-1ex}
    \centering
    \begin{subfigure}[t]{0.23\textwidth}
        \centering
      \includegraphics[height=16ex]{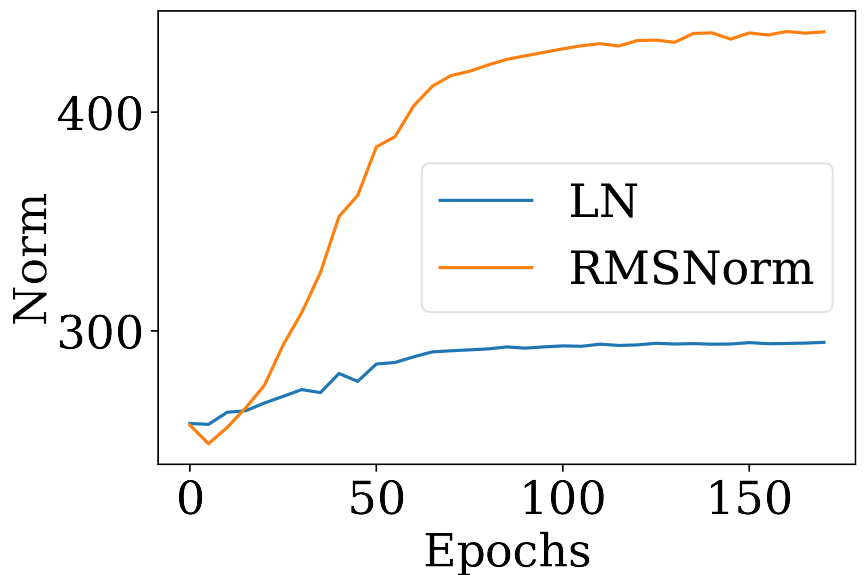}
        \caption{d=15 on CIFAR-10.}
    \end{subfigure}
    \hspace{0.05in}	
    \begin{subfigure}[t]{0.23\textwidth}
        \centering
      \includegraphics[height=16ex]{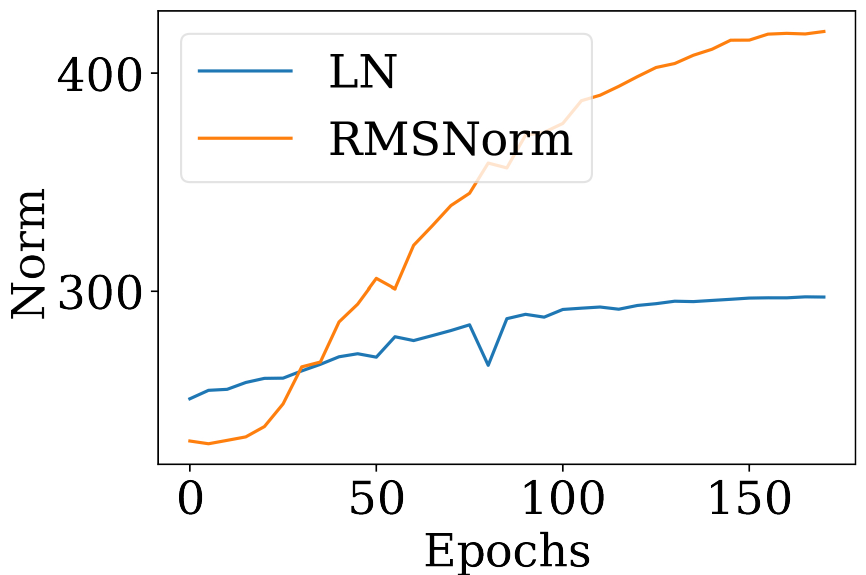}
        \caption{d=35 on CIFAR-10.}
    \end{subfigure}

    \begin{subfigure}[t]{0.23\textwidth}
        \centering
      \includegraphics[height=16ex]{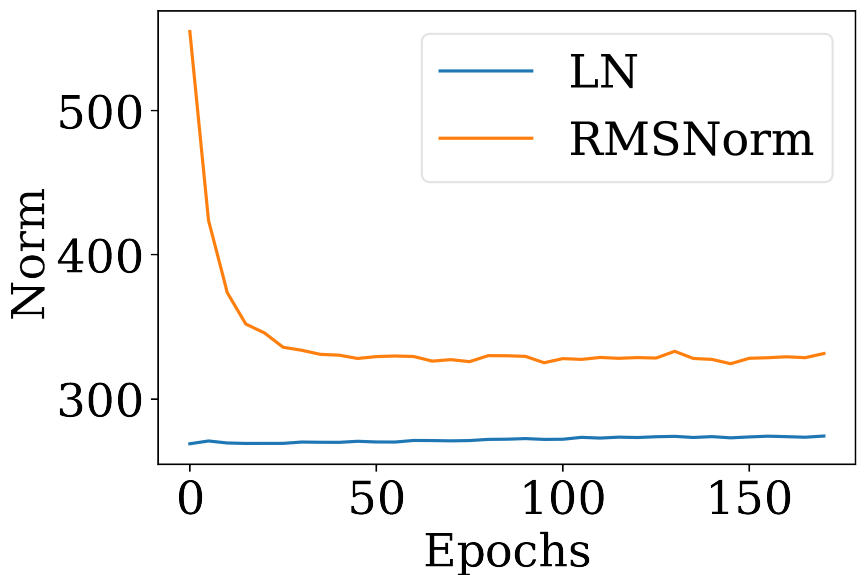}
        \caption{d=65 on MNIST.}
    \end{subfigure}
    \hspace{0.05in}	
    \begin{subfigure}[t]{0.23\textwidth}
        \centering
      \includegraphics[height=16ex]{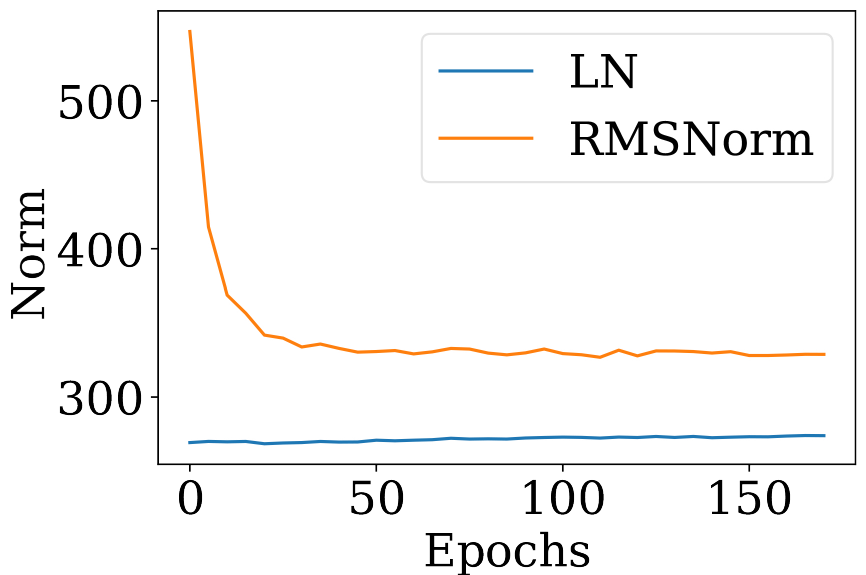}
        \caption{d=100 on MNIST.}
    \end{subfigure}
    
  \caption{Norm of the input to the final layer for MLPs of different depths ($d$). LN more effectively controls the sample norm before the final layer throughout training.}
  \label{fig:exp-centering}
  \vspace{-2ex}
\end{figure}

\subsection{Observation of Centering}
\label{exp:ln}

Although the benefits of centering in LN remain debated, we observe that centering helps stabilize the activation range of a network. To examine this effect, we conduct ablation experiments on MLPs of different depths using LN and RMS\-Norm on the CIFAR-10~\citep{Krizhevsky2009CIFAR10} and MNIST~\citep{lecun1998gradient} classification tasks. We monitor the input statistics of each layer throughout training. Detailed experimental settings and additional results are provided in Appendix~\ref{apx:ablation}. 

As shown in Figure~\ref{fig:exp-centering}, LN keeps the norm of the last layer’s input within a narrower range throughout training. In contrast, without the centering operation, the input mean and norm exhibit substantially larger variations. We also report the variation of input norm across all linear layers in Appendix~\ref{apx:ablation}, further illustrating the stabilizing role of LN’s centering operation, particularly in the presence of residual connections.

These observations suggest that LN retains practical benefits over RMS\-Norm that are worth preserving, thereby justifying a structured and theoretically grounded replacement of LN rather than a direct switch to RMS\-Norm.

\subsection{Inference-Time Acceleration}
\label{sec:acceleration}

\begin{figure}[t]
    \centering
    \begin{subfigure}[t]{0.23\textwidth}
        \centering
        \includegraphics[height=17ex]{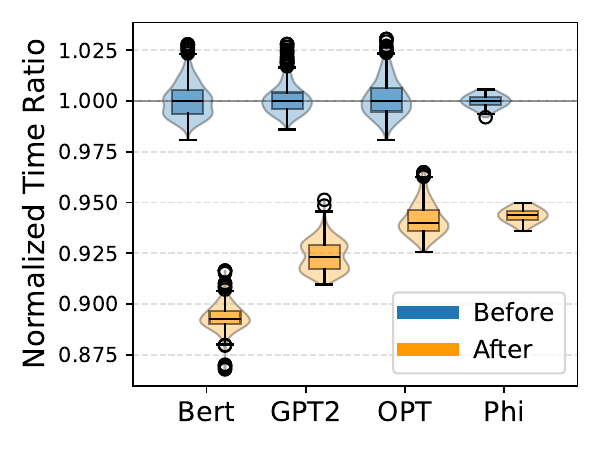}
        \caption{Normalized time.}
    \end{subfigure}  
    \begin{subfigure}[t]{0.23\textwidth}
        \centering
        \includegraphics[height=17ex]{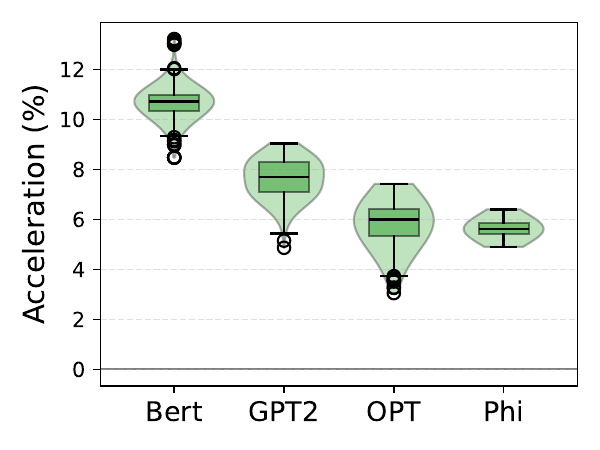}
        \caption{Acceleration ratio.}
    \end{subfigure}   
    \caption{Inference latency comparison across four representative models (GPT-2, BERT, OPT, and Phi-3). Both the normalized latency and the acceleration percentage are reported relative to the mean baseline latency of each model.}
    \label{fig:cuda-total-time-accel}
    \vspace{-2ex}
\end{figure}

At inference time, CBWC+RMS\-Norm provides a direct efficiency advantage over standard LN. Notably, CBWC only needs to be applied once to a pre-trained model before deployment, since the weight matrix remains fixed during inference. At the layer level, replacing LN with RMS\-Norm substantially reduces the cost of normalization, yielding a theoretical latency reduction of roughly 50\% to 80\% for the normalization step itself. At the whole-model level, this translates into an expected end-to-end runtime reduction of roughly 4\% to 7\%, depending on the proportion of LN in the model. More details are provided in Appendix~\ref{apx:inference-accele-theory}.

To verify the practical effectiveness of our method, we compare LN with RMS\-Norm on a single A100-40G GPU. We conduct experiments on GPT-2~\citep{Radford2019GPT2}, BERT~\citep{2019_ACL_Devlin}, BLOOM~\citep{2022_BLOOM_Scao}, OPT~\citep{zhang2022opt}, and Phi~\citep{2023_Phi_Textbooks_Gunasekar}. Our method reduces end-to-end inference time by approximately 2\% to 12\% at batch size 2 and sequence length 1024, while preserving identical outputs, as shown in Figure~\ref{fig:cuda-total-time-accel}.

Although the relative improvement may appear modest, it translates into substantial absolute savings in energy and operational cost when deployed at scale. Additional experiments on both V100 and A100 GPUs show consistent latency reductions of 2.5\% to 5.9\% together with throughput gains of 1,119 to 3,385 tokens/s across multiple architectures. Moreover, the acceleration becomes more pronounced at longer sequence lengths, with the relative speedup increasing substantially in long-context settings. These results confirm that CBWC+RMSNorm provides a reliable inference-time improvement while preserving equivalence across diverse architectures and deployment scenarios.

\subsection{Empirical Study of Training}

CBWC+RMS\-Norm can be more efficient than LN during training, especially in long-sequence settings. Let $b$ denote the batch size, $s$ the sequence length, and $d$ the hidden dimension. Let the weight matrix be $W \in \mathbb{R}^{d \times p}$. For one training epoch with $B$ samples, a centering operation incurs a computational cost of approximately $\mathcal{O}(Bsd)$, while CBWC introduces a cost of $\mathcal{O}(Bdp/b)$. Therefore, CBWC is more efficient when $s \times b > p$, which commonly holds in practice, particularly for long-context learning.

Although Prop.~\ref{prop:opt} establishes the theoretical equivalence between a foldable LN and CBWC+RMS\-Norm during both inference and training, practical Transformer architectures typically include dropout layers~\citep{2017_NIPS_Vaswani} between general linear layers and LN. In training mode, dropout disrupts the zero-mean property of activations, thereby breaking exact equivalence. We therefore conduct experiments to empirically evaluate the performance and efficiency of CBWC+RMS\-Norm on Transformer models. We compare three variants: the baseline with LN, the variant with vanilla RMS\-Norm, and the variant using our method, CBWC+RMS\-Norm.

\begin{figure}[t]
    \centering
    \begin{subfigure}[t]{0.23\textwidth}
        \centering
        \includegraphics[height=17ex]{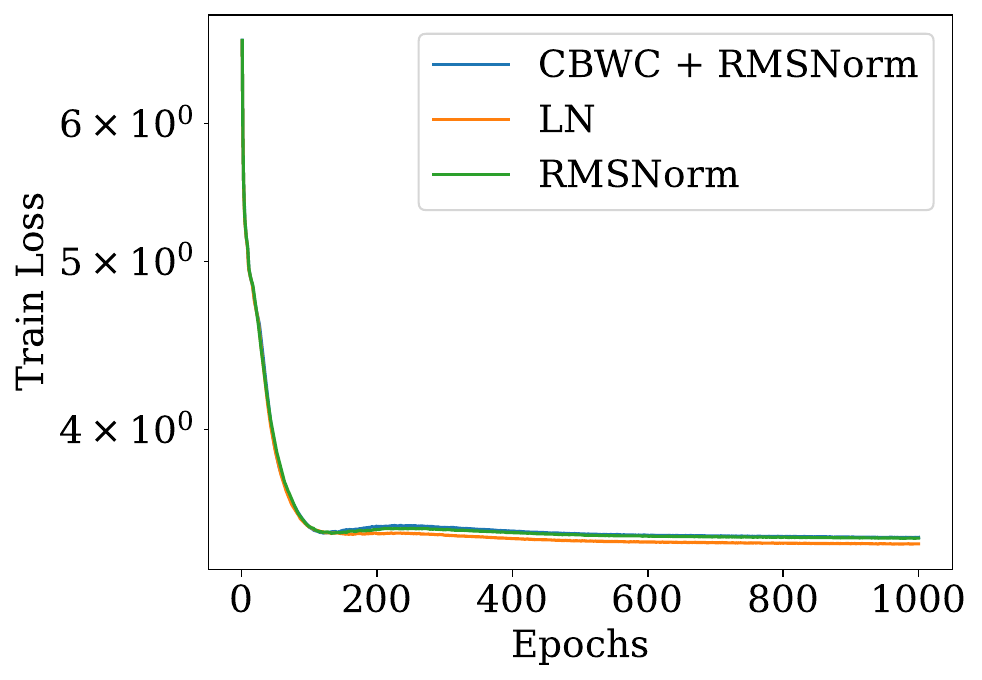}
        \caption{Training loss.}
    \end{subfigure}   
    \begin{subfigure}[t]{0.23\textwidth}
        \centering
        \includegraphics[height=17ex]{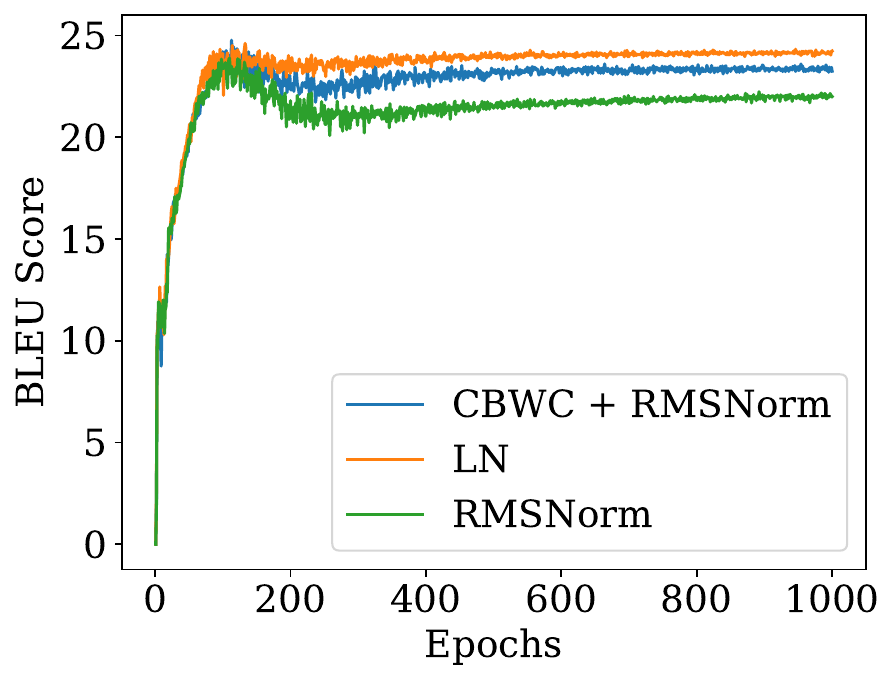}
        \caption{Evaluation BLEU.}
    \end{subfigure}   
    
    \caption{Training loss curves and validation BLEU scores of Transformer on the Multi30K translation task. Our proposed CBWC+RMS\-Norm achieves final performance between standard LN and vanilla RMS\-Norm.}
    \label{Results of Transformer Models for Text Translation Task}
    \vspace{-2ex}
\end{figure}

\paragraph{Text Translation.}

We train the models on the Multi30K~\citep{2016_Multi30k_ILLC} dataset following the experimental protocol of \citet{2017_NIPS_Vaswani}. We measure training performance using the loss value and evaluate inference performance using BLEU scores~\citep{2002_BLEU_ACL}. Higher BLEU and lower loss indicate better performance. All models are trained for 1000 epochs. As shown in Figure~\ref{Results of Transformer Models for Text Translation Task}, after the first 100 epochs, vanilla RMS\-Norm lags behind both the baseline and our method. Although CBWC+RMS\-Norm is slightly inferior to LN in terms of final training loss and BLEU, it consistently outperforms vanilla RMS\-Norm.

\begin{figure}[t]
    \centering
    \vspace{-1ex}
    \begin{subfigure}[t]{0.23\textwidth}
        \centering
        \includegraphics[height=17ex]{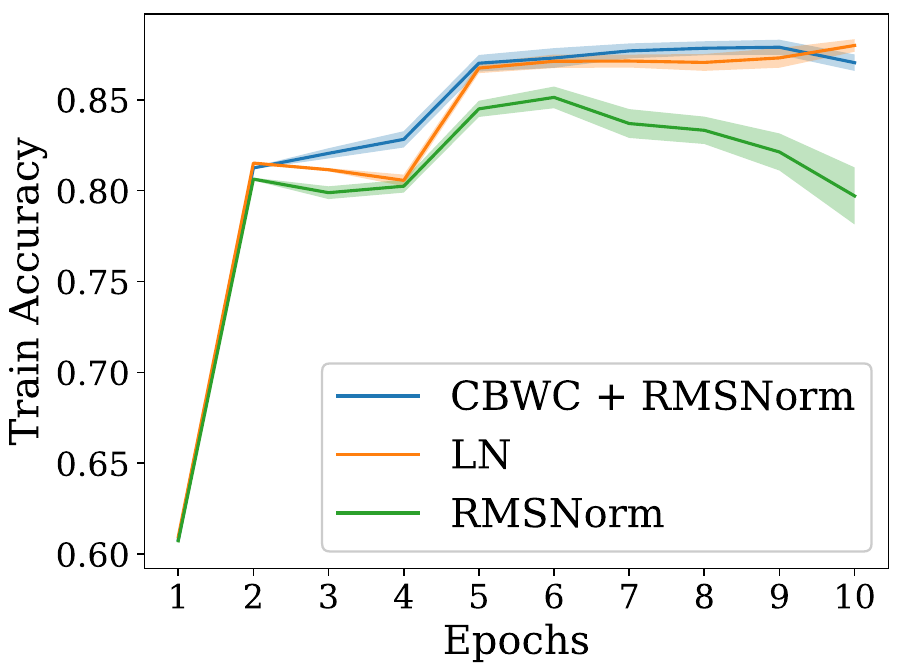}
        \caption{Training accuracy.}
    \end{subfigure}
    \hspace{0.05in}	
    \begin{subfigure}[t]{0.23\textwidth}
        \centering
        \includegraphics[height=17ex]{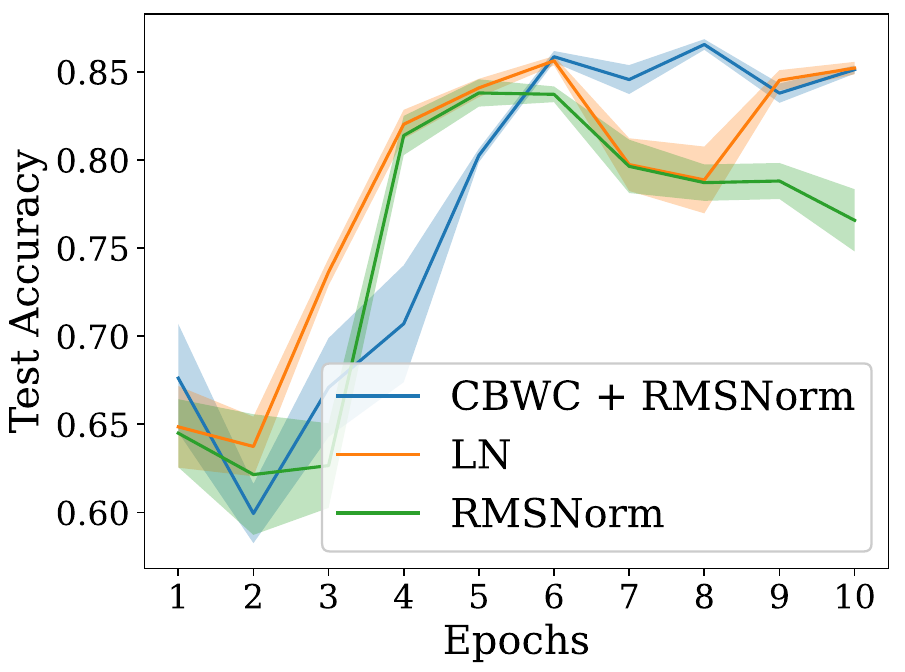}
        \caption{Test accuracy.}
    \end{subfigure}
        \begin{subfigure}[t]{0.23\textwidth}
        \centering
        \includegraphics[height=17ex]{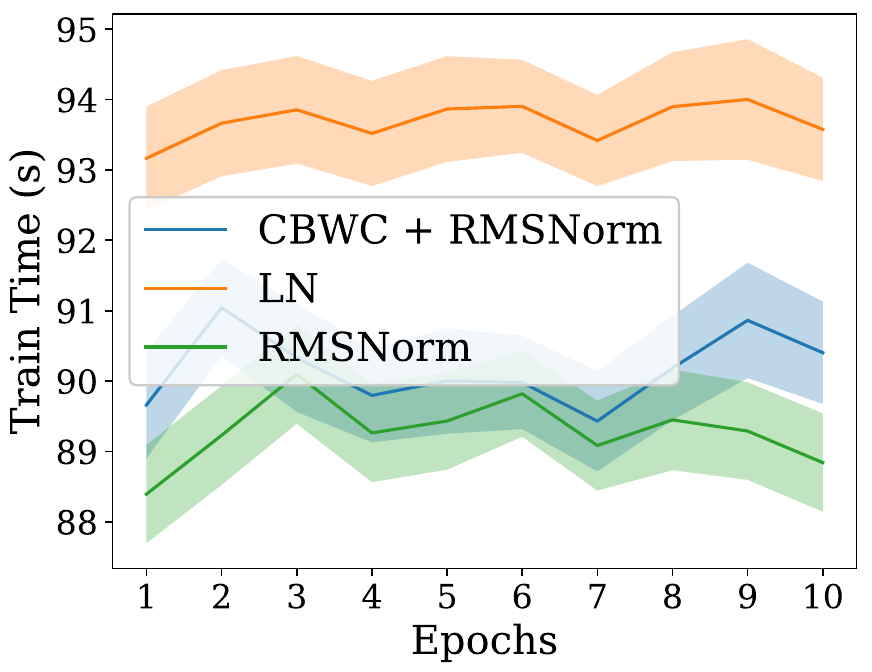}
        \caption{Training time.}
    \end{subfigure}
    \hspace{0.05in}	
    \begin{subfigure}[t]{0.23\textwidth}
        \centering
        \includegraphics[height=17ex]{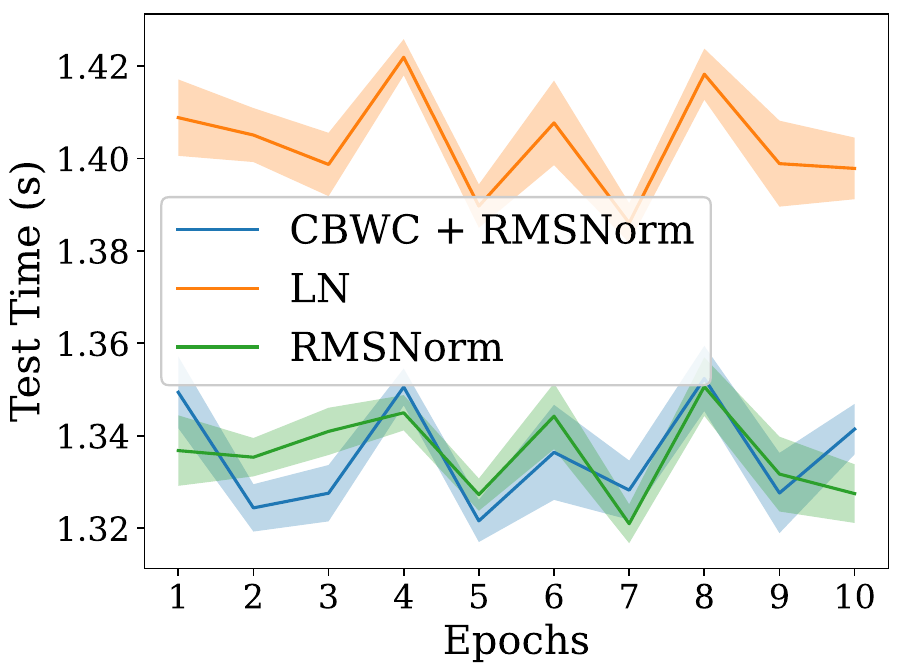}
        \caption{Test time.}
    \end{subfigure}
    
    \caption{Training and test performance of Transformer-based text classification on the AG News dataset. The results are averaged over 5 random seeds with shaded regions indicating standard deviation. Our CBWC+RMS\-Norm matches the convergence behavior and final accuracy of standard LN while achieving nearly the same training and inference throughput as vanilla RMS\-Norm.}
    \label{Results of Transformer Models for Text Classification Task When Training}
    \vspace{-1ex}
\end{figure}

\paragraph{Text Classification.}
We train the models on the AG News~\citep{2015_AGNEWS_NeurIPS} dataset using the same Transformer setting as in the translation experiment. Each model is trained for 10 epochs with 5 random seeds, and we evaluate both accuracy and time usage. As shown in Figure~\ref{Results of Transformer Models for Text Classification Task When Training}, the baseline and CBWC+RMS\-Norm both outperform vanilla RMS\-Norm during training, while RMS\-Norm exhibits less stable behavior. In terms of efficiency, CBWC+RMS\-Norm is faster than the LN baseline in both training and evaluation, while remaining close to vanilla RMS\-Norm.

\paragraph{Image Classification.}
We conduct image classification experiments on ImageNet-100~\citep{deng2009imagenet,imagenet100pytorch} using Swin-Tiny~\citep{liu2021Swin}. We evaluate both inference performance and runtime efficiency, including forward-pass, backward-pass, and validation time. All models are trained for 100 epochs, and results are averaged over 3 random seeds. Detailed settings are provided in Appendix~\ref{apx:swin}. As shown in Table~\ref{table:performace-swin} and Figure~\ref{table:time-usage}, our method achieves the best test performance among the three variants while also improving runtime over the LN baseline. Its latency remains close to that of vanilla RMS\-Norm in both the forward pass and evaluation.

\begin{table}[t]
  \centering
  \caption{Image classification performance on ImageNet-100 using Swin-Tiny (averaged over 3 random seeds). Our method achieves the best results on both accuracy and loss.}
  \label{table:performace-swin}
  \begin{tabular}{rcc}
    \toprule
    Method & Test Acc (\%) $\uparrow$ & Test Loss $\downarrow$ \\ \midrule
    LN       & $57.270 \pm 0.060$ & \textcolor[gray]{0.6}{$3.352 \pm 0.015$} \\
    RMS\-Norm & \textcolor[gray]{0.6}{$57.079 \pm 0.317$} & $3.350 \pm 0.023$ \\
    CBWC+RMS & $\mathbf{57.768 \pm 0.417}$ & $\mathbf{3.232 \pm 0.043}$ \\
    \bottomrule
  \end{tabular}
\end{table}

\begin{figure}[t]
\vspace{-1ex}
    \centering
    \begin{subfigure}[t]{0.23\textwidth}
        \centering
        \includegraphics[height=17ex]{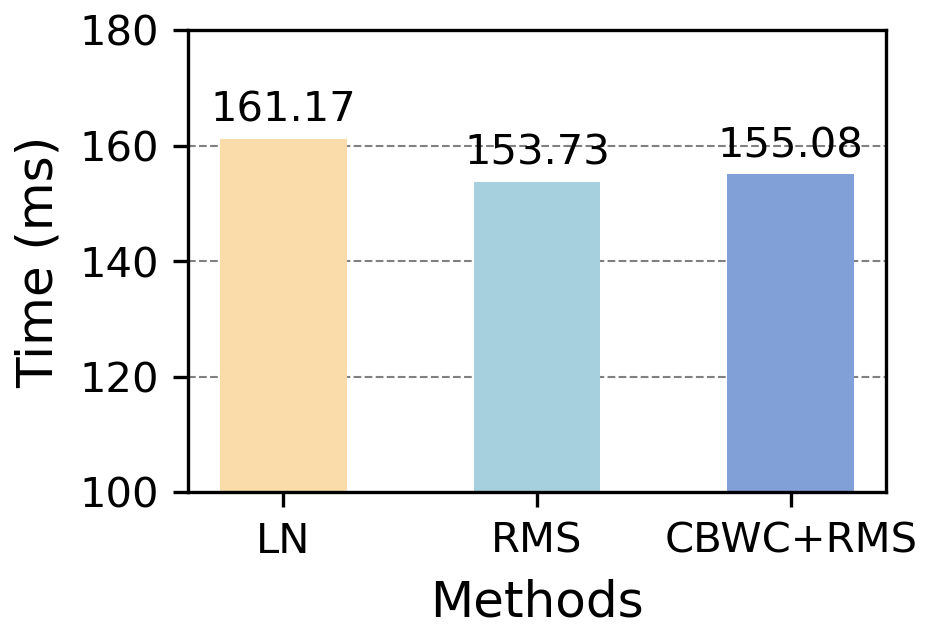}
        \caption{Forward pass.}
    \end{subfigure}   
    \begin{subfigure}[t]{0.23\textwidth}
        \centering
        \includegraphics[height=17ex]{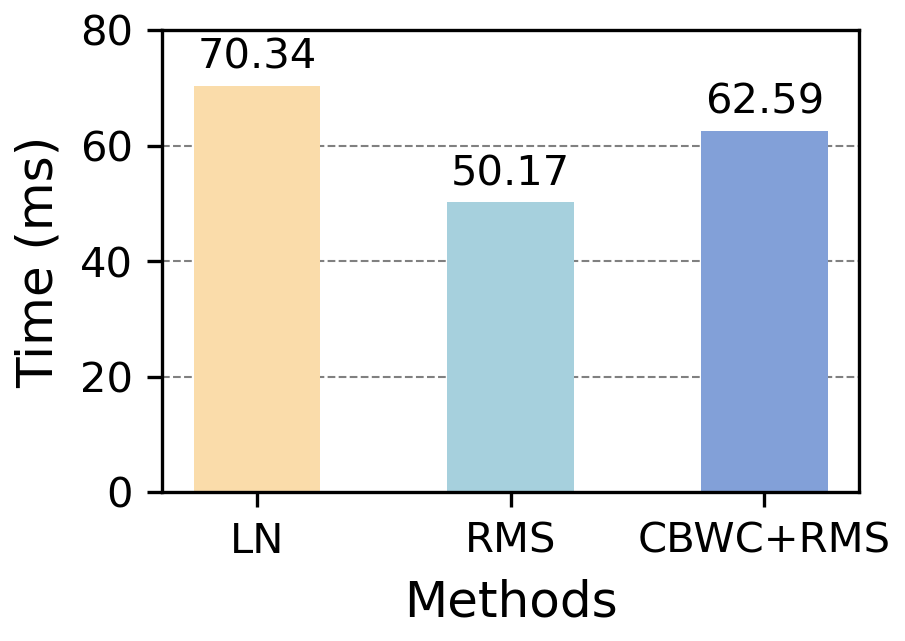}
        \caption{Back propagation.}
    \end{subfigure}   
    \caption{End-to-end latency comparison of Swin-Tiny variants for different processes on ImageNet-100. Our CBWC+RMS\-Norm delivers markedly lower latency than LN in both forward and backward passes, approaching the efficiency of vanilla RMS\-Norm.}
    \label{table:time-usage}
    \vspace{-2ex}
\end{figure}

\paragraph{Summary.}
When exact folding conditions do not fully hold during training, CBWC+RMS\-Norm still consistently performs between LN and vanilla RMS\-Norm, and often remains close to LN in predictive performance while providing noticeable efficiency gains, especially in long-sequence tasks.

\paragraph{Practical Trade-off.}
The core goal of our method is to reduce the computational cost of LN while preserving its behavior as much as possible. Therefore, when exact equivalence does not fully hold, whether to apply the method during training or only at inference time becomes a practical design choice. If strict equivalence and no accuracy degradation are required, it is preferable to apply the method only at inference time, for example by converting an LN-trained model after training to accelerate deployment. By contrast, if faster training is preferred and a small performance gap from LN is acceptable, applying the method during training provides a favorable compromise. In practice, this allows one to choose between better fidelity to LN and greater computational efficiency depending on the target use case.

\subsection{Discussion on Fine-tuning}
\label{sec:fine-tune}
In practical applications, models are often fine-tuned from pre-trained checkpoints rather than trained from scratch. Under the conditions of Prop.~\ref{prop:opt}, LN and CBWC+RMS\-Norm follow the same optimization process, and their parameters are mathematically equivalent under the same weight matrix. Therefore, a model pre-trained with LN can in principle be converted and continue training with our method equivalently. We provide a verification experiment for this equivalence in Appendix~\ref{sec:mlp-verify}.

We further study the feasibility of applying our method during fine-tuning on Transformer models. In practical settings, however, dropout layers break the exact equivalence between the LN model and the CBWC+RMS\-Norm model during training. We therefore conduct experiments under three settings: continuing fine-tuning with LN as the baseline, converting the pre-trained model to CBWC+RMS\-Norm before fine-tuning, and converting the pre-trained model to CBWC while fine-tuning with vanilla RMS\-Norm. We pre-train GPT-3 Medium (350M)~\citep{brown2020language} on WikiText-103~\citep{merity2016pointersentinelmixturemodels} for 3 epochs and then fine-tune the models on Alpaca~\citep{alpaca} for 1 epoch. More detailed settings are provided in Appendix~\ref{app:expcom}.

We evaluate the models using loss and perplexity (PPL), where lower values indicate better performance. The results are shown in Table~\ref{table:setting-verfication}. CBWC+RMS\-Norm achieves the best test loss and perplexity among the compared variants while also benefiting from improved efficiency. These results suggest that applying our method after LN pre-training remains a viable and effective strategy for fine-tuning.

\begin{table}[t]
\centering
\caption{Fine-tuning performance of GPT-3 Medium (350M) on the Alpaca instruction-following dataset after pre-training on WikiText-103. Our CBWC+RMS\-Norm achieves the best test loss and perplexity among all variants.}
\label{table:setting-verfication}
\begin{tabular}{l@{\hspace{4pt}}c@{\hspace{4pt}}c}
\toprule
Method & Test Loss $\downarrow$& Test PPL $\downarrow$\\ \midrule
LN        & \textcolor[gray]{0.6}{0.2553} & \textcolor[gray]{0.6}{1.29} \\
RMS\-Norm  & 0.2526 & 1.29 \\
CBWC+RMS  & $\mathbf{0.2430}$ & $\mathbf{1.28}$ \\
\bottomrule
\end{tabular}
\vspace{-4ex}
\end{table}

\section{Conclusion}
We propose a general framework for determining whether an LN in an arbitrary DNN can be replaced with RMS\-Norm while preserving mathematical equivalence under appropriate structural conditions. We show how the centering operation of LN can be removed through the proposed column-centered constraint (CCC) and column-based weight centering (CBWC), and further extend this analysis to arbitrary neural networks through the notions of foldable LN and zero-mean graph. This framework enables exact replacement of foldable LNs during inference with measured end-to-end inference speedups of 2\%–12\% on representative models and is especially relevant to inference acceleration in LN-based LLMs and VLMs, where exact forward equivalence is particularly desirable. In particular, it provides an effective practical strategy for training and fine-tuning even when the exact folding conditions do not fully hold. More broadly, it offers a principled way to reduce the computational cost of LN while preserving its benefits as much as possible.

\paragraph{Limitations and Future Work.}
Our analysis and empirical validation, while covering several representative architectures and tasks, do not yet exhaust the full range of modern large-scale models with LN. In future work, it would be valuable to extend the study to a broader set of LLMs and VLMs and to provide more systematic evaluations of the practical gains achieved by the proposed method in different deployment settings. In addition, exact equivalence may fail in architectures containing modules such as dropout, which break the zero-mean property during training. Developing more refined treatments of such modules and extending the framework to a wider class of practical architectures remain important directions for future research.

\section*{Acknowledgments}
This work was partially supported by Beijing Natural Science Foundation (Grant No. QY24137), National Natural Science Foundation of China (Grant No. 62476016 and 62441617) and the Fundamental Research Funds for the Central Universities.

\section*{Impact Statement}
This paper presents work whose goal is to advance the field of machine learning. We do not identify any specific negative societal impacts that require particular discussion beyond the general considerations associated with improving the efficiency of deep neural networks.

\bibliography{main}
\bibliographystyle{icml2026}
\newpage

\onecolumn
\appendix
\clearpage

\section{Bias in the Weight Matrix}
\label{apx:bias}
In the main paper, we omit the bias for simplicity. Here, we provide the methodology for merging the bias into the weight matrix. For a linear layer $\mathbf h= \mW \rvx\ +\ \rvb $, we denote the bias $\rvb = [b_1, b_2,\dots,b_m]^\top \in \mathbb R ^{m\times 1}$. For simplicity, we append an additional dimension to $\rvx$, which turns $\rvx$ into $\rvx'=[x_1,x_2,\dots,x_d,1]^\top\in \mathbb R^{(d+1)\times 1}$, and adds an additional column in $\mW'\in\mathbb R^{m \times (d+1)}$, where: 
\begin{equation}
     \label{eqn:48}
     \mW' = \left[
    \begin{matrix}
    {w}_{1,1} & {w}_{1,2} & \cdots & {w}_{1,d} & {b}_1\\ 
    {w}_{2,1} & {w}_{2,2} & \cdots & {w}_{2,d} & {b}_2\\
    \vdots & \vdots & \ddots & \vdots & \vdots \\
    {w}_{n,1} & {w}_{m,2} & \cdots & {w}_{m,d} & {b}_m
    \end{matrix}
    \right].
\end{equation}
Therefore, we have $\mathbf h = \mW \rvx\ +\ \rvb =\mW' \rvx'$

In more general cases, for a linear transformation$f(\cdot;\vtheta):\R^n\rightarrow\R^m$, we similarly add an additional dimension to $\rvx'=[x_1,x_2,\cdots,x_d,1]$, and a related weight vector $\rvw_{d+1}=\rvb = [b_1,b_2,\cdots,b_m]^\top\in\R^m$. Therefore, we have transformed expression:
\begin{equation}
   \rvh(\cdot;\vtheta)=\sum^{d}_{i=1}\rvw_i(\vtheta)x_i + \rvb =\sum^{d+1}_{i=1}\rvw_i(\vtheta)x_i,
\end{equation}
thus incorporating the bias into the weight matrix.

\section{Proof of Propositions}

\subsection{Proof of Prop.~\ref{prop:zero-mean-ccc}}
\label{apx:proof-ccc}

\begin{proof}
    For arbitrary input $\rvx$, the output is $\rvh = \sum_{i=1}^n \rvw_i(\vtheta_k) x_i$.
    Taking the mean of $\rvh$ over the $m$ output neurons,
    $$
    \mu_h = \frac{1}{m} \boldsymbol{1}_m^\top \rvh = \frac{1}{m} \sum_{i=1}^{n} x_i \underbrace{\boldsymbol{1}_m^\top \rvw_i(\vtheta_k)}_{=0} = 0,
    $$
    where the equality follows from CCC \eqref{eqn:ccc-general}.
\end{proof}

\subsection{Proof of Prop.~\ref{prop:opt}}
\label{apx:proof-equivalent-opt}

\begin{proof}[Sketch]
    Let model $A$ have a standard linear layer followed by LN, while model $B$ uses a linear layer with CBWC followed by RMS\-Norm.
    In model $A$, the centered output after LN can be written as 
    $\widetilde{\rvh}_A = \left(\mI - \frac{1}{m} \boldsymbol{1}_m \boldsymbol{1}_m^\top \right) \mW_A \rvx_A,$
    where $\rvx_A$ is the input.
    In model $B$, the output is 
    $\widetilde{\rvh}_B = \mV_B \rvx_B = \left(\mI - \frac{1}{m} \boldsymbol{1}_m \boldsymbol{1}_m^\top \right) \mW_B \rvx_B.$
    If $\rvx_A = \rvx_B$ and $\mW_A = \mW_B$, then forward outputs are identical.

    Since the parameters and outputs coincide, their gradients and losses coincide as well. Therefore, in model $A$, we have backpropagation process as:
    \begin{equation}
        \frac{\partial \mathcal{L}}{\partial \rvh_A}=\left(\mI-\frac{1}{m}\boldsymbol{1}_m\boldsymbol{1}_m^\top\right)^\top\frac{\partial \mathcal{L}}{\partial \widetilde\rvh_A},\qquad
        \frac{\partial \mathcal{L}}{\partial \rvx_A} = \mW_A^\top \frac{\partial \mathcal{L}}{\partial \rvh_A},\qquad
        \frac{\partial \mathcal{L}}{\partial \mW_A} =\frac{\partial \mathcal{L}}{\partial \rvh_A}\rvx_A^\top.
    \end{equation}
    When in model $B$, according to the definition of backward transformation $\pmb\psi$ in CBWC, similarly we have:
    \begin{equation}
        \frac{\partial \mathcal{L}}{\partial \rvx_B}=\mV_B\frac{\partial \mathcal{L}}{\partial\widetilde \rvh_B},\qquad
        \frac{\partial \mathcal{L}}{\partial \mV_B} = \frac{\partial \mathcal{L}}{\partial \widetilde\rvh_B}\rvx_B^\top,\qquad
         \frac{\partial \mathcal{L}}{\partial \mW_B} = \left(\mI-\frac{1}{m}\boldsymbol{1}_m\boldsymbol{1}_m^\top\right)^\top \frac{\partial \mathcal{L}}{\partial \mV_B}.
    \end{equation}
    It is straightforward to verify that the two backpropagation processes are identical:
    \begin{equation}
        \frac{\partial \mathcal{L}}{\partial \rvx} = \left(\mI-\frac{1}{m}\boldsymbol{1}_m\boldsymbol{1}_m^\top\right)^\top \mW^\top\frac{\partial \mathcal{L}}{\partial \widetilde\rvh},\qquad
        \frac{\partial \mathcal{L}}{\partial \mW} = \left(\mI-\frac{1}{m}\boldsymbol{1}_m\boldsymbol{1}_m^\top\right)^\top \frac{\partial \mathcal{L}}{\partial \widetilde \rvh}\ \rvx^\top.
    \end{equation}    
    Therefore, CBWC+RMSNorm yields an identical optimization process.
\end{proof}

\section{Column-Based Weight Centering of General Linear Layers}
\label{apx:proof-of-ccc}
In this section, we introduce the column-based weight transformation (CCWT) and give the explicit definition of general linear layers. We provide how typical general linear layers can be converted into a linear transformation, including recurrent layers and convolution layers, and derive their corresponding CCC and CBWC. We further provide the grouped-CCC and grouped CBWC for group normalization.

\subsection{Column-Centered Weight Transformation}
\label{apx:CBWT}

\begin{wrapfigure}{r}{0.25\linewidth}
\vspace{-3ex}
        \centering
        \includegraphics[height=17.5ex]{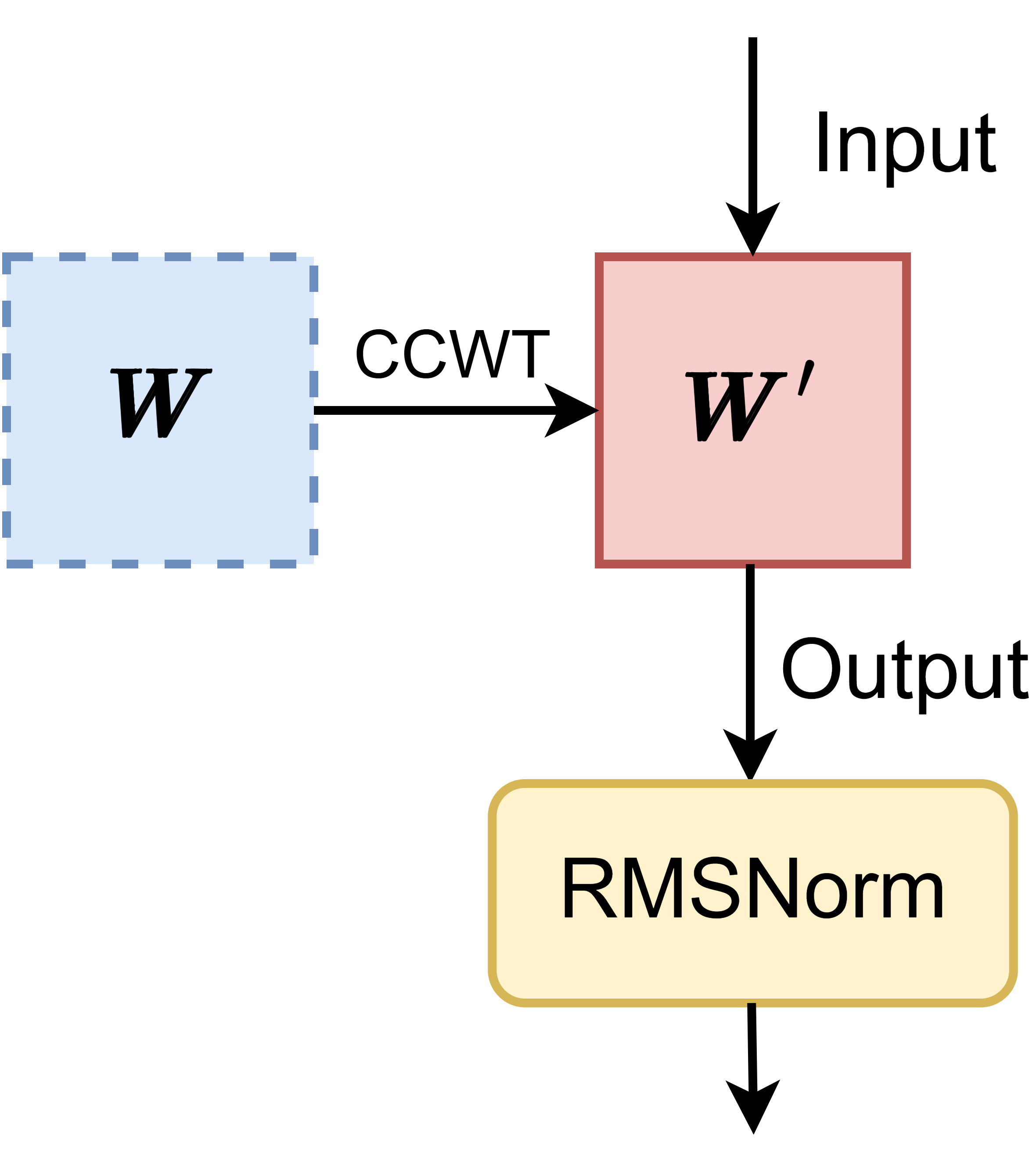}
    \caption{Sketch map of CCWT.}
    \vspace{-3ex}
\end{wrapfigure}

To achieve the CCC of a general linear layer in inference and ensure a subsequent LN foldable, we propose \textit{column-centered weight transformation}. The aim of CCC is to fold centering operation (in Eqn.\ref{eqn:centering}) into linear layer. Notice that given a layer input vector $\rvx = [x_1, x_2, \dots, x_n]^\top \in \mathbb{R}^n$, the centering operation can be written into the form of $\widetilde{\rvx} = \left(\mI - \frac{1}{m}\boldsymbol{1}^\top_m\boldsymbol{1}_m\right) \rvx$, where the matrix $\left(\mI - \frac{1}{m}\boldsymbol{1}^\top_m\boldsymbol{1}_m\right)$ can be moved into the linear layer.

We have the definition below.
\begin{definition}
    [Column-Centered Weight Transformation (CCWT)]
    \emph{Column-centered weight transformation} aims to apply transformation on the weight matrix to ensure column-centered constraint. We construct a specific transformation $\varphi$, changes $\mW$ into $\mW'$, as:
    \begin{equation}
    \label{eqn:ccwt-linear}
        \mW' = \varphi(\mW) = \left(\mI - \frac{1}{m}\boldsymbol{1}^\top_m\boldsymbol{1}_m\right)\mW,
    \end{equation}
    where $m$ is the output neuron number. 
\end{definition}
Apparently, CCWT ensures that the weight for each input in the transformed matrix $\mW'$ has zero mean, thus guarantees CCC and ensures a zero-mean output of the layer.
For different general linear layers, the transformation $\Psi$ takes different forms, but the essence of its construction based on column-centered constraint will not change.

CBWC is a special transformation on the parameter, folding the parameter space onto that of CCC while maintaining mathematical equivalence of the original model. With the construct of CCWT, it becomes easier to get the form of CBWC. For a given CCWT, we transfer it into CBWC by using a proxy weight in the place of $\mW'$. Instead of directly replacing the computational weight $\mV$, we ensure $\mV$ is under column-centered constraint while $\mW$ is calculated from $\mV$ by CCWT.

\subsection{General Linear Layer}
In the main paper, we introduce the concept of general linear layer, which apply linear transformations to their inputs. Recurrent layers with shared weights in RNNs and convolution layers in CNNs are all general linear layers as they satisfy additivity and homogeneity. We formalize the definition of general linear layer as follows.
\begin{definition}[General Linear Layers]
    A layer parameterized by $\vtheta$ with transformation $\rvh(\rvx; \vtheta)$, where $\rvx \in \R^n$, is a \emph{general linear layer} if the mapping 
    $\rvh(\cdot; \vtheta)$ satisfies additivity and homogeneity:
    \begin{itemize}
        \setlength{\itemsep}{0pt}
        \item $\rvh((\rvx_1+\rvx_2);\vtheta)=\rvh(\rvx_1;\vtheta) + \rvh(\rvx_2;\vtheta),\forall \rvx_1,\rvx_2\in\R^n.$
        \item $\rvh(c\rvx;\vtheta) = c\ \rvh(\rvx;\vtheta),\forall\ \rvx\in\R^n,\ c\in\R.$
    \end{itemize}
    Namely, the transformation $\rvh(\rvx;\vtheta)$ satisfies linearity, and is a linear transformation of $\rvx$. 
\end{definition}
Such layers can be explicitly expressed as the form of $\rvh(\rvx; \vtheta) = \mW(\vtheta) \rvx$, allowing for direct application of CCC and CBWC.

\subsubsection{Recurrent Layer}
\label{proof:rnn}
Despite linearity, the recurrent neural network is different from the original linear layer with its recurrent connection and shared weight matrix. 
Due to the fact that our constraints are independent of input and output, the parameter sharing is excluded from our consideration. As for the recurrent connection, the weights for ordinary input and recurrent input can be seen as two linear layers. 

For the $l$-th layer and $t$-th time step of a recurrent neural network, 
we define the input as $\rvx_t^{l-1}\in\R^{d_{l-1}}$, the recurrent input as $\rvh_{t-1}^{l}\in\R^{d_{l}}$ and the output of the hidden layer as $\rvc_{t}^{l}$. 
We have the weight matrix 
$\mW_v=[\rvw^v_1,\cdots,\rvw^{v}_{d_{l-1}}]\in\R^{d_l\times d_{l-1}}$ 
and 
$\mW_h=[\rvw^{h}_{1},\cdots,\rvw^{h}_{d_l}]\in\R^{d_l\times d_{l}}$, 
where $\rvw^v_i, \rvw^h_i\in\R^{d_{l-1}}$ and  
which is shared between all time steps. 
We define $\mW=[\mW_v,\mW_h]$ and input $\rvx = [(\rvx_{t}^{l-1})^\top,(\rvh_{t-1}^{l})^\top]^\top$. 
We have recurrent layer in form of a linear transformation as follows:
\begin{equation}
    \rvc_t = \rvh_{\mathrm{rnn}}(\rvx;\mW)=\mW_v\rvx_t^{l-1}+\mW_h\rvh_{t-1}^l,
\end{equation}
where we have $\rvw_i({\mW})=\rvw_i^v$ when $i=1,\cdots,d_{l-1}$ and $\rvw_i({\mW})=\rvw_{i-d_{l-1}}^h$ when $i=(d_{l-1}+1),\cdots,(d_{l-1}+d_l)$.
Therefore, we have the constraint:
\begin{equation}
\begin{aligned}
    \label{constrain:rnn}
    \mW_0\in\Theta_\mathrm{rnn}=\Big\{\ \mW: (\rvw^v_j)^\top\cdot\boldsymbol{1}_{d_l}=\sum^{d_{l}}_{i=1} w^v_{i,j}=0,\ (\rvw^h_j)^\top\cdot\boldsymbol{1}_{d_l}=\sum^{d_{l}}_{i=1}w^h_{i,k}= 0,\ \\j=1,2,\dots,d_{l-1},\ k=1,2,\dots,d_{l}\ \Big\}.
\end{aligned}
\end{equation}

Under the constraint of \Eqn\ref{constrain:rnn}, we have the mean of output with:
\begin{equation}
\begin{aligned}
    \mu_t^c &
    = \frac{1}{d_l}\sum_{i=1}^{d_l}\left( \sum_{j=1}^{d_{l-1}} w^v_{i,j} x_j + \sum_{k=1}^{d_{l}}w^h_{i,k} h_k \right)\\
    & =\frac{1}{d_l}\left( \sum_{j=1}^{d_{l-1}}\left(\sum_{i=1}^{d_l}w^v_{i,j}\right) x_j + \sum_{k=1}^{d_{l}}\left(\sum_{i=1}^{d_l}w^h_{i,k}\right) h_k \right)\\
    & = \frac{1}{d_l}\left( \sum_{j=1}^{d_{l-1}} 0\cdot x_j + \sum_{k=1}^{d_{l}}0\cdot h_k \right) = 0.
\end{aligned}
\end{equation}

Thus, for the shared weight matrix for both input from the last layer and from the last time step, applying constraints on them centralizes the output of the hidden layer.
We have the transformation $\varphi_{\mathrm{rnn},v},\varphi_{\mathrm{rnn},h}$ of the CCWT on the recurrent neural network, as follows:

\begin{equation}
    \begin{aligned}
        \mW'^v &= \varphi_{\mathrm{rnn},v}(\mW^v)=(I-\frac{1}{m_v}\boldsymbol{1}_{m_v}\boldsymbol{1}_{m_v}^\top)\mW^v \\
        \mW'^h &= \varphi_{\mathrm{rnn},h}(\mW^h)=(I-\frac{1}{m_h}\boldsymbol{1}_{m_h}\boldsymbol{1}_{m_h}^\top)\mW^h.
    \end{aligned}
\end{equation}

We have the corresponding CBWC for the recurrent layer as below:
\begin{equation}
    \begin{aligned}
        \mV^v &= \varphi_{\mathrm{rnn},v}(\mW^v)=(I-\frac{1}{m_v}\boldsymbol{1}_{m_v}\boldsymbol{1}_{m_v}^\top)W^v, \\
        \mV^h &= \varphi_{\mathrm{rnn},h}(\mW^h)=(I-\frac{1}{m_h}\boldsymbol{1}_{m_h}\boldsymbol{1}_{m_h}^\top)\mW^h,\\
        \frac{\partial\mathcal{L}}{\partial\mW^v}&= \phi_{\mathrm{rnn},v}\left(\frac{\partial\mathcal{L}}{\partial\mV}\right)=\left(\mI - \frac{1}{m_v}\boldsymbol{1}_{m_v}\boldsymbol{1}_{m_v}^\top\right)^\top \frac{\partial\mathcal{L}}{\partial\mV},\\
        \frac{\partial\mathcal{L}}{\partial\mW^h}&= \phi_{\mathrm{rnn},h}\left(\frac{\partial\mathcal{L}}{\partial\mV}\right)=\left(\mI - \frac{1}{m_h}\boldsymbol{1}_{m_h}\boldsymbol{1}_{m_h}^\top\right)^\top \frac{\partial\mathcal{L}}{\partial\mV}.
    \end{aligned}
\end{equation}

\subsubsection{Convolution Layer}
\label{proof:cnn}
Under the circumstances of the convolution layer, the convolution kernel can be regarded as a combination of a set of shared weights.
All of them should fulfill the constraint of the linear layer. We hence use a vector to denote the elements among different channels among the kernel.

We denote the input tensor $\rvx \in \mathbb R ^{d_{l-1}\times h\times w}$ and the output tensor $\mH \in \mathbb R ^{d_{l}\times h'\times w'}$. We have convolution kernels $\mW \in \mathbb R ^{d_l\times d_{l-1}\times F_h \times F_w}$.

For every channel of output tensor $\mH_i\in\mathbb R^{h'\times w'}\ (i=1,\dots,d_l)$, every channel of input tensor  $\rvx_j\in\mathbb R^{h\times w}\ (j=1,\dots,d_{l-1})$ and corresponding convolution kernel $\mathbf w_{i,j}\in \mathbb R ^{F_h \times F_w}\ (i=1,\dots,d_l,\ j=1,\dots,d_{l-1})$, we have:
\begin{equation}
    \mH_k =  \sum ^{d_{l-1}}_{t=1} \mathbf x_t*\mathbf w_{k,t}.
\end{equation}
For more specifics, we have:
\begin{equation}
    h_{k,i,j}=\sum^{d_{l-1}}_{t=1}\sum^{F_h}_{a=1}\sum^{F_w}_{b=1}x_{t, (is-s-p+a-1),(js-s-p+b-1)}\cdot w_{k,t,a,b},
\end{equation}
where the subscript sequence refers to the order number of the element in each dimension.
Obviously, this equation can be written into the form of a linear transformation:
\begin{equation}
    \mH = \rvh_{\mathrm{cnn}}(\rvx;\mW)= \sum^{d_{l-1}}_{t=1}\sum^{h}_{a=1}\sum^{w}_{b=1} \rvw_{t,a,b}(\mW)\cdot x_{t,a,b}.
\end{equation}
Therefore, we have the constraint:
\begin{equation}
\label{constrain:cnn}
    \mW_0\in\Theta_\mathrm{cnn} = \left\{\ \mW : \sum^{d_{l}}_{i=1}\rvw_{i,j}=\pmb 0,\ j=1,2,\dots,d_{l-1}\ \right\}.
\end{equation}

Due to the convolution operation, we have $a*(b+c)=a*b+a*c$. Thus, under the constraint of Eqn.\ref{constrain:cnn}, we have:
\begin{equation}
    \begin{aligned}
\mu_h = \frac{1}{d_l}\sum ^{d_l}_{i=1} \mH_i &= \frac{1}{d_{l}}\sum ^{d_l}_{i=1}\sum ^{d_{l-1}}_{j=1} \mathbf x_j* \rvw_{i,j}= \frac{1}{d_l}\sum ^{d_{l-1}}_{j=1} \mathbf x_j * \left(\sum ^{d_{l}}_{i=1} \rvw_{i,k}\right)\\
&= \frac{1}{d_l}\sum ^{d_{l-1}}_{j=1} \mathbf x_j *  \pmb 0 =  0.
\end{aligned}
\end{equation}

It thus can be seen that the column-centered constraint on the convolution kernels achieves the effect of the centering of the LN.
We have the transformation $\varphi_{cnn}$ of the CCWT on convolution layers, as follows:
\begin{equation}
 \mW = \varphi_\mathrm{cnn}(\mW)=(I-\frac{1}{h\times w}\boldsymbol{1}_{h\times w}^\top\boldsymbol{1}_{h\times w})\mW.
\end{equation}

To be noted, the tensor $\mW$ here is a four-dimension tensor. The transformation here is to do centering on its second dimension.

We have a corresponding CBWC for the convolution layer as below:
\begin{equation}
    \begin{aligned}
         \mV &= \varphi_\mathrm{cnn}(\mW)=(I-\frac{1}{h\times w}\boldsymbol{1}_{h\times w}^\top\boldsymbol{1}_{h\times w})\mW.\\
        \frac{\partial\mathcal{L}}{\partial \mW}&= \phi_\mathrm{cnn}\left(\frac{\partial\mathcal{L}}{\partial \mV}\right)=\left(\mI - \frac{1}{h\times w}\boldsymbol{1}_{h\times w}\boldsymbol{1}_{h\times w}^\top\right)^\top \frac{\partial\mathcal{L}}{\partial \mV}.
    \end{aligned}
\end{equation}

\subsection{Grouped Column-Centered Constraint For GN}
\label{apx:gccc}
We extend the conclusion to Group Normalization (GN) \citep{Wu_2018_ECCV}.
Group normalization is first defined on the channel dimension for convolution input $\mX \in \R^{d\times h \times w}$. So the Group Normalization here is more similar to grouped Layer Normalization, with the definition below:

\begin{definition}
    [Group Normalization (GN)]
    Suppose the number of groups is $g$, and $ d = g \times c $. Let $ \rvx = [ \vz_1^\top, \dots, \vz_g^\top]^\top $, where $ \vz_i = [ z_{i1}, \dots, z_{ic} ]^\top , (i=1,\dots,g) $. Assuming $ \rvx = [x_1, \dots, x_d]^\top $, we denote $ z_{ij} = x_{(i - 1) \times c + j} $. Let $ \hat\rvx = \mathit{GN}(\rvx) $, where $ \mathit{GN}(\cdot) $ denotes the \emph{group normalization}. GN can be calculated by $ \mu_i = (z_{i1} + \dots + z_{ic})/c $, $ \sigma_i^2 = [(z_{i1}-\mu_i)^2 + \dots + (z_{ic} - \mu_i)^2]/c $, and then $ \hat z_{ij} = (z_{ij} - \mu_i)/\sigma_i $. Thus, we have $ \hat\rvx = [ \hat\vz_1^\top, \dots, \hat\vz_g^\top]^\top $, where $ \hat\vz_i = \mathit{LN}(\vz_i), (i=1,\dots,g) $.
\end{definition}

For every input, we divide the neurons into groups and apply normalization in every group. Thus the centering step in this normalization is to ensure the output sum of all neurons in each group is zero. For a sampled input $\mathbf h = [h_1,h_2,\dots,h_m]^\top$, for $g$ groups and $c$ channels in every group ($g\times c = m$), we have:
\begin{equation}
    \mu _{hj} =\frac{1}{c}\sum^{c}_{i=1} h_{ji} = 0\quad(j=1,\dots,g).
\end{equation}

For a given general linear layer The parameter $\vtheta_0$ is under column-centered constraint of GN if $\vtheta_0$ satisfies:
\begin{equation}
    \vtheta_0 \in \Sigma = \left\{\vtheta : \rvw_i^\top(\vtheta)\cdot\boldsymbol{1}_{(c,k\times c)}=0,\ k=1,2,\dots,n\ \right\}.
\end{equation}

Take the linear layer $\rvh=\mW\rvx$ as an example, we have the weight matrix $\mW=[\rvw_1,\rvw_2,\cdots,\rvw_{d_{l-1}}]\in \R^{d_{l} \times d_{l-1}}$, where $\rvw_i\in\R^{d_{l}},i=1,\cdots,d_{l-1}$. Similarly, we have $\rvw_i(\mW)=\rvw_i$.
Therefore, the column-centered constraint for GN can be expressed as:
\begin{equation}
\label{eq:57}
    \mW_0 \in \Sigma_\mathrm{GN}= \left\{\ \mW :\sum^{c}_{k=1}w_{j,(k + c \times i)}=0,\ i=1,2,\dots,g,\  j=1,2,\dots,d\ \right\}.
\end{equation}

Given $\mathbf h = \mathbf W  \rvx $, for the $i$-th neuron output $h_i$ in the $j$-th group of $\mathbf h$, we have:
\begin{equation}
    h_{i}  = \sum ^{d}_{k=1} w_{i,j} \cdot x_j.
\end{equation}

Under the constraint of \Eqn\ref{eq:57}, we have:
\begin{equation}
    \begin{aligned}
\mu _{hj} =\frac{1}{c}\sum^{c}_{i=1} h_{ji} 
=  \frac{1}{c}\sum ^c_{i=1} \sum ^d_{j=1} w_{i,j}\cdot x_j = \frac{1}{c}\sum ^d_{j=1} \left(\sum ^c_{i=1} w_{i,j}\right)x_j  
= \frac{1}{c}\sum ^d_{j=1} 0 \cdot x_j  = 0.
\end{aligned}
\end{equation}

Thus, we replace the centering step of GN with a grouped column-centered constraint.
To be mentioned, the core idea of designing a constraint is to ensure every group of input weight has zero mean.

For the transformation $\varphi_{\mathrm{GN}}$ of the CCWT on a normal linear layer under GroupNorm, we have:
\begin{equation}
    \mW = \varphi_{\mathrm{GN}}(\mW)=(\mI - \mA)W.
\end{equation}
$A$ is a matrix that we construct with the equation below:
\begin{equation}
    \mA = \mI - \frac{1}{c}\sum_{k=0}^{d-1}\boldsymbol{1}_{(c,k\times c)}^\top \boldsymbol{1}_{(c,k\times c)},
\end{equation}
where $\boldsymbol{1}_{(c,k\times c)}$ refers to a vector whose elements are all zero except that the $(k\times c)$-th element to $(k\times c +c)$-th element are ones. Specifically, $A$ is a matrix with its diagonal arrayed with $c\times c$ matrices of ones.

We have a corresponding grouped CBWC for the linear layer as below:
\begin{equation}
    \begin{aligned}
        \mV &= \varphi_{\mathrm{GN}}(\mW) = (\mI - \mA)\mW\\
        \frac{\partial\mathcal{L}}{\partial\mW}&= \phi_{\mathrm{GN}}\left(\frac{\partial\mathcal{L}}{\partial\mV}\right)=(\mI - \mA)^\top \frac{\partial\mathcal{L}}{\partial\mV}.
    \end{aligned}
\end{equation}

It is worth mentioning that InstanceNorm~\citep{ulyanov2016instance} is simply the special case where the number of groups $g$ equals the number of channels.

In addition, folding fails when the number of parameters influencing a single output element is smaller than the group size $g$ (which can occur in certain shallow or highly grouped settings). In such cases, the grouped column-centered constraint becomes over-constrained and cannot be satisfied.

Beyond GroupNorm, any normalization (or post-processing) that consists solely of additive (among elements of the sample) and scalar multiplication operations applied per sample can be folded into the preceding general linear layer using the same principle.

\subsection{Discussion of Column-Based Weight Centering and Weight Decay Technology}
According to the definition of weight decay, this technology involves adding a penalty term based on the L2 norm of the model weights to the original loss function, thereby encouraging the optimizer to favor smaller weight values and improving the model's generalization performance.

In our method, the model contains two weight matrices. This is because CBWC is implemented by re parameterization: the matrix ($\mV$) used for forward/backward calculation is different from the updated parameter matrix ($\mW$), but the two matrices are connected by deterministic and differentiable transformations (as described in Definition~\ref{def:cb}). The L2 norm of the calculation matrix $\mV$ and the storage parameter matrix $\mW$ are slightly different. Therefore, two implementation methods of weight decay are derived from this transformation.

Standard implementations (including PyTorch) apply weight decay to the trainable parameters ($\mW$ in our case). However, whether weight decay should technically be applied to $\mV$ or $\mW$ is an interesting and subtle question. Applying it to $\mW$ preserves the classic regularization interpretation (penalizing the magnitude of stored parameters), while applying it to $\mV$ would more directly penalize the magnitude of the pre-activation weights.

Both choices are defensible, and the community has not reached a consensus of a better choice, which is also orthogonal to the core contribution of our paper. Our discussed implementation (decay on $\mW$) follows PyTorch’s default behavior and does not alter the intended regularization effect in practice.

\section{Algorithm for Detecting Foldable LN}

\subsection{Parallel Connection}
\label{proof:res}
In this section, we prove how to fold an LN while the sample meets residual structure during backtracking and provide the reason why other parallel connection fails to fold a downstream LN.
\subsubsection{Residual Structure}
Here we consider a two branch residual structure as an example. The conclusion remains unchanged when the number of branches increases.
For a residual structure, we define the input $\rvx\in \R^n$ and the output $\rvy\R^m$, as shown below:
\begin{equation}
    \rvy = \mathcal{F}(\rvx;\vtheta_F) + \mathcal{G}(\rvx;\vtheta_G),
\end{equation}
where both $\mathcal{F}(\cdot;\vtheta_F)$ and $\mathcal{G}(\cdot;\vtheta_G)$ are subnetworks, and $\vtheta_F$ and $\vtheta_G$ are learnable parameters.
Due to the complexity of $\mathcal{F}(\cdot)$, it is intuitively difficult to construct a constraint on this function to eliminate the mean of $\mathcal{G}(\rvx)$. Therefore, we treat the two terms separately and apply constraint based on their content each.

To eliminate the mean of $\rvy$, we have $\rvy'=(I-\frac{1}{m}\boldsymbol{1}_m^\top\boldsymbol{1}_m)\rvy$.
According to the distributive law of multiplication, we have:
\begin{equation}
    \rvy' = (I-\frac{1}{m}\boldsymbol{1}_m^\top\boldsymbol{1}_m)\rvy = (I-\frac{1}{m}\boldsymbol{1}_m^\top\boldsymbol{1}_m)\mathcal{F}(\rvx;\vtheta_F) + (I-\frac{1}{m}\boldsymbol{1}_m^\top\boldsymbol{1}_m)\mathcal{G}(\rvx;\vtheta_G),
\end{equation}
Such that, the residual structure has zero-mean output if the output of each branch has zero mean.
\subsubsection{Other Parallel Connection}
Other parallel connections do not satisfy linearity, such that we cannot transfer the zero-mean property to each branch. 

Take concatenation as an example. When all composite vectors for concatenation have zero mean, stacking vectors along the axis produces a zero-mean tensor. But this zero-mean property of composite vectors and output vectors are not mathematically equivalent in back propagation, thus preventing the downstream LN from folding.

To put this another way, for a concatenation operation, we define the input $\rvx_i\in \R^{n_i},i=1,\cdots,n$ and the output $\rvy\R^m$, where $m=\sum_{i=1}^nn_i$, as shown below:
\begin{equation}
    \rvy=f(\rvx_1,\cdots,\rvx_n)=[\rvx_1,\cdots,\rvx_n].
\end{equation}
To eliminate the mean of $\rvy$, we have $\rvy'=(I-\frac{1}{m}\boldsymbol{1}_m^\top\boldsymbol{1}_m)\rvy$. It is obvious that:
\begin{equation}
    \rvy' = (I-\frac{1}{m}\boldsymbol{1}_m^\top\boldsymbol{1}_m)\rvy \ne [(I-\frac{1}{m}\boldsymbol{1}_m^\top\boldsymbol{1}_m)\rvx_1,\cdots,(I-\frac{1}{m}\boldsymbol{1}_m^\top\boldsymbol{1}_m)\rvx_n]. 
\end{equation}
Such that zero mean of each composite vector does not equal to the zero mean of the output vector. Therefore, we cannot fold the downstream LN into the concatenation.

\subsection{A Stricter Criteria of Foldable LN}
\label{sec:stricter-critieria}
As proposed in the previous section, we can replace the LN with RMS\-Norm if layers corresponding to the upstream vertices in zero-mean graph are all general linear layers. However, the zero-mean output will be propagated indiscriminately to all connected layers of the upstream layers, which may lead to nonequivalent output and unexpected results. 

To ensure mathematical equivalence, we expect all the layers requiring CBWC to affect only LNs. Here we define the \textit{affected layer}.

\begin{definition}
    [Affected layer]
    Given a neural network $G=(V,E)$ and the zero-mean graph $G_z=(V_z,E_z)$ for $v_{\mathrm{LN}}$. We initialize $V_0=\{u\notin V_z\mid \forall v\in (V_z\setminus V_{\mathrm{LN}}),\langle v,u\rangle\in E\}$ and $A_0=\emptyset$. For each iteration $k=0,1,2,\cdots$:
    \begin{enumerate}
        \item Let $A_{k+1}=(V_k\cap( V_c\cup V_l\cup V_-))$, update affected layers.
        \item Let $V_{k+1}=\{u\mid v\in (V_k\setminus ( V_c\cup V_l\cup V_-)),\langle u,v\rangle\in E\}$, trace the next layer through forward pass.
    \end{enumerate}
    We define $A=\bigcup^\infty_{k=0}A_k $ as affected layer.
\end{definition}

If the affected layers only include LN, then the zero-mean activation in the zero-mean graph safely ensures a foldable LN.
We note that the commonly used models nowadays, such as transformers, all meet this requirement.

\subsection{Self-Attention Module and Transformer Structure}

\label{sec:gpt}

\subsubsection{Self-Attention Module}

To be mentioned, self-attention module can be seen as a pile of layers. We make use of its posterior linear component and thus construct the constraint on it. 

For a sampled input $\rvx\in \mathbb R^{n\times d}$, we apply three different learnable weight matrices $\mQ,\mK \in \mathbb R^{d\times d_k},\ \mV \in \mathbb R^{d \times d_v}$ and have three input matrices $\mH_Q,\mH_K \in \mathbb R^{n\times d_k},\ \mH_V \in \mathbb R^{n \times d_v}$ with:
\begin{equation}
    \mH_Q = \rvx\cdot\mQ,\quad
    \mH_K = \rvx\cdot\mK,\quad
    \mH_V = \rvx\cdot\mV.
\end{equation}

According to the definition of scaled dot-product attention, we have:
\begin{equation}
    \operatorname{Attention}(\mH_Q,\mH_K,\mH_V) = \operatorname{softmax}\left(\frac{\mH_Q\mH_K^\top}{\sqrt{d_k}}\right)\mH_V.
\end{equation}
When this expression is expanded, it is written as:
\begin{equation}
    \operatorname{Attention}(\rvx;\mQ,\mK,\mV) = \operatorname{softmax}\left(\frac{\rvx\cdot\mQ\cdot\mK^\top\cdot\rvx^\top}{\sqrt{d_k}}\right)\rvx\cdot\mV.
\end{equation}
To simplify, we denote:
\begin{equation}
    \mB = \operatorname{softmax}\left(\frac{\rvx\cdot\mQ\cdot\mK^\top\cdot\rvx^\top}{\sqrt{d_k}}\right)\rvx\in\R^{n\times d}.
\end{equation}
We can see this module as a linear transformation $\rvh=\rvh_{\mathrm{trans}}(\mB,\mV) = \mB \mV$, where $\mB$ is the input. We denote $\mB=[\rvb_1^\top,\cdots,\rvb_d^\top]^\top$, where $\rvb_i^\top\in\R^n,i=1,\cdots,d$ and $\mV= [\rvv_1,\cdots,\rvv_{d}]^\top\in\R^{d\times{d_v}}$, where $\rvv_i\in\R^{d_v},i=1,\cdots,d$. 

We have $\rvh=\rvh_{\mathrm{trans}}(\mB,\mV)=\sum_{i=1}^d \rvv_i\cdot\rvb_i$ and $\rvw_i(\mV)=\rvv_i$. Therefore, we have the constraint:
\begin{equation}
\label{constrain:attention}
    \mW_0\in\Gamma_\mathit{trans} = \left\{\ \mW : \rvv_i^\top\cdot \pmb 1_{d_v}=\sum^{d_{v}}_{k=1}v_{j,k}= 0,\ i=1,2,\dots,d_{l-1}\ \right\}.
\end{equation}

By \Eqn\ref{constrain:attention}, we have
\begin{equation}
    \begin{aligned}
\mu _a &=\sum^{d_v}_{b=1} (\mB\mH_V)_{(a,b)}
= \sum^{d_v}_{b=1} \sum^{n}_{j=1} b_{a,j} \cdot {\mH_V}_{(j,b)}= \sum^{d_v}_{b=1} \sum^{n}_{j=1} b_{a,j} \left(\sum^{d}_{k=1}x_{j,k} \cdot v_{k,b}\right)\\
&= \sum^{d_v}_{b=1} \sum^{n}_{j=1} \sum^{d}_{k=1}b_{a,j} \cdot x_{j,k} \cdot v_{k,b}=  \sum^{n}_{j=1} \sum^{d}_{k=1}b_{a,j} \cdot x_{j,k}\left(\sum^{d_v}_{b=1}v_{k,b}\right)\\
&=  \sum^{n}_{j=1} \sum^{d}_{k=1}b_{a,j} \cdot x_{j,k} \cdot 0 =  0.
\end{aligned}
\end{equation}

Accordingly, we have the transformation $\varphi_{\mathrm{trans}}$ of the CCWT on self-attention modules, as follows:

\begin{equation}
\mW^v = \varphi_{\mathrm{trans}}(\mW^v)=(I-\frac{1}{m_v}\pmb 1_{m_v}^\top\pmb 1_{m_v})\mW^v.
\end{equation}

We have a corresponding CBWC for the attention module as below:
\begin{equation}
    \begin{aligned}
        \mW^v &= \varphi_{\mathrm{trans}}(\mW^v)=(I-\frac{1}{m_v}\pmb 1_{m_v}^\top\pmb 1_{m_v})\mW^v.\\
        \frac{\partial\mathcal{L}}{\partial\mW}&= \phi_{\mathrm{trans}}\left(\frac{\partial\mathcal{L}}{\partial\mV}\right)=\left(\mI - \frac{1}{m_v}\pmb 1_{m_v}\pmb 1_{m_v}^\top\right)^\top \frac{\partial\mathcal{L}}{\partial\mV}.
    \end{aligned}
\end{equation}

Moreover, multi-head attention modules contain an extra linear layer at the last. Simply applying CCC and CBWC of the linear layer onto it ensures zero-mean output.

\subsubsection{Post-LN Transformer}

\begin{figure}[h]
    \vspace{-1ex}
    \centering
    \begin{subfigure}[t]{0.30\textwidth}
        \centering
        \includegraphics[height=6.5ex]{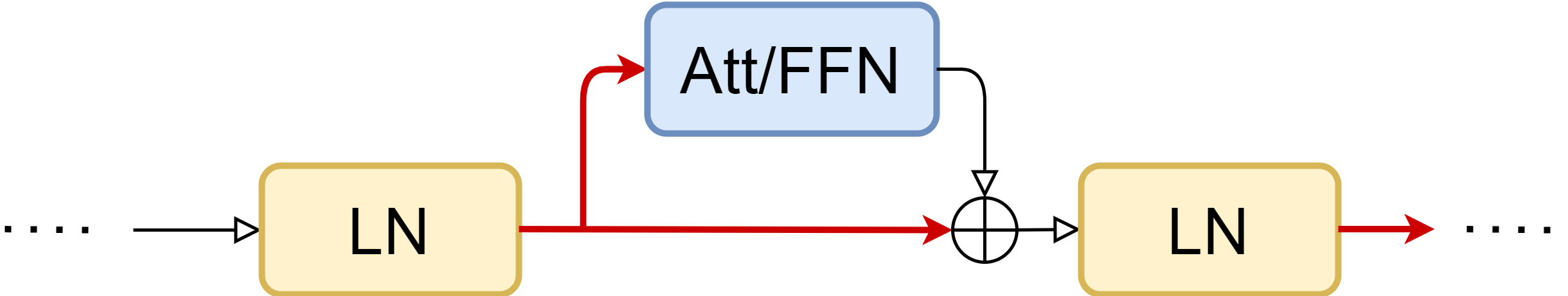}
        \caption{Post-LN transformer structure.}
    \end{subfigure}   
    \begin{subfigure}[t]{0.30\textwidth}
        \centering
    \includegraphics[height=6.5ex]{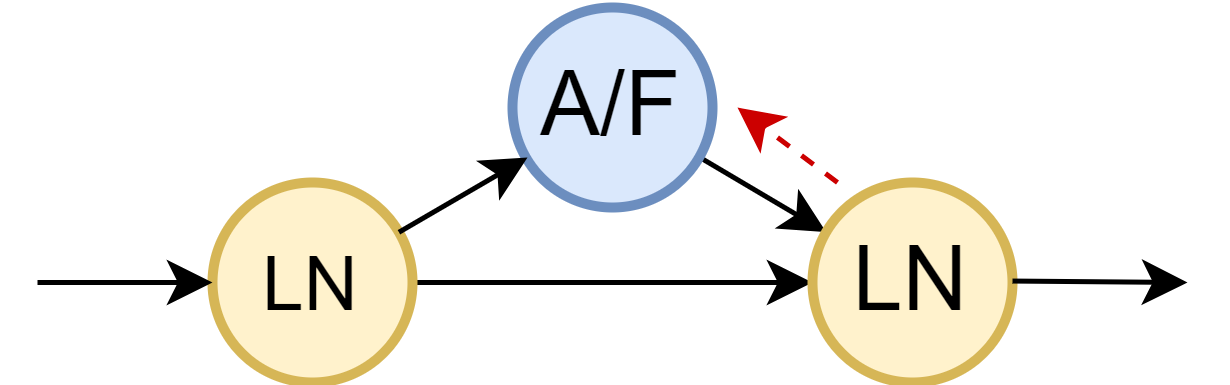}
        \caption{Zero-mean graph.}
    \end{subfigure}   
    \begin{subfigure}[t]{0.30\textwidth}
        \centering
    \includegraphics[height=6.5ex]{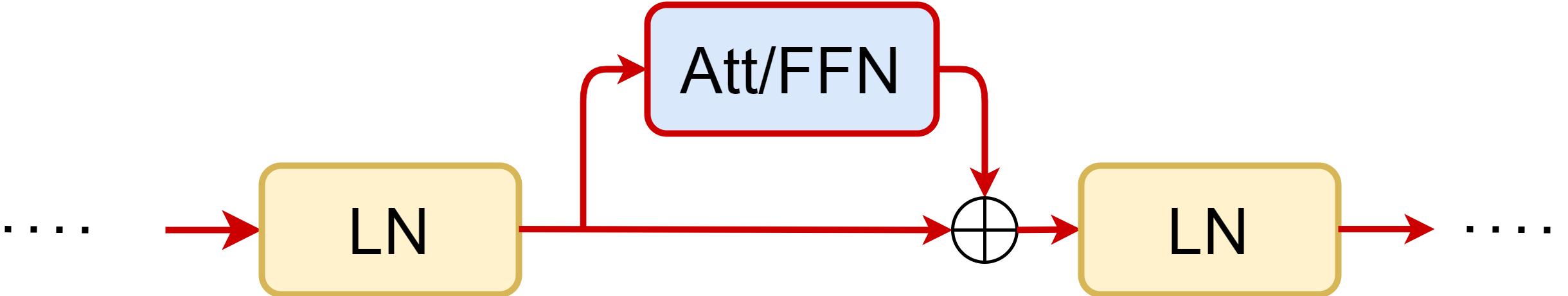}
        \caption{Our method on Post-LN.}
    \end{subfigure}   
    \vspace{-1ex}
\caption{The proof and application of our method on Post-LN. `Att' and `FFN' refer to the Attention layer and the Feed-Forward Network, which are both general linear layers.}
    \vspace{-2ex}
\end{figure}

For the post-LN transformer, the residual structure connects an LN and an self-attention module or a feed-forward network layer. Obviously, all the LNs are foldable.

\subsubsection{Pre-LN Transformer}

\begin{figure}[h]
    \vspace{-1ex}
    \centering
    \begin{subfigure}[t]{0.30\textwidth}
        \centering
        \includegraphics[height=5.5ex]{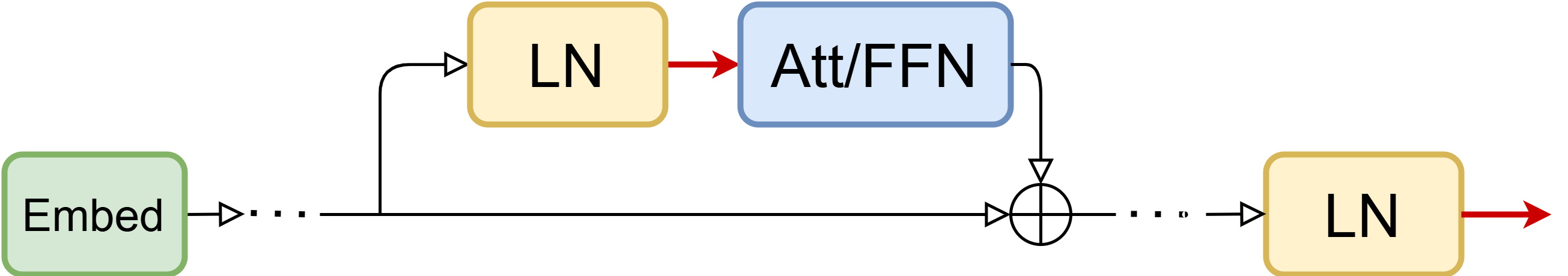}
        \caption{Pre-LN transformer structure.}
    \end{subfigure}   
    \begin{subfigure}[t]{0.30\textwidth}
        \centering
        \includegraphics[height=8ex]{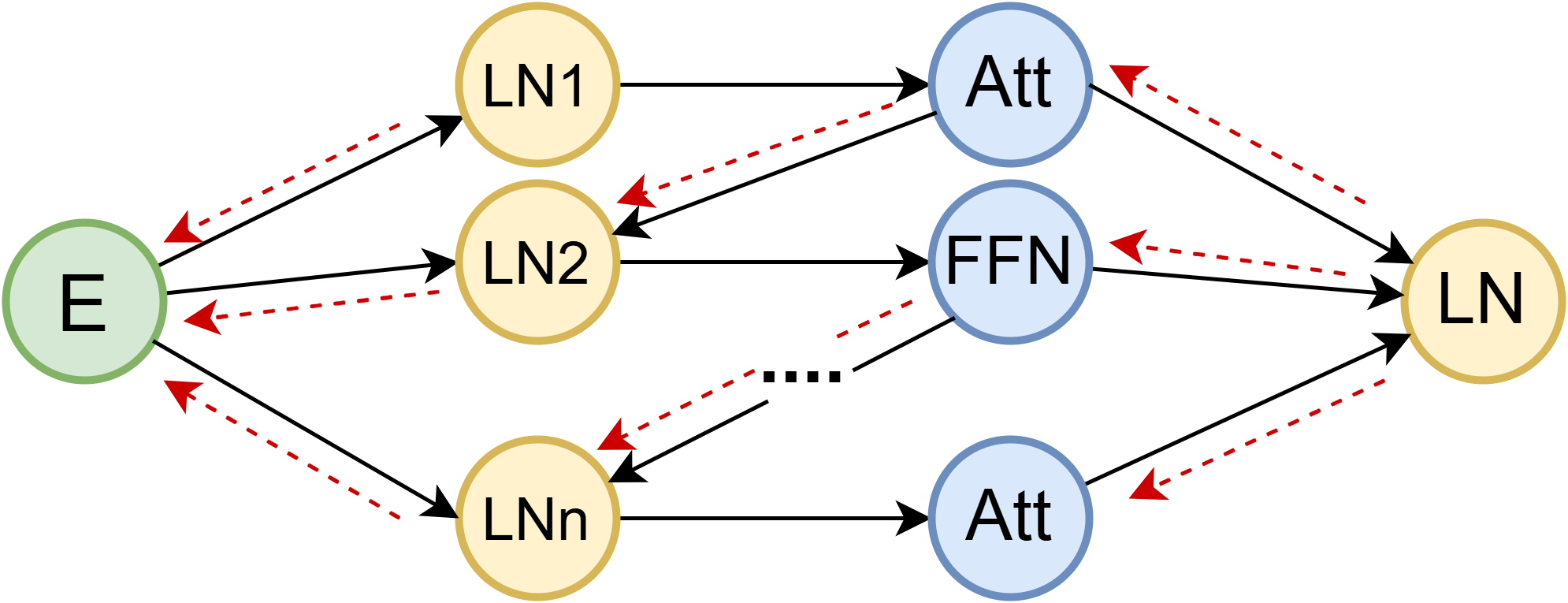}
        \caption{Zero-mean graph.}
    \end{subfigure}   
    \begin{subfigure}[t]{0.30\textwidth}
        \centering
        \includegraphics[height=5.5ex]{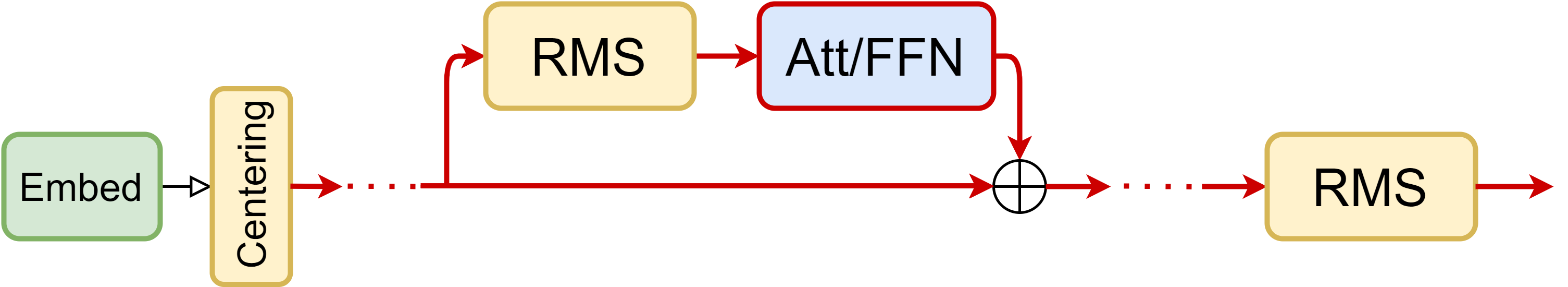}
        \caption{Our method on Pre-LN.}
    \end{subfigure}   
    \vspace{-1ex}
\caption{The proof and application of our method on Pre-LN. `Att' and `FFN' refers to the Attention layer and Feed-Forward Network, which are both general linear layers. `Embed' refers to the Embedding layer, which it is not.}
    \vspace{-2ex}
\end{figure}

Although the embedding block does not provide zero-mean output, an auxiliary centering operation can be inserted after it. For most models like GPT2, there is an LN after the last transformer block. In this case, all downstream LNs can be folded by applying CBWC to the attention and feed-forward layers. Therefore, all LNs in the pre-LN Transformer can be folded under the practical extension. This is also the main idea of the previous paper \cite{jiang2023pre}.

\section{Implementation Details of Algorithm 1}
\label{apx:f-ln-table}
\subsection{Foldable LN in Common Models}

Here, we list 11 common models and the number of LN and foldable LN in Table~\ref{table:fln-num}. Notice that some models follow the pre-norm Transformer architecture and some LNs can be folded under the practical extension. The models we mentioned here are 
GPT-2~\citep{Radford2019GPT2}, BERT~\citep{2019_ACL_Devlin}, ViT~\citep{2021_ICLR_Dosovitskiy}, Phi~\citep{2023_Phi_Textbooks_Gunasekar}, OPT~\citep{zhang2022opt}, BLOOM~\citep{2022_BLOOM_Scao} and VMamba~\citep{liu2024vmamba}.

\begin{table}[h]
\caption{Number of LN and foldable LN in 11 common models.}
\label{table:fln-num}
    \vspace{1ex}
\centering
\begin{tabular}{@{}lcccccc@{}}
\toprule
Model  & Total & Foldable & Percentage & Foldable (Practical) & Percentage \\ \midrule
GPT-2  & 25    & 0        & 0          & 25                          & 100.00\% \\
BERT   & 25    & 24       & 96.00\%    & 25                          & 100.00\% \\
ViT    & 25    & 0        & 0          & 25                          & 100.00\% \\
Phi    & 25    & 0        & 0          & 25                          & 100.00\% \\
OPT    & 25    & 0        & 0          & 25                          & 100.00\% \\
BLOOM  & 6     & 5        & 83.33\%    & 6                           & 100.00\% \\
VMamba  & 51     & 0        & 0          & 36                           & 70.59\%        \\ \bottomrule
\end{tabular}
\end{table}

To be mentioned, Phi3~\citep{abdin2024phi3technicalreporthighly}, Qwen2~\citep{yang2024qwen2technicalreport}, T5~\citep{raffel2023exploringlimitstransferlearning}, Mamba2~\citep{dao2024transformers} and LLaMA~\citep{touvron2023llama} originally use RMS\-Norm, instead of LN.

\subsection{Runtime and Memory Overhead}
The runtime and memory overhead of our automatic folding algorithm is minimal. The detection phase requires only one forward pass to determine which LN layers are foldable, after which we store a small list mapping foldable layers to their corresponding modules.

Here, we conduct extensive experiments on different models, hardware, batch size and sequence lengths. All inference comparisons use PyTorch’s official fused CUDA LayerNorm as the baseline, and our optimized CUDA RMS\-Norm kernel derived from it, which we have open-sourced.

We first conduct a confirmatory experiment on GPT-2 using a single RTX 3050 Laptop GPU. The folding algorithm takes 0.0347 seconds. For inference with batch size 32 and sequence length 512, the inference time is reduced from 0.2189 s to 0.2023 s (averaged over 100 runs), yielding a 7.6\% speedup. Notably, the memory overhead introduced by the folding algorithm itself is only 19.00 MB (which is released after folding), and the total tensor memory of the model remains unchanged (486.70 MB both before and after folding). These results demonstrate that our method achieves meaningful inference acceleration without any increase in model memory footprint.

To provide a broader empirical picture, we conduct additional experiments. We record time usage and memory usage for Bert, GPT2, OPT and Phi, on A100-40GB and V100-32GB. The batch size and sequence length for inference are 1 and 256 respectively. Memory footprint remains identical before and after folding. We list the results in the tables below:

\begin{table}[htbp]
\caption{Overhead and inference time benefits of folding algorithm on A100-32G GPU.}

\centering
\begin{tabular}{@{}rccccccc@{}}
\toprule
Model & Fold (s) & Init (s) & Ratio (\%) & Before (s) & After (s) & Speedup (\%) & Break-even\\ \midrule
Bert & 0.0452 & 2.4056  & 1.84 & 0.0077 & 0.0073 & 5.19  & 112.9 \\
GPT2 & 0.0651 & 3.0651  & 2.08 & 0.0119 & 0.0113 & 5.22  & 105.0 \\
OPT  & 0.0484 & 2.6283  & 1.81 & 0.0092 & 0.0078 & 14.60 & 36.1  \\
Phi  & 0.0928 & 22.2692 & 0.42 & 0.0531 & 0.0507 & 4.52  & 38.7  \\ \bottomrule
\end{tabular}
\end{table}

\begin{table}[htbp]
\caption{Overhead and benefits of folding algorithm on V100-32G GPU.}
\centering
\begin{tabular}{@{}rccccccc@{}}
\toprule
Model & Fold (s) & Init (s) & Ratio (\%) & Before (s) & After (s) & Speedup (\%) & Break-even\\ \midrule
Bert & 0.1263 & 2.0772  & 5.73 & 0.0095 & 0.0086 & 9.47  & 140.4 \\
GPT2 & 0.1639 & 2.6343  & 5.86 & 0.0115 & 0.0110 & 4.34  & 327.8 \\
OPT  & 0.1183 & 2.2988  & 4.90 & 0.0106 & 0.0091 & 14.73 & 75.6  \\
Phi  & 0.1779 & 21.8865 & 0.81 & 0.0693 & 0.0668 & 3.54  & 72.6  \\ \bottomrule
\end{tabular}
\end{table}
{\small
(a) Fold = Folding time.
(b) Init = Model initialize time.
(c) Ratio = Fold / (Fold + Init) × 100\%.
(d) Before = Inference latency without folding. 
(e) After = Inference latency with algorithm applied.  
(f) Speedup = (Before - After) / Before × 100\%.  
(g) Break-even = Fold / (Before - After).
}

The one-time folding algorithm ($\sim$60–100 ms on A100, $\sim$100-200 ms on V100) is fully amortized after only a few inference steps in any real deployment. Importantly, larger batch sizes and longer sequences achieve greater acceleration with our algorithm in inference, thus requiring fewer runs to break even on the algorithm cost.

Moreover, time usage of folding algorithm is theoretically independent of batch size and sequence length. We conduct experiments on Bert on V100, evaluated the algorithm time usage across six batch sizes (4 to 128) and four sequence lengths (64 to 4096). The result is listed below.

\begin{table}[htbp]
\caption{Folding algorithm time usage (ms) across different batch sizes (BS) and sequence lengths (SL).}
\label{table:bs-sl}
\centering
\begin{tabular}{@{}rccccccc@{}}
\toprule
BS\textbackslash{}SL & 64    & 256   & 1024  & 4096  & Average          & Std              & CV (\%)          \\ \midrule
4                    & 132.3 & 132.8 & 133.7 & 138.4 & 134.3            & 2.77             & 2.06             \\
8                    & 137.1 & 132.9 & 133.9 & 136.9 & 135.2            & 2.11             & 1.56             \\
16                   & 139.3 & 134.8 & 134.9 & 136.0 & 136.3            & 2.11             & 1.55             \\
32                   & 134.6 & 135.5 & 136.1 & 137.6 & 135.9            & 1.26             & 0.92             \\
64                   & 142.7 & 136.7 & 138.3 & 142.0 & 139.9            & 2.86             & 2.04             \\
128                  & 135.9 & 132.1 & 136.0 & 145.1 & 137.3            & 5.51             & 4.01             \\
Average              & 137.0 & 134.2 & 135.5 & 139.3 & 136.5            & \textbackslash{} & \textbackslash{} \\
Std                  & 3.65  & 1.80  & 1.73  & 3.50  & \textbackslash{} & 3.28             & \textbackslash{} \\
CV (\%)              & 2.66  & 1.34  & 1.28  & 2.51  & \textbackslash{} & \textbackslash{} & 2.40             \\ \bottomrule
\end{tabular}
\end{table}

As shown in Table~\ref{table:bs-sl}, despite a 2048-fold increase in the total number of processed tokens, the standard deviation of algorithm time usage across all 24 configurations is only 3.3 ms, with a coefficient of variation of merely 2.40\%. Both marginal averages (per row and per column) and their corresponding CV values remain below 4\%, confirming that algorithm time usage is essentially independent of both batch size and sequence length in practical deployment scenarios.

\section{Ablation Experiment of Centering Operation}
\label{apx:ablation}
For the ablation experiments on the centering operation in LN in Section~\ref{exp:ln}, we build up simple MLPs by stacking linear layers, LNs and ReLUs.
The depth of MLP (i.e., the number of linear layers) varies among 6, 15, and 35 on CIFAR-10 \cite{Krizhevsky2009CIFAR10}, with a width of 256 and 512. We introduce the residual structure to help converge for deeper MLP with a depth of 65 and 100 on MNIST \cite{lecun1998gradient}, with a width of 512. We train the model with a learning rate of 0.01 and a batch size of 256. We train all the models for 175 epochs. The training accuracy of all the models in this experiment is $100\%$. 

\begin{figure}[hbp]
    \vspace{-1ex}
    \centering
    \begin{subfigure}[t]{0.30\textwidth}
        \centering
        \includegraphics[height=20ex]{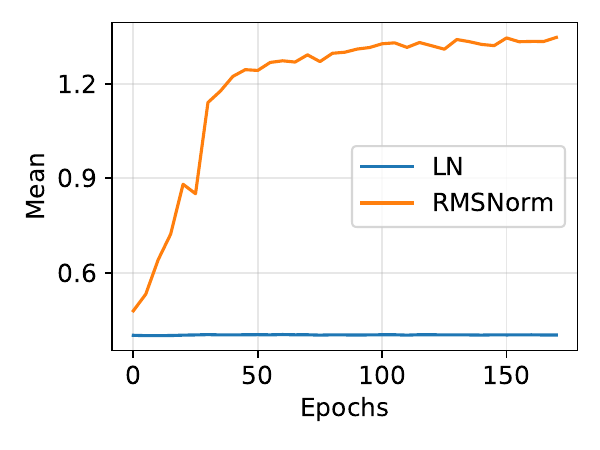}
        \caption{d=6 on CIFAR-10.}
    \end{subfigure}   
    \begin{subfigure}[t]{0.30\textwidth}
        \centering
      \includegraphics[height=20ex]{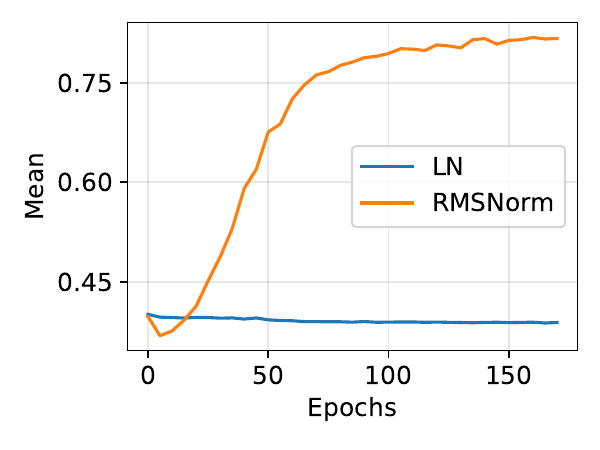}
        \caption{d=15 on CIFAR-10.}
    \end{subfigure}   
    \begin{subfigure}[t]{0.30\textwidth}
        \centering
      \includegraphics[height=20ex]{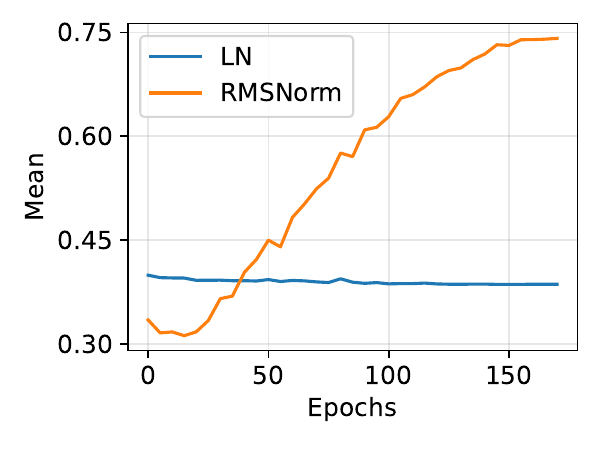}
        \caption{d=35 on CIFAR-10.}
    \end{subfigure}   
        \begin{subfigure}[t]{0.30\textwidth}
        \centering
      \includegraphics[height=20ex]{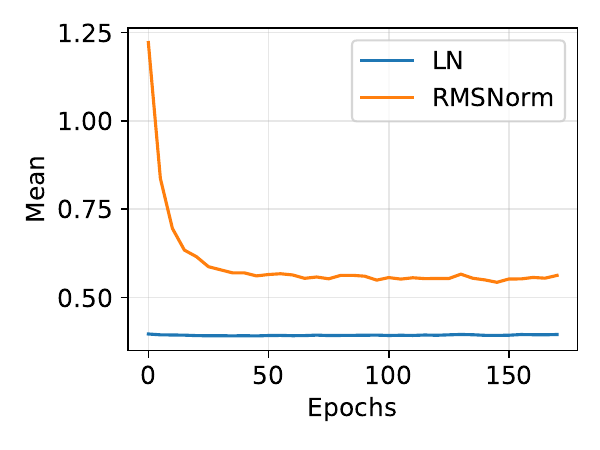}
        \caption{d=65 on MNIST.}
    \end{subfigure}   
    \begin{subfigure}[t]{0.30\textwidth}
        \centering
      \includegraphics[height=20ex]{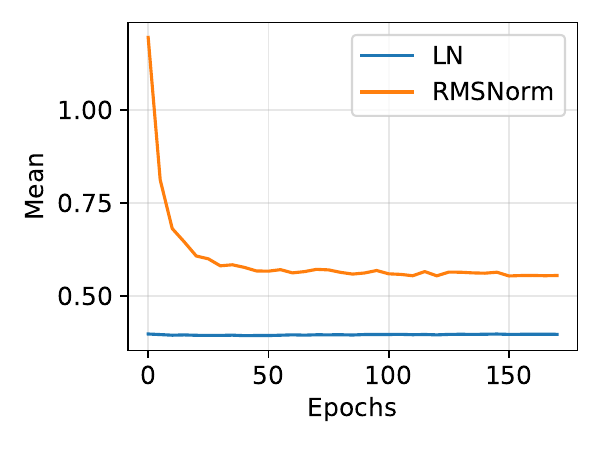}
        \caption{d=100 on MNIST.}
    \end{subfigure}   
    \vspace{-1ex}
  \caption{Mean of the final layer's input for MLPs of different depth ($d$). The change is similar to the change of norm.}
  \label{fig:mean-last-layer}
    \vspace{-2ex}
\end{figure}

We show the plot of the input norm of last layer during the training process in Figure~\ref{fig:exp-centering}.
Here, we also provide the input mean value of the last layer in Figure~\ref{fig:mean-last-layer}. Since it is quite close to the change of the input norm, we will not go into detail.

\begin{figure}[hbp]
    \vspace{-1ex}
    \centering
    \begin{subfigure}[t]{0.30\textwidth}
        \centering
      \includegraphics[height=20ex]{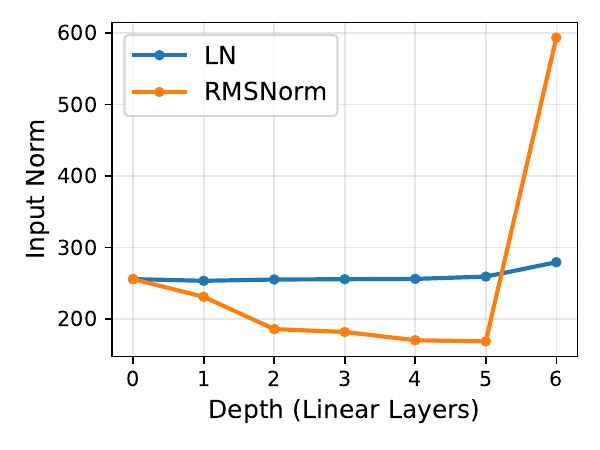}
        \caption{d=6 on CIFAR-10.}
    \end{subfigure}   
    \begin{subfigure}[t]{0.30\textwidth}
        \centering
      \includegraphics[height=20ex]{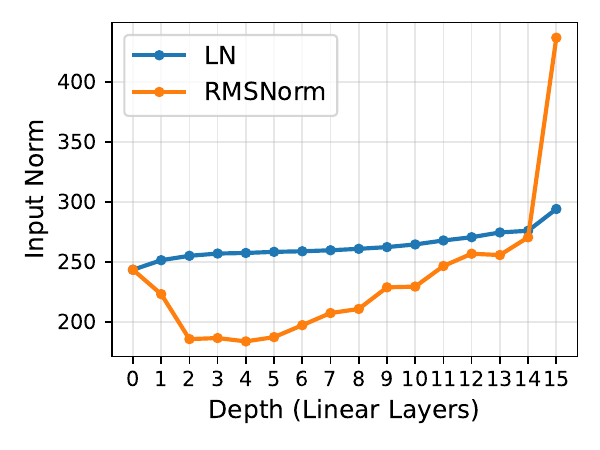}
        \caption{d=15 on CIFAR-10.}
    \end{subfigure}   
    \begin{subfigure}[t]{0.30\textwidth}
        \centering
      \includegraphics[height=20ex]{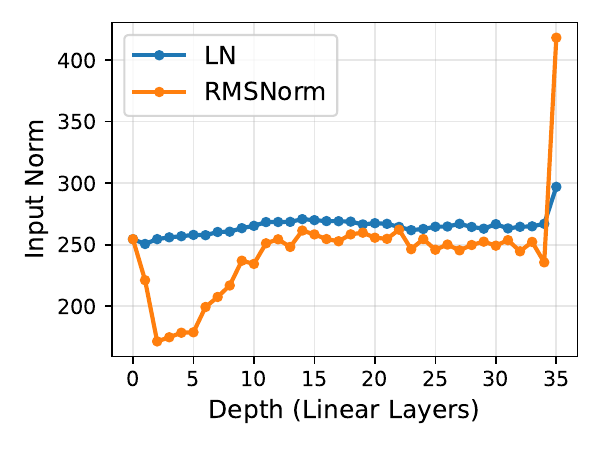}
        \caption{d=35 on CIFAR-10.}
    \end{subfigure}   
        \begin{subfigure}[t]{0.30\textwidth}
        \centering
      \includegraphics[height=20ex]{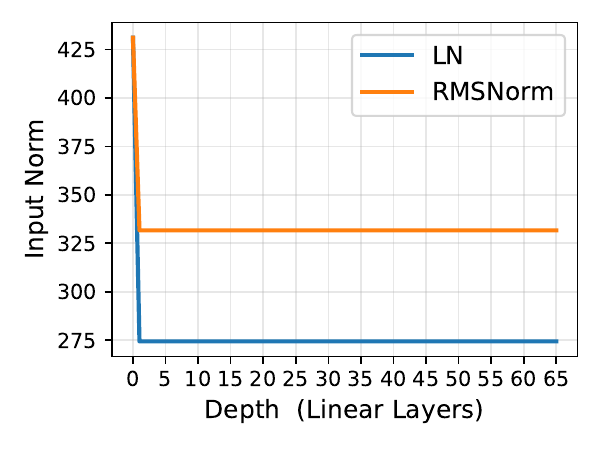}
        \caption{ d=65 on MNIST.}
    \end{subfigure}   
    \begin{subfigure}[t]{0.30\textwidth}
        \centering
      \includegraphics[height=20ex]{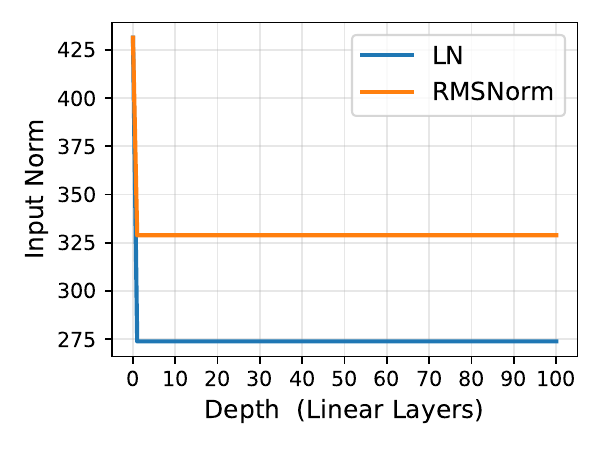}
        \caption{d=100 on MNIST.}
    \end{subfigure}   
    \vspace{-1ex}
  \caption{Norm of the input across linear layers for MLPs of different depth ($d$). LN better controls the norm of samples throughout the model.}
  \label{fig:norm-linear-layer}
\end{figure}

Additionally, we recorded the input norms of every linear layer across all layers for the four models in the final training epoch in Figure~\ref{fig:norm-linear-layer}. The results are summarized as follows:

\begin{itemize}
    \item In plain MLP networks (without residual connections), the input norms of the LN model are not strictly smaller than those of RMS\-Norm at every layer. However, the RMS\-Norm-based models exhibit a sharp spike in input norm at the final layer, accompanied by clear oscillations in input norms between layers.
    \item In residual-connected MLP networks, the LN model consistently shows significantly smaller input norms than RMS\-Norm at every single layer. Moreover, these input norms remain remarkably stable and nearly constant across layers—an effect directly induced by the residual connections.
\end{itemize}

\begin{figure}[hbp]
    \vspace{-1ex}
    \centering
    \begin{subfigure}[t]{0.30\textwidth}
        \centering
      \includegraphics[height=20ex]{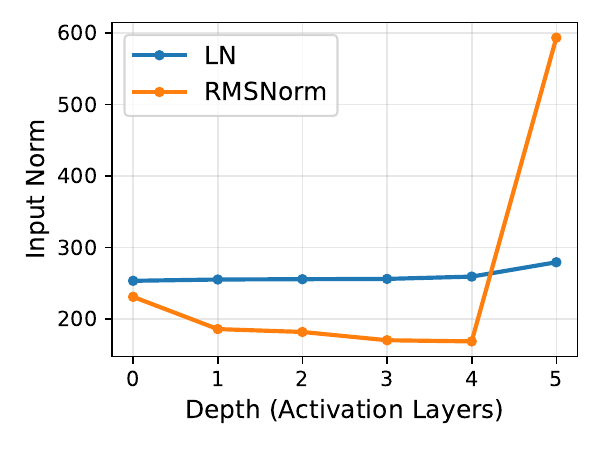}
        \caption{ d=6 on CIFAR-10.}
    \end{subfigure}   
    \begin{subfigure}[t]{0.30\textwidth}
        \centering
      \includegraphics[height=20ex]{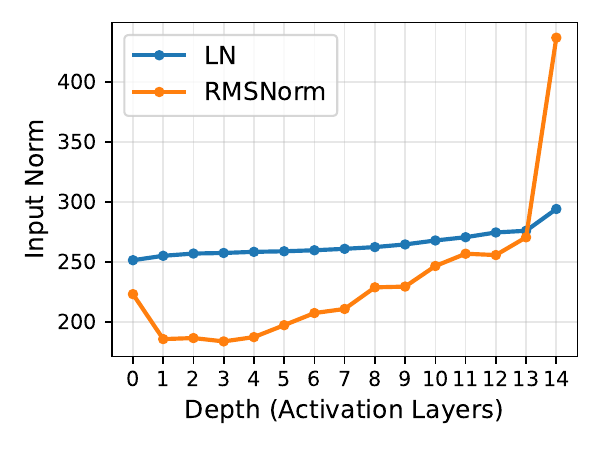}
        \caption{ d=15 on CIFAR-10.}
    \end{subfigure}   
    \begin{subfigure}[t]{0.30\textwidth}
        \centering
      \includegraphics[height=20ex]{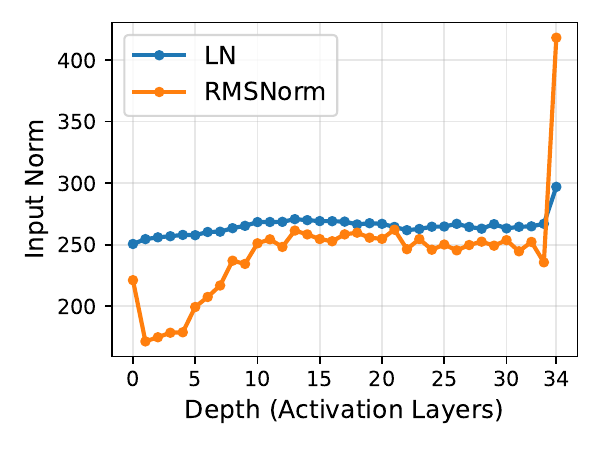}
        \caption{ d=35 on CIFAR-10.}
    \end{subfigure}   
        \begin{subfigure}[t]{0.30\textwidth}
        \centering
      \includegraphics[height=20ex]{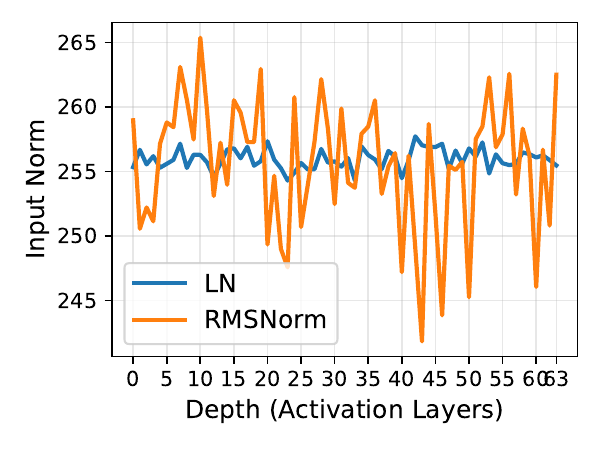}
        \caption{ d=65 on MNIST.}
    \end{subfigure}   
    \begin{subfigure}[t]{0.30\textwidth}
        \centering
      \includegraphics[height=20ex]{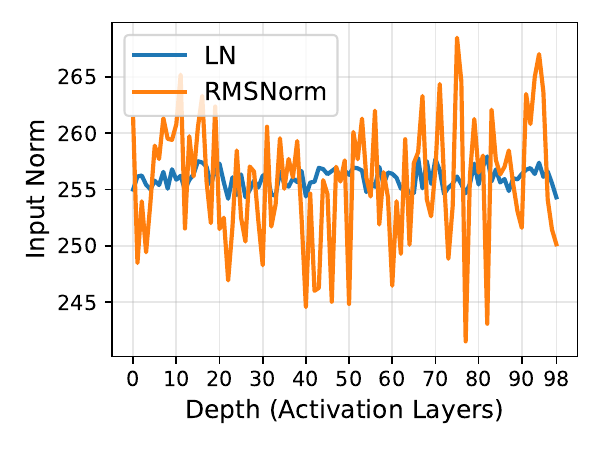}
        \caption{d=100 on MNIST.}
    \end{subfigure}   
    \vspace{-1ex}
  \caption{Norm of the input across activation layers for MLPs of different depth $d$. The change is similar to the change of norm.}
  \label{fig:norm-activation-layer}
  \vspace{-2ex}
\end{figure}

\begin{figure}[hbp]
    \vspace{-1ex}
    \centering
    \begin{subfigure}[t]{0.30\textwidth}
        \centering
      \includegraphics[height=20ex]{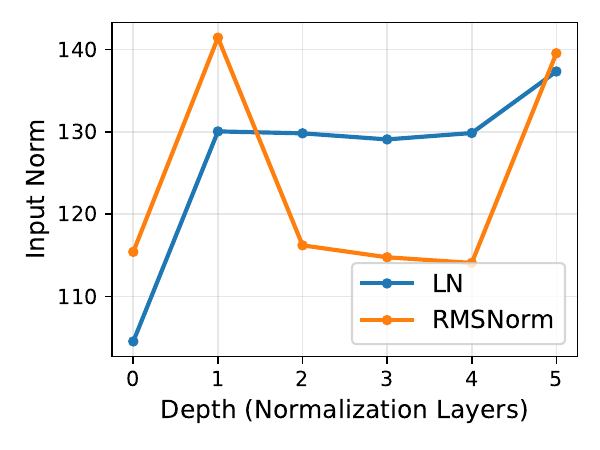}
        \caption{d=6 on CIFAR-10.}
    \end{subfigure}   
    \begin{subfigure}[t]{0.30\textwidth}
        \centering
      \includegraphics[height=20ex]{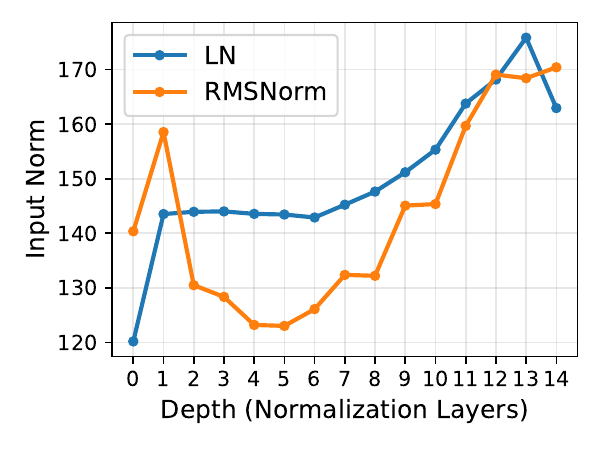}
        \caption{d=15 on CIFAR-10.}
    \end{subfigure}   
    \begin{subfigure}[t]{0.30\textwidth}
        \centering
      \includegraphics[height=20ex]{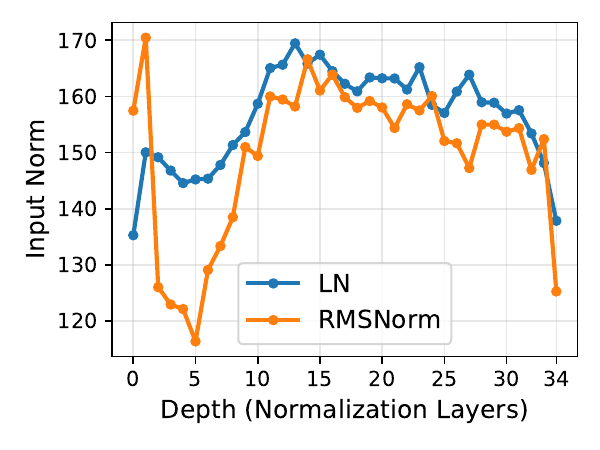}
        \caption{d=35 on CIFAR-10.}
    \end{subfigure}   
        \begin{subfigure}[t]{0.30\textwidth}
        \centering
      \includegraphics[height=20ex]{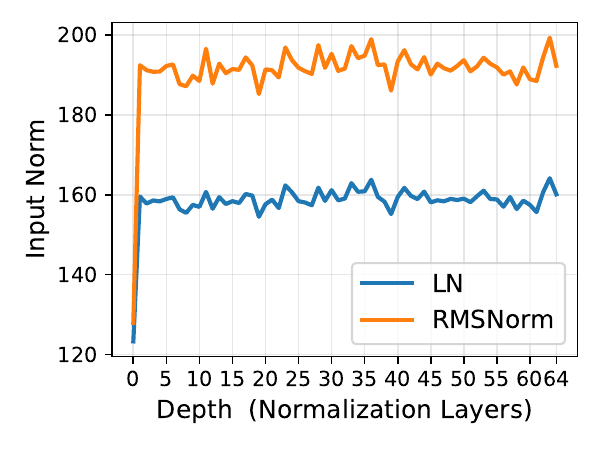}
        \caption{d=65 on MNIST.}
    \end{subfigure}   
    \begin{subfigure}[t]{0.30\textwidth}
        \centering
      \includegraphics[height=20ex]{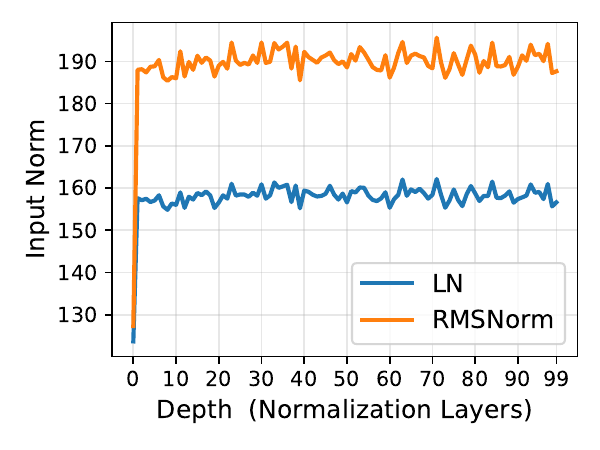}
        \caption{d=100 on MNIST.}
    \end{subfigure}   
    \vspace{-1ex}
  \caption{Norm of the input across normalization layers for MLPs of different depth $d$.}
  \label{fig:norm-norm-layer}
  \vspace{-2ex}
\end{figure}

We also observed the variations in input norms across other layers. As depicted in Figure~\ref{fig:norm-activation-layer}, the trend observed in activation function layers is similar to that in linear layers, which we will not elaborate on further. To be noted, the range of input norm with residual connection is actually relatively small, between 240 and 270.
As for the normalization layers shown in Figure~\ref{fig:norm-norm-layer}, without residual connections, the range of input norms under RMS\-Norm exhibits more significant fluctuations, and the inter layer oscillations are greater; in the presence of residual connections, the input norms are effectively constrained by LN. 

Therefore, we believe that the centering operation of LN controls the range of activations both across layers and throughout the training process.

\section{Inference Acceleration}
\label{apx:inference-accele}

In this section, we analyze the inference-time acceleration brought by CBWC+RMS\-Norm, both theoretically and empirically. Since CBWC is applied only once before deployment and the weight matrix remains fixed during inference, the runtime gain mainly comes from replacing LN with RMS\-Norm.

\subsection{Theoretical Analysis}
\label{apx:inference-accele-theory}
Theoretically, since calculating RMS\-Norm has fewer computational steps compared to LN, replacing LN with RMS\-Norm has acceleration in FLOPs, thus reflecting on throughput.

\subsubsection{Layer-level FLOPs Analysis}
FLOPs (Floating-Point Operations) measure the total number of floating-point operations required for one forward pass and are a standard metric for evaluating the computational cost of a model. Although FLOPs are decoupled from the number of parameters, they directly affect inference latency and energy consumption, and are therefore widely used in efficiency-oriented research.

We assume that addition, subtraction, and multiplication each cost one clock tick, while division costs three clock ticks.

Consider a sample $\rvx$ of dimension $d$. LN and RMS\-Norm are defined as follows:

\begin{equation}
\label{eqn:ln}
    \mathrm{LN}(\rvx)=\frac{\rvx-\mu}{\sqrt{\sigma^2+\epsilon}},\quad
    \text{where}\quad
    \mu = \frac{1}{d}\sum_{i=1}^d x_i,\quad
    \sigma^2 = \frac{1}{d}\sum_{i=1}^d(x_i-\mu)^2.
\end{equation}

\begin{equation}
\label{eqn:rms}
    \mathrm{RMS}(\rvx)=\frac{\rvx}{\sqrt{\sigma_{rms}^2+\epsilon}},\quad
    \text{where}\quad
    \sigma_{rms}^2 = \frac{1}{d}\sum_{i=1}^d x_i^2.
\end{equation}

According to the formulas above, we compute the operation counts in Table~\ref{table:flops}. For a $d$-dimensional sample, LN requires $5d$ additions, $2d$ multiplications, and $d$ divisions, whereas RMS\-Norm requires only $d$ additions, $2d$ multiplications, and $d$ divisions. This indicates a clear layer-level computational advantage for RMS\-Norm.

\begin{table}[htbp]
\caption{FLOPs calculation for LN and RMS\-Norm.}
\label{table:flops}
\centering
\begin{tabular}{@{}lcccccccc@{}}
\toprule
\multirow{2}{*}{Steps} &  & \multicolumn{3}{c}{LN} &  & \multicolumn{3}{c}{RMS} \\ \cmidrule(l){3-9}
 &  & $+/-$ & $\times$ & $\div$ &  & $+/-$ & $\times$ & $\div$ \\ \midrule
Calculating average ($\mu$) &  & $d-1$ & $0$ & $1$ &  & --- & --- & --- \\
Calculating variance ($\sigma^2$) &  & $2d-1$ & $d$ & $1$ &  & $d-1$ & $d$ & $1$ \\ \midrule
Inverse square root for variance &  & \multicolumn{7}{c}{One \texttt{rsqrt}, ignored} \\
Centering \& scaling for each dimension &  & $d$ & $0$ & $d$ &  & $0$ & $0$ & $d$ \\
Affine &  & $d$ & $d$ & $0$ &  & $0$ & $d$ & $0$ \\ \midrule
Total &  & $5d$ & $2d$ & $d$ &  & $d$ & $2d$ & $d$ \\ \bottomrule
\end{tabular}
\end{table}

\subsubsection{Welford-based FLOPs Analysis}

In practice, PyTorch implements fused LN using the Welford algorithm, whose operation count differs from the naive formulation. We therefore also compare LN and RMS\-Norm under a Welford-style implementation, which better reflects practical GPU execution.

Instead of traversing the data twice, this method requires only one traversal, with the following updates:
\begin{equation}
\begin{aligned}
    \overline{x}_{n+1} & = \overline{x}_{n} + \frac{x_{n+1} - \overline{x}_{n}}{n+1}, \\
    \widetilde{\sigma^2}_{n+1} & = \widetilde{\sigma^2}_{n} + (x_{n+1}-\overline{x}_n)(x_{n+1}-\overline{x}_{n+1}),
\end{aligned}
\end{equation}
where $\overline{x}_{n}$ and $\widetilde{\sigma^2}_{n}=n\sigma_{n}^2$ denote the mean and the scaled variance of the first $n$ elements.

For parallel computation, we combine two groups $M$ and $N$ with $m$ and $n$ elements through
\begin{equation}
\begin{aligned}
    \overline{x}^{(M\cup N)} & = \frac{m}{m+n}\overline{x}^{(M)} + \frac{n}{m+n}\overline{x}^{(N)}, \\
    \widetilde{\sigma^2}^{(M\cup N)} & = \widetilde{\sigma^2}^{(M)}+\widetilde{\sigma^2}^{(N)}+\frac{mn}{m+n}\left(\overline{x}^{(M)}-\overline{x}^{(N)}\right)^2.
\end{aligned}
\end{equation}

Similarly, RMS\-Norm under the same computation pattern can be written as
\begin{equation}
    \widetilde{\sigma^2}_{n+1} = \widetilde{\sigma^2}_{n} + x_{n+1}^2,
\end{equation}
and
\begin{equation}
    \widetilde{\sigma^2}^{(M\cup N)}=\widetilde{\sigma^2}^{(M)}+\widetilde{\sigma^2}^{(N)},
\end{equation}
which are substantially simpler than the corresponding updates for LN.

According to the formulas above, we compute the operation counts in Table~\ref{tab:flops-welford}. For a $d$-dimensional sample divided into $g$ groups for parallel computation, LN requires $7d$ additions, $3d+7g$ multiplications, and $d$ divisions, whereas RMS\-Norm requires only $d$ additions and $3d$ multiplications.

\begin{table}[htpb]
\centering
\caption{FLOPs calculation for LN and RMS\-Norm under the Welford algorithm.}
\label{tab:flops-welford}
\begin{tabular}{@{}lccccccccc@{}}
\toprule
\multirow{2}{*}{Steps} & \multirow{2}{*}{Times} &  & \multicolumn{3}{c}{LN} &  & \multicolumn{3}{c}{RMS} \\ \cmidrule(l){3-10}
 &  &  & $+/-$ & $\times$ & $\div$ &  & $+/-$ & $\times$ & $\div$ \\ \midrule
Group initialization ($x_1^2$) & $g$ &  & $0$ & $1$ & $0$ &  & $0$ & $1$ & $0$ \\
Summing within groups & $d-g$ &  & $5$ & $1$ & $1$ &  & $1$ & $1$ & $0$ \\
Combining groups & $g-1$ &  & $5$ & $7$ & $1$ &  & $1$ & $0$ & $0$ \\ \midrule
Inverse square root for variance & $1$ &  & \multicolumn{7}{c}{One \texttt{rsqrt}, ignored} \\
Centering \& scaling for each dimension & $d$ &  & $1$ & $1$ & $0$ &  & $0$ & $1$ & $0$ \\
Affine & $d$ &  & $1$ & $1$ & $0$ &  & $0$ & $1$ & $0$ \\ \midrule
Total & --- &  & $7d$ & $3d+7g$ & $d$ &  & $d$ & $3d$ & $0$ \\ \bottomrule
\end{tabular}
\end{table}

Although the Welford algorithm increases the theoretical FLOPs compared with the naive formulation, it reduces traversal overhead and enables efficient fused GPU execution in practice. Under this implementation, the estimated normalization-layer latency is reduced from approximately $13d+7g$ clock ticks for LN to $4d$ for RMS\-Norm, corresponding to a theoretical reduction of roughly 60\% to 80\% for the normalization step itself.

A similar improvement is expected for throughput. Since our method mainly removes the centering operation and does not significantly change model size or memory footprint, the throughput gain is primarily determined by the reduction in computational latency. Therefore, the throughput trend is expected to be consistent with the latency analysis above.

\subsubsection{From Layer-level Savings to Model-level Speedup}

Although the normalization layer itself can be substantially accelerated, the end-to-end model speedup depends on the fraction of inference time originally spent in LN. We therefore estimate the average proportion of LN latency in several representative models and combine it with the layer-level savings above to obtain an expected whole-model speedup.

We estimate the average inference-time proportion of LN by running each model 1000 times on a single RTX 3090 GPU. Based on the Welford-style analysis above, we expect RMS\-Norm to provide a 50\% to 80\% reduction for the normalization layer itself. The resulting estimated end-to-end speedups are shown in Table~\ref{table:timeusage-of-LN}.

\begin{table}[htpb]
\centering
\caption{Estimated end-to-end inference speedup from LN time proportion in representative models.}
\label{table:timeusage-of-LN}
\begin{tabular}{@{}lcccc@{}}
\toprule
Model & Total Time Usage & LN Time Usage & Proportion & Expected Acceleration \\ \midrule
GPT-2 & 7.125218 & 0.763676 & 10.72\% & \textasciitilde\ 6.5\% \\
BERT  & 7.462299 & 0.713215 & 9.56\%  & \textasciitilde\ 6\%   \\
BLOOM & 2.321148 & 0.191480 & 8.25\%  & \textasciitilde\ 5\%   \\
OPT   & 10.14867 & 0.743909 & 7.33\%  & \textasciitilde\ 4.5\% \\
ViT   & 6.552383 & 0.766613 & 11.70\% & \textasciitilde\ 7\%   \\ \bottomrule
\end{tabular}
\vspace{-2ex}
\end{table}

\subsection{Reference Experiment with Custom LN and RMSNorm}
\label{apx:inference-accele-ref}

Previous work~\citep{jiang2023pre} compared custom implementations of LN and RMS\-Norm that directly follow Eqns.~\ref{eqn:ln} and~\ref{eqn:rms}. Such comparisons are useful as a reference, but they do not reflect real deployment settings, because PyTorch already provides a highly optimized fused LN kernel, whereas RMS\-Norm requires a dedicated optimized implementation to be competitive.

In this reference experiment, both LN and RMS\-Norm are implemented by our team without vendor-level acceleration. As a result, neither implementation reflects the optimized inference kernels used in practice. In particular, a custom LN implementation is substantially slower than PyTorch’s fused LN, which causes LN to occupy an unrealistically large proportion of total model runtime. Moreover, the gap between LN and RMS\-Norm under this setting differs from the Welford-based comparison in Table~\ref{tab:flops-welford}. We therefore report this experiment only as a reference point, including both runtime and CUDA memory statistics.

\begin{table}[htpb]
\centering
\caption{Average total and CUDA runtime (ms) and acceleration percentage for 16 runs across 6 models using custom LN and custom RMS\-Norm implementations.}
\label{tab:custom-ln-rms-time}
\vskip 0.1in
\begin{tabular}{@{}llccclccc@{}}
\toprule
\multirow{2}{*}{Model} &  & \multicolumn{3}{c}{Total Runtime} &  & \multicolumn{3}{c}{CUDA Runtime} \\ \cmidrule(lr){3-5} \cmidrule(l){7-9}
 &  & Folded & Original & Acceleration &  & Folded & Original & Acceleration \\ \midrule
GPT-2 &  & 10.929 & 11.923 & 8.34\%  &  & 2.784 & 2.933 & 5.07\%  \\
BERT  &  & 11.949 & 12.916 & 7.49\%  &  & 2.904 & 3.146 & 7.70\%  \\
BLOOM &  & 3.322  & 3.548  & 6.38\%  &  & 0.246 & 0.285 & 13.59\% \\
OPT   &  & 12.363 & 12.692 & 2.59\%  &  & 2.932 & 2.994 & 2.08\%  \\
Phi   &  & 32.005 & 32.556 & 1.69\%  &  & 26.290 & 26.691 & 1.50\% \\
ViT   &  & 10.709 & 11.683 & 8.34\%  &  & 4.457 & 4.723 & 5.64\%  \\ \bottomrule
\end{tabular}
\end{table}

Table~\ref{tab:custom-ln-rms-time} shows speedups ranging from roughly 2\% to 12\%, depending on the LN proportion and model structure. We also report CUDA memory usage in Table~\ref{tab:cuda-mem-custom}, where the measured memory reduction ranges from roughly 0.5\% to 8.5\%.

\begin{table}[htpb]
\centering
\caption{Average CUDA memory usage (Bytes) and reduction percentage for 16 runs across 6 models using custom LN and custom RMS\-Norm implementations.}
\label{tab:cuda-mem-custom}
\vskip 0.1in
\begin{tabular}{@{}lccc@{}}
\toprule
Model & Folded & Original & Reduction \\ \midrule
GPT-2 & 303602176  & 313445376  & 3.140\% \\
BERT  & 119192576  & 129035776  & 7.628\% \\
BLOOM & 7358976    & 7558656    & 2.642\% \\
OPT   & 161150976  & 161939968  & 0.487\% \\
Phi   & 1350695936 & 1376923136 & 1.905\% \\
ViT   & 163479552  & 178634752  & 8.484\% \\ \bottomrule
\end{tabular}
\end{table}

\subsection{Main CUDA Verification Experiments}
\label{apx:acc-exp}

\subsubsection{Main Setting and Primary Results}

We compare our custom CUDA-accelerated RMS\-Norm kernel against PyTorch’s official fused LN implementation, which serves as the strongest practical baseline. PyTorch’s LN already incorporates a highly optimized fused implementation, whereas its RMS\-Norm support is not a comparably strong baseline for this purpose. Therefore, we implement our own CUDA RMS\-Norm kernel using the same Welford algorithm as fused LN, yielding a comparison that is much closer to the practical upper bound of achievable inference speedup.

We use \texttt{torch.profiler} to trace both total runtime and CUDA runtime. For each model, we conduct 16 runs, each averaged over 100 independent inference passes. The measured runtime of each run is normalized by the mean baseline runtime of the corresponding model. Due to limited hardware resources, the start times of different comparison groups are not perfectly aligned, which introduces some variability in measured runtime.

\begin{figure}[hbp]
    \begin{subfigure}[t]{0.48\textwidth}
        \centering
        \includegraphics[height=15ex]{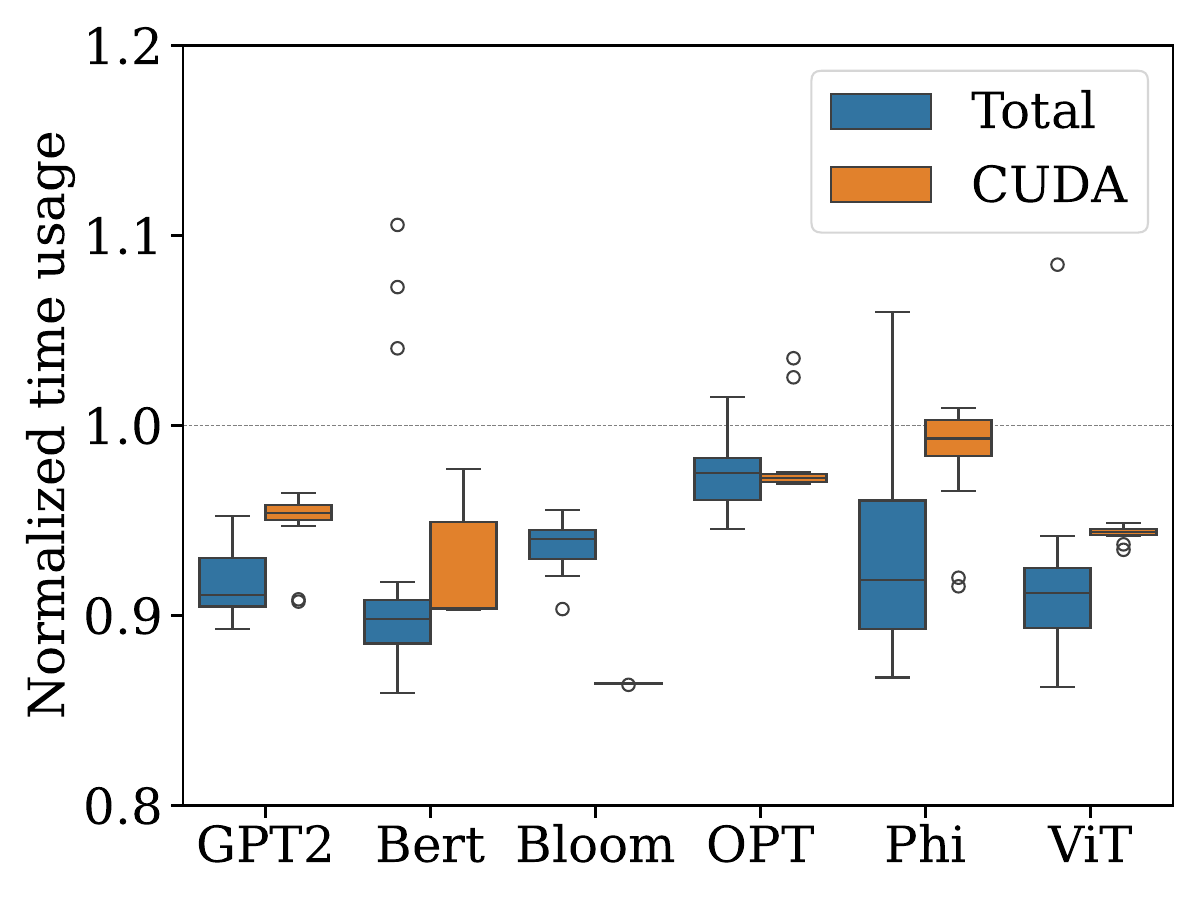}
        \caption{Normalized latency.}
    \end{subfigure}
    \begin{subfigure}[t]{0.48\textwidth}
        \centering
        \includegraphics[height=15ex]{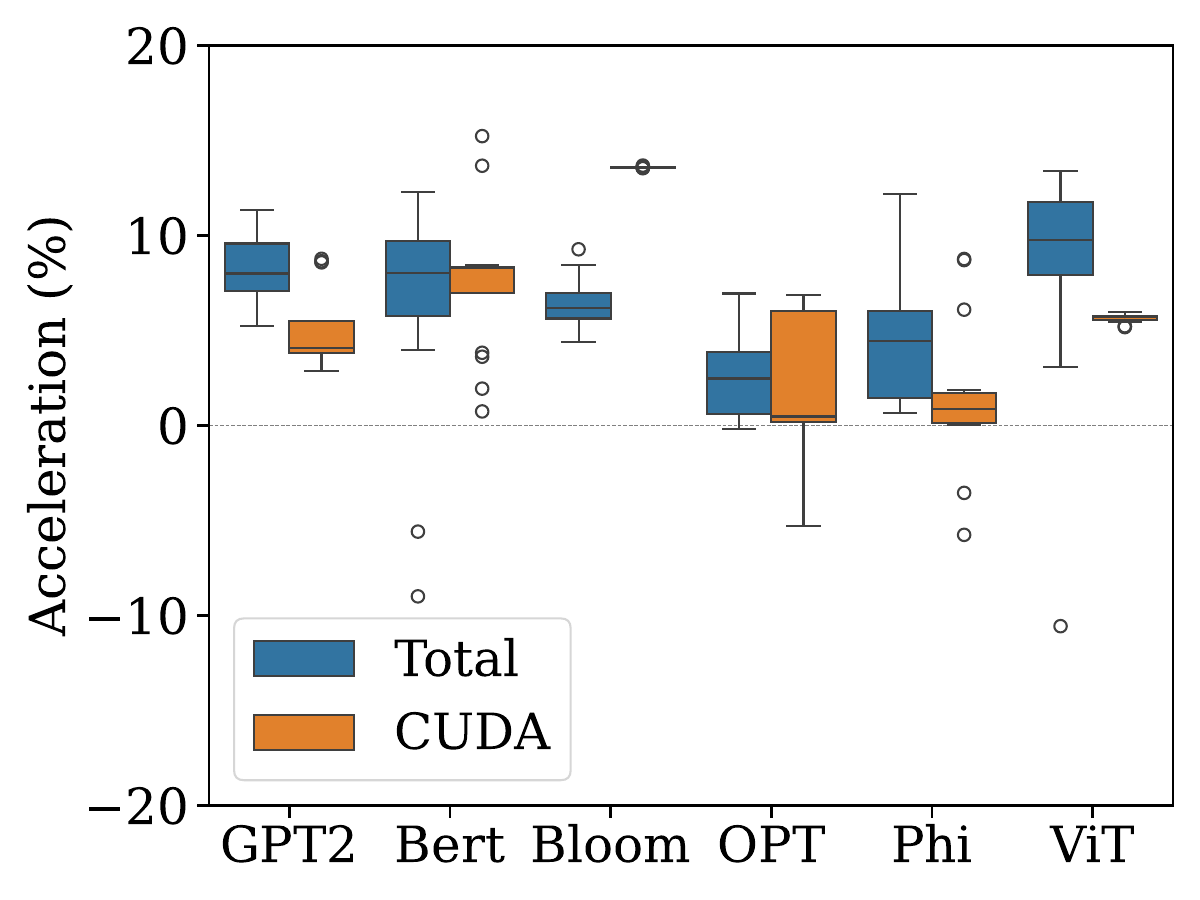}
        \caption{Acceleration ratio.}
    \end{subfigure}
    \caption{Inference latency comparison across six representative models (GPT-2, BERT, BLOOM, OPT, Phi-3, and ViT). Our CBWC+RMS\-Norm achieves a consistent end-to-end runtime reduction of 2\%--12\% with no accuracy degradation.}
\end{figure}

We also measure the proportion of inference runtime originally spent in LN for these models, as reported in Table~\ref{table:timeusage-of-LN-main}.

\begin{table}[htpb]
\centering
\caption{Average inference runtime (ms) and LN proportion in representative models.}
\label{table:timeusage-of-LN-main}
\vspace{1ex}
\begin{tabular}{@{}lccc@{}}
\toprule
Model & Total Time Usage & LN Time Usage & Proportion \\ \midrule
BERT  & 7.462299 & 0.713215 & 9.56\%  \\
BLOOM & 2.321148 & 0.191480 & 8.25\%  \\
GPT-2 & 7.125218 & 0.763676 & 10.72\% \\
OPT   & 10.14867 & 0.743909 & 7.33\%  \\
ViT   & 6.552383 & 0.766613 & 11.70\% \\ \bottomrule
\end{tabular}
\end{table}

The inference performance of BERT, GPT-2, Phi, and OPT in Section~\ref{sec:acceleration} was evaluated on a single NVIDIA A100-40GB GPU. We first performed 50 warm-up inferences, followed by 15 runs of 50 measured inferences each. All measurements used a fixed batch size of 2 and an input/output sequence length of 1024 tokens.

For analysis, we first computed the mean baseline runtime for each run, and then normalized all measurements within that run against this baseline. The reported acceleration percentages were computed accordingly.

\begin{figure}[hbp]
    \centering
    \begin{subfigure}[t]{0.48\textwidth}
        \centering
        \includegraphics[height=18ex]{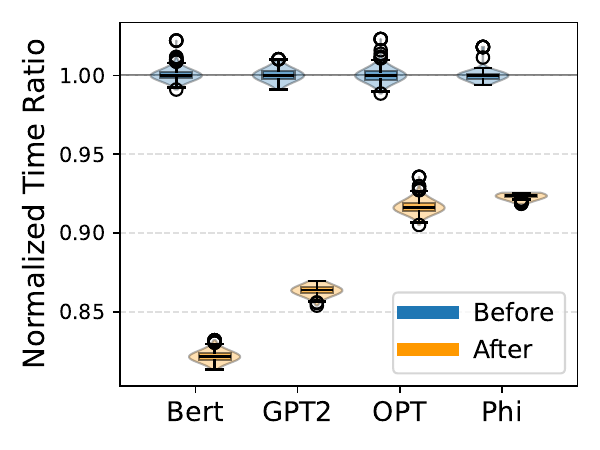}
        \caption{Normalized latency (seq=4096).}
    \end{subfigure}
    \begin{subfigure}[t]{0.48\textwidth}
        \centering
        \includegraphics[height=18ex]{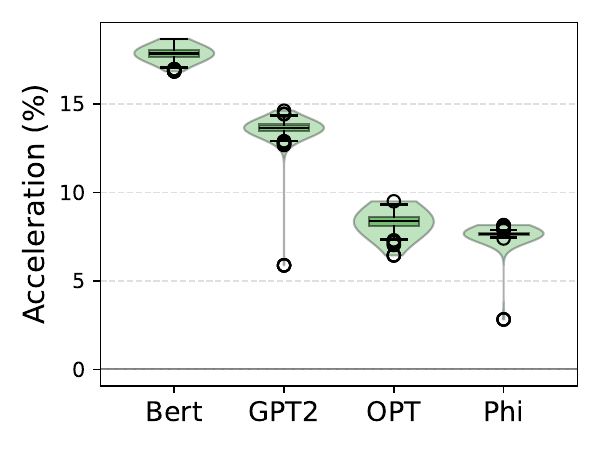}
        \caption{Acceleration ratio (seq=4096).}
    \end{subfigure}
    \begin{subfigure}[t]{0.48\textwidth}
        \centering
        \includegraphics[height=18ex]{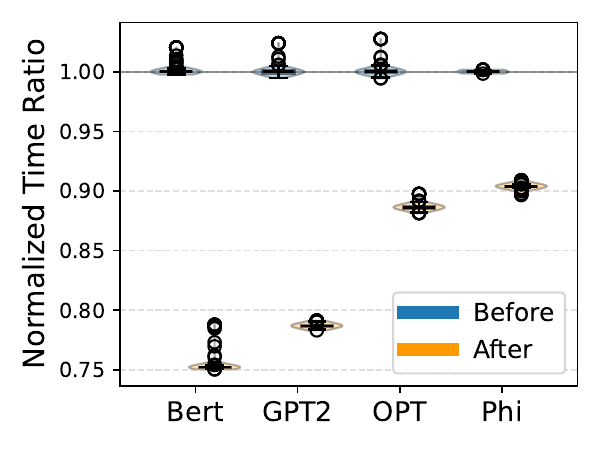}
        \caption{Normalized latency (seq=16384).}
    \end{subfigure}
    \begin{subfigure}[t]{0.48\textwidth}
        \centering
        \includegraphics[height=18ex]{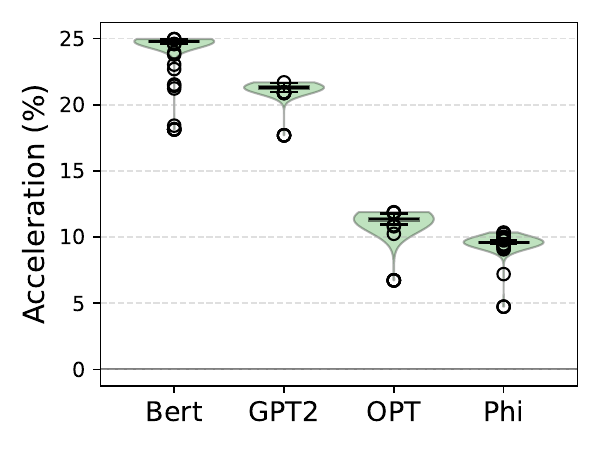}
        \caption{Acceleration ratio (seq=16384).}
    \end{subfigure}
    \caption{Inference latency comparison under longer sequence lengths. Our CBWC+RMS\-Norm achieves larger end-to-end speedups as sequence length increases.}
    \label{fig:cuda-total-time-accel-4096}
\end{figure}

We further investigate performance at longer sequence lengths of 4096 and 16384 tokens with a fixed batch size of 2. As shown in Figure~\ref{fig:cuda-total-time-accel-4096}, the acceleration becomes more pronounced under these conditions, which is consistent with the theoretical analysis.

\subsubsection{Latency and Throughput on A100 and V100}
We evaluate the inference latency and throughput of BERT, GPT-2, and OPT on single NVIDIA V100-32GB and A100-40GB GPUs.

\begin{table}[htbp]
\caption{Average inference latency (seconds) on a single GPU with batch size 2 and sequence length 256. Results are averaged over 5 runs, each with 50 measured inferences after 50 warm-up steps.}
\label{table:avg-time-hardware}
\vspace{1ex}
\centering
\begin{tabular}{@{}llccc@{}}
\toprule
Model & Hardware & Original Inference Time (s) & New Inference Time (s) & Acceleration (\%) \\ \midrule
BERT  & V100 & 0.0116 & 0.0111 & 3.90 \\
GPT-2 & V100 & 0.0121 & 0.0117 & 3.57 \\
OPT   & V100 & 0.0119 & 0.0116 & 2.53 \\
BERT  & A100 & 0.0090 & 0.0087 & 2.79 \\
GPT-2 & A100 & 0.0093 & 0.0089 & 4.22 \\
OPT   & A100 & 0.0094 & 0.0089 & 5.85 \\ \bottomrule
\end{tabular}
\end{table}

\begin{table}[htbp]
\caption{Average inference throughput (tokens/second) on a single GPU with batch size 2 and sequence length 256. Results are averaged over 5 runs, each with 50 measured inferences after 50 warm-up steps.}
\label{table:avg-throughput-hardware}
\vspace{1ex}
\centering
\begin{tabular}{@{}llccc@{}}
\toprule
Model & Hardware & Original Throughput (tokens/s) & New Throughput (tokens/s) & $\Delta$ (tokens/s) \\ \midrule
BERT  & V100 & 44329.0 & 46126.1 & +1797.1 \\
GPT-2 & V100 & 42199.0 & 43760.7 & +1561.7 \\
OPT   & V100 & 43144.9 & 44263.9 & +1119.0 \\
BERT  & A100 & 57117.4 & 58756.0 & +1638.7 \\
GPT-2 & A100 & 55351.4 & 57787.8 & +2436.5 \\
OPT   & A100 & 54468.1 & 57853.1 & +3385.0 \\ \bottomrule
\end{tabular}
\end{table}

All measurements were performed with a fixed batch size of 2 and input/output sequence length of 256 tokens. To ensure stable timing, we first performed 50 warm-up inferences, followed by 50 measured inferences, and repeated the full process 5 times for each configuration. Total inference time was measured using Python’s \texttt{time.time()}.

As shown in Table~\ref{table:avg-time-hardware} and Table~\ref{table:avg-throughput-hardware}, the proposed method consistently reduces inference latency by 2.5\%--5.9\% and increases throughput by 1,119--3,385 tokens/s across all three architectures and both GPU types. These gains are robust and reproducible, demonstrating that the acceleration is not specific to a particular model family or hardware platform.

\subsubsection{Scaling Ability of Folding Technology}

\paragraph{Batch Size}
To assess scalability with respect to batch size, we conduct additional experiments on BERT using a single NVIDIA V100-32GB GPU with a fixed sequence length of 256. Batch sizes are varied from 4 to 128 (powers of two). The same measurement protocol is used: 50 warm-up inferences, followed by 50 measured inferences, repeated 5 times, with results averaged.

\begin{table}[htbp]
\caption{Average inference latency (seconds) of BERT-base on a single V100 GPU (sequence length 256) across varying batch sizes. Results are averaged over 5 runs $\times$ 50 measured inferences after 50 warm-up steps.}
\label{table:scale-bs-time}
\vspace{1ex}
\centering
\begin{tabular}{@{}rccc@{}}
\toprule
Batch Size & Original Inference Time (s) & New Inference Time (s) & Acceleration (\%) \\ \midrule
4   & 0.0201 & 0.0193 & 3.99 \\
8   & 0.0360 & 0.0347 & 3.64 \\
16  & 0.0693 & 0.0675 & 2.58 \\
32  & 0.1339 & 0.1269 & 5.25 \\
64  & 0.2557 & 0.2418 & 5.45 \\
128 & 0.5092 & 0.4763 & 6.45 \\ \bottomrule
\end{tabular}
\end{table}

\begin{table}[htbp]
\caption{Average inference throughput (tokens/second) of BERT-base on a single V100 GPU (sequence length 256) across varying batch sizes. Results are averaged over 5 runs, each with 50 measured inferences after 50 warm-up steps.}
\label{table:scale-bs-throughput}
\vspace{1ex}
\centering
\begin{tabular}{@{}rccc@{}}
\toprule
Batch Size & Original Throughput (tokens/s) & New Throughput (tokens/s) & $\Delta$ (tokens/s) \\ \midrule
4   & 51029.1 & 53147.9 & +2118.8 \\
8   & 56920.5 & 59071.2 & +2150.7 \\
16  & 59148.0 & 60715.7 & +1567.6 \\
32  & 61185.0 & 64577.2 & +3392.1 \\
64  & 64072.1 & 67761.8 & +3689.8 \\
128 & 64357.4 & 68792.8 & +4435.4 \\ \bottomrule
\end{tabular}
\end{table}

As shown in Table~\ref{table:scale-bs-time} and Table~\ref{table:scale-bs-throughput}, the proposed method maintains stable and consistent performance gains across the full batch-size range. Latency reduction ranges from 2.58\% to 6.45\%, while the absolute throughput gain increases progressively with batch size, reaching +4,435 tokens/s at batch size 128. These results confirm that the acceleration mechanism scales effectively with batch size and does not degrade under higher memory or compute intensity.

\subsubsection{Scaling with Sequence Length}

\paragraph{Sequence Length}
To evaluate how the gains evolve with sequence length, we conduct experiments on BERT, GPT-2, and OPT using a single NVIDIA A100-40GB GPU with a fixed batch size of 4. The input/output sequence length is varied from 64 to 16,384 tokens. All measurements follow the same protocol: 50 warm-up inferences, 50 measured inferences, repeated 5 times, with results averaged across runs.

\begin{table}[htbp]
\caption{Average inference throughput (tokens/second) of BERT-base on a single A100 GPU (batch size 4) across increasing sequence lengths.}
\label{table:scale-sl-throughput}
  \vspace{1ex}
\centering
\begin{tabular}{@{}rccc@{}}
\toprule
Seq Length & Origin Throughput (tokens/s) & New Throughput (tokens/s) & $\Delta$ (tokens/s) \\ \midrule
64    & 33464.1 & 34133.3 & +669.3   \\
256   & 67590.8 & 72572.6 & +4981.9  \\
1024  & 69072.5 & 77283.0 & +8210.5  \\
4096  & 46579.5 & 57115.6 & +10536.1 \\
16384 & 19349.4 & 25782.3 & +6432.9  \\ \bottomrule
\end{tabular}
\end{table}

\begin{table}[htbp]
\caption{Absolute throughput improvement (tokens/second) achieved by the proposed method across sequence lengths (batch size 4, single A100 GPU).}
\label{table:scale-sl-throughput-model}
\vspace{1ex}
\centering
\begin{tabular}{@{}lccccc@{}}
\toprule
Model & 64 & 256 & 1024 & 4096 & 16384 \\ \midrule
BERT  & +669.3  & +4981.9 & +8210.5 & +10536.1 & +6432.9 \\
GPT-2 & +1193.0 & +4890.3 & +6235.8 & +8806.8  & +8182.2 \\
OPT   & +4697.6 & +3639.8 & +4145.0 & +5558.0  & +4630.1 \\ \bottomrule
\end{tabular}
\end{table}

As shown in Table~\ref{table:scale-sl-throughput}, Table~\ref{table:scale-sl-throughput-model}, and Table~\ref{table:scale-sl-time}, the proposed folding technique exhibits strong positive scaling with sequence length. On BERT, the absolute throughput improvement grows from +669 tokens/s at sequence length 64 to +10,536 tokens/s at 4096, and remains substantial (+6,433 tokens/s) even at 16,384 tokens. Correspondingly, the relative speedup also increases with sequence length, rising from 1.96\% at 64 tokens to 24.95\% at 16,384 tokens.

\begin{table}[hbp]
\caption{Relative inference speedup achieved by the proposed method across sequence lengths (batch size 4, single A100 GPU).}
\label{table:scale-sl-time}
\vspace{1ex}
\centering
\begin{tabular}{@{}lccccc@{}}
\toprule
Model & 64 & 256 & 1024 & 4096 & 16384 \\ \midrule
BERT  & 1.96\%  & 6.86\% & 10.62\% & 18.45\% & 24.95\% \\
GPT-2 & 4.23\%  & 6.65\% & 8.50\%  & 13.88\% & 21.47\% \\
OPT   & 14.13\% & 5.09\% & 5.17\%  & 8.04\%  & 11.51\% \\ \bottomrule
\end{tabular}
\end{table}

These results show that our method becomes increasingly effective as sequences grow longer, making it particularly valuable for modern long-context applications such as document-level reasoning, code generation, and retrieval-augmented systems. In contrast, as shown in Table~\ref{table:scale-bs-time} and Table~\ref{table:scale-bs-throughput}, the acceleration remains relatively stable across batch sizes at fixed sequence length, confirming complementary scaling behavior along the batch and sequence dimensions.

\paragraph{Model Size}
To examine robustness across substantially different model scales, we further evaluate GPT-J-6B~\citep{wang2021gptj}, a 6-billion-parameter decoder-only Transformer that is approximately $48\times$ larger than GPT-2 (124M parameters) and $17\times$ larger than BERT (340M parameters) in non-embedding parameter count. All experiments are conducted on a single NVIDIA A100-40GB GPU following the same measurement protocol as above.

\begin{table}[htbp]
\caption{Average inference latency of GPT-J-6B on a single A100-40GB GPU. The proposed method achieves 4.2\%--7.7\% latency reduction, with larger gains at longer sequences. OOM denotes out-of-memory.}
\label{tab:gptj-latency}
\vspace{1ex}
\centering
\begin{tabular}{@{}rrccc@{}}
\toprule
Batch Size & Seq Length & Original Inference Time (s) & New Inference Time (s) & Acceleration (\%) \\ \midrule
4  & 64   & 0.2030 & 0.1931 & 4.89 \\
4  & 256  & 0.7840 & 0.7453 & 4.94 \\
4  & 1024 & 3.0976 & 2.8637 & 7.55 \\
8  & 64   & 0.4187 & 0.4010 & 4.22 \\
8  & 256  & 1.5709 & 1.4666 & 6.64 \\
8  & 1024 & 6.2038 & 5.7415 & 7.45 \\
16 & 64   & 0.7793 & 0.7409 & 4.92 \\
16 & 256  & 2.9893 & 2.7601 & 7.67 \\
16 & 1024 & OOM    & OOM    & \textbackslash{} \\ \bottomrule
\end{tabular}
\end{table}

\begin{table}[htbp]
\caption{Inference throughput (tokens/second) of GPT-J-6B under the same settings as Table~\ref{tab:gptj-latency}. The proposed method consistently improves throughput by 53--114 tokens/s, with the largest absolute gains appearing at longer sequences.}
\label{tab:gptj-throughput}
\vspace{1ex}
\centering
\begin{tabular}{@{}rrccc@{}}
\toprule
Batch Size & Seq Length & Original Throughput (tokens/s) & New Throughput (tokens/s) & $\Delta$ \\ \midrule
4  & 64   & 1260.9 & 1325.7 & +64.8  \\
4  & 256  & 1306.1 & 1374.0 & +67.9  \\
4  & 1024 & 1322.3 & 1430.3 & +108.0 \\
8  & 64   & 1223.0 & 1276.8 & +53.8  \\
8  & 256  & 1303.8 & 1396.4 & +92.7  \\
8  & 1024 & 1320.5 & 1426.8 & +106.3 \\
16 & 64   & 1314.1 & 1382.0 & +68.0  \\
16 & 256  & 1370.2 & 1484.0 & +113.8 \\
16 & 1024 & OOM    & OOM    & \textbackslash{} \\ \bottomrule
\end{tabular}
\end{table}

Table~\ref{tab:gptj-latency} and Table~\ref{tab:gptj-throughput} report inference latency and throughput for representative $(\text{batch size}, \text{sequence length})$ configurations. Even on this substantially larger model, our method continues to deliver consistent and meaningful speedup. Relative latency reduction ranges from 4.2\% to 7.7\% across all in-memory configurations. The gains increase with sequence length, following the same ``longer-is-better'' trend observed on smaller models. Absolute throughput improvement reaches +108--114 tokens/s at sequence length 1024, which becomes substantial when amortized over large-scale generation workloads.

These results confirm that the acceleration mechanism remains effective when moving from hundred-million-scale models (GPT-2, BERT, OPT) to multi-billion-parameter models, with no clear sign of diminishing returns.

\begin{table}[htbp]
\caption{Relative speedup comparison between GPT-2 (124M) and GPT-J-6B (6B) on identical hardware and settings (single A100-40GB GPU, batch size 4).}
\label{tab:gptj-comparison}
\vspace{1ex}
\centering
\begin{tabular}{@{}lccc@{}}
\toprule
Model & Parameters & SeqLen=256 Speedup & SeqLen=1024 Speedup \\ \midrule
GPT-2 & 124M & 6.65\% & 8.50\% \\
GPT-J & 6B   & 4.94\% & 7.55\% \\ \bottomrule
\end{tabular}
\end{table}

Moreover, comparing GPT-2 and GPT-J shows that even when the parameter count increases by nearly $50\times$, the relative speedup remains comparable (roughly 5\%--8\%), demonstrating strong scalability with model size.

\section{Empirical Experiments}

Consider a sample $\rvx$, with dimension $d$. We have the equation of LN and RMS\-Norm as follows:
\begin{equation}
\label{eqn:ln-sim}
    \mathrm{LN}(\rvx)=\frac{\rvx-\mu}{\sqrt{\sigma^2+\epsilon}},\ \mathrm{where}\ \mu = \frac{1}{d}\sum_{i=1}^dx_i\ \mathrm{and}\ \frac{1}{d}\sum_{i=1}^d(x_i - \mu)^2.
\end{equation}

\begin{equation}
    \label{eqn:rms-sim}
    \mathrm{RMS}(\rvx)=\frac{\rvx}{\sqrt{\sigma_{rms}^2+\epsilon}},\ \mathrm{where}\ \sigma_{rms}^2 = \frac{1}{d}\sum_{i=1}^d x_i ^2.
\end{equation}

It is worth mentioning that we compare custom-PyTorch module implementations of RMS\-Norm and LayerNorm (implemented according to Eqn.\ref{eqn:ln-sim} and Eqn.\ref{eqn:rms-sim}), both without CUDA kernels to isolate algorithmic speedup.

\subsection{Text Classification Task}

\begin{figure}[hbp]
    \centering
    \begin{subfigure}[t]{0.48\textwidth}
        \centering
        \includegraphics[height=25ex]{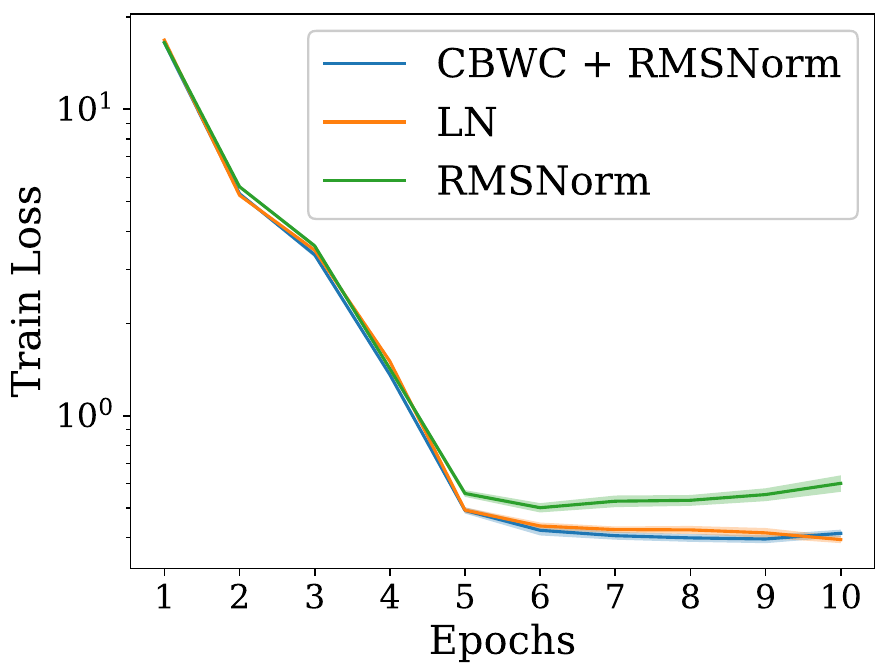}
        \caption{Training loss.}
    \end{subfigure}   
    \begin{subfigure}[t]{0.30\textwidth}
        \centering
        \includegraphics[height=25ex]{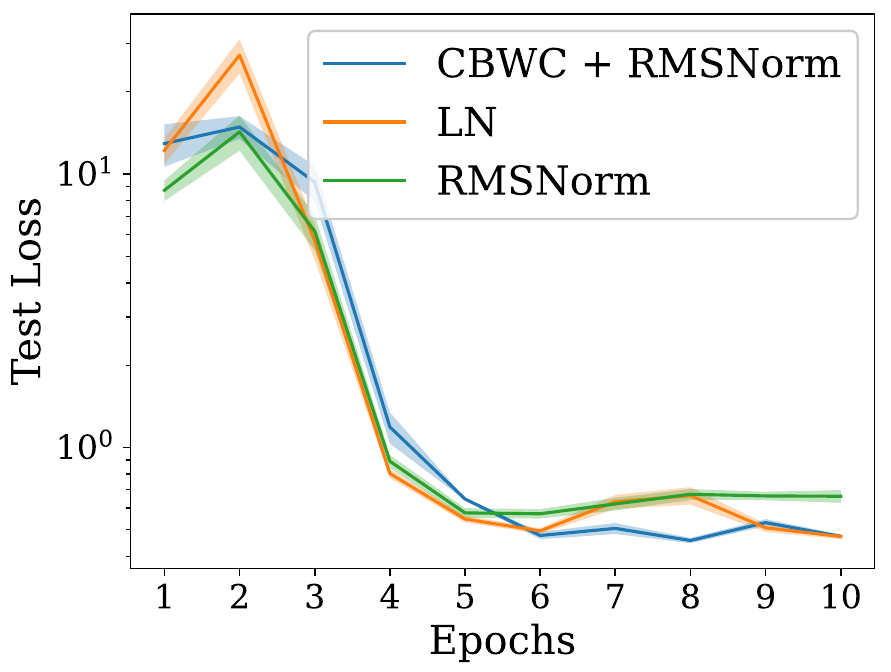}
        \caption{Test loss.}
    \end{subfigure}   
    \vspace{-1ex}
    \caption{Performance of transformer models for text classification task of Transformer on the AG News dataset. The results are averaged over 5 random seeds with shaded regions indicating standard deviation}
    \label{fig:acc-text-classification}
    \vspace{-2ex}
\end{figure}

The experiment is conducted on a 3090Ti. The average results are reported in Table~\ref{tab:text-classification-performace} and Figure~\ref{fig:acc-text-classification}. The best results are highlighted in bold, and the worst are shown in gray.

\begin{table}[htp]
  \centering
  \caption{Performance of Transformer on text classification.}
  \label{tab:text-classification-performace}
  \begin{tabular}{@{}lcc@{}}
  \toprule
  Method   & Test Loss                     & Test Acc (\%)                      \\ \midrule
  LN       & $\mathbf{0.472 \pm 0.010}$   & $\mathbf{85.21 \pm 0.35}$  \\
  RMS      & \textcolor[gray]{0.6}{$0.663 \pm 0.037$} & \textcolor[gray]{0.6}{$76.57 \pm 1.77$} \\
  CBWC+RMS & $0.473 \pm 0.007$            & $85.12 \pm 0.23$           \\ \bottomrule
  \end{tabular}
  \vspace{-2ex}
\end{table}

\subsection{Image Classification}
\label{apx:swin}

For the Imagenet100, we select 100 classes from Imagenet1k~\citep{deng2009imagenet} according to the given classes in~\citep{tian2019contrastive}. 
We chose SWIN-T for this experiment and trained on a single 3090. 
We apply the AdamW optimizer with a learning rate of $10^{-4}$ and a batch size of 128.
Here we list the top 1 and top 5 accuracy and loss for both test and training in Table~\ref{swin}. We have the best results in bold and the worst results in gray.

\begin{table}[thpb]
\vspace{-3ex}
\caption{Average training results (mean $\pm$ std) under 3 random seeds for SWIN on ImageNet100.}
\label{swin}
\centering
\begin{tabular}{lccc}
\toprule
Model & TrainAcc@1 (\%) & TrainAcc@5 (\%) & TrainLoss \\
\midrule
LN       & \(\mathbf{99.397 \pm 0.003}\) & \(99.919 \pm 0.010\) & \(\mathbf{0.0254 \pm 0.0002}\) \\
RMS      & \(99.376 \pm 0.004\) & \(99.913 \pm 0.001\) & \(0.0258 \pm 0.0004\) \\
CBWC+RMS & \(99.382 \pm 0.032\) & \(\mathbf{99.922 \pm 0.005}\) & \(0.0257 \pm 0.0009\) \\
\bottomrule
\end{tabular}
\end{table}

We use \texttt{time.time()} to trace the time usage. This strategy is consistent with practices in prior efficiency-focused works \cite{jiang2023pre}, where authors also compared unoptimized module versions when no fast official RMS\-Norm existed.

Moreover, to verify the training stability of our method, we applied different learning rates on the models. We trained 70 epochs for each group.

\begin{figure}[hbp]
    \vspace{-1ex}
    \centering
    \begin{subfigure}[t]{0.48\textwidth}
        \centering
        \includegraphics[height=25ex]{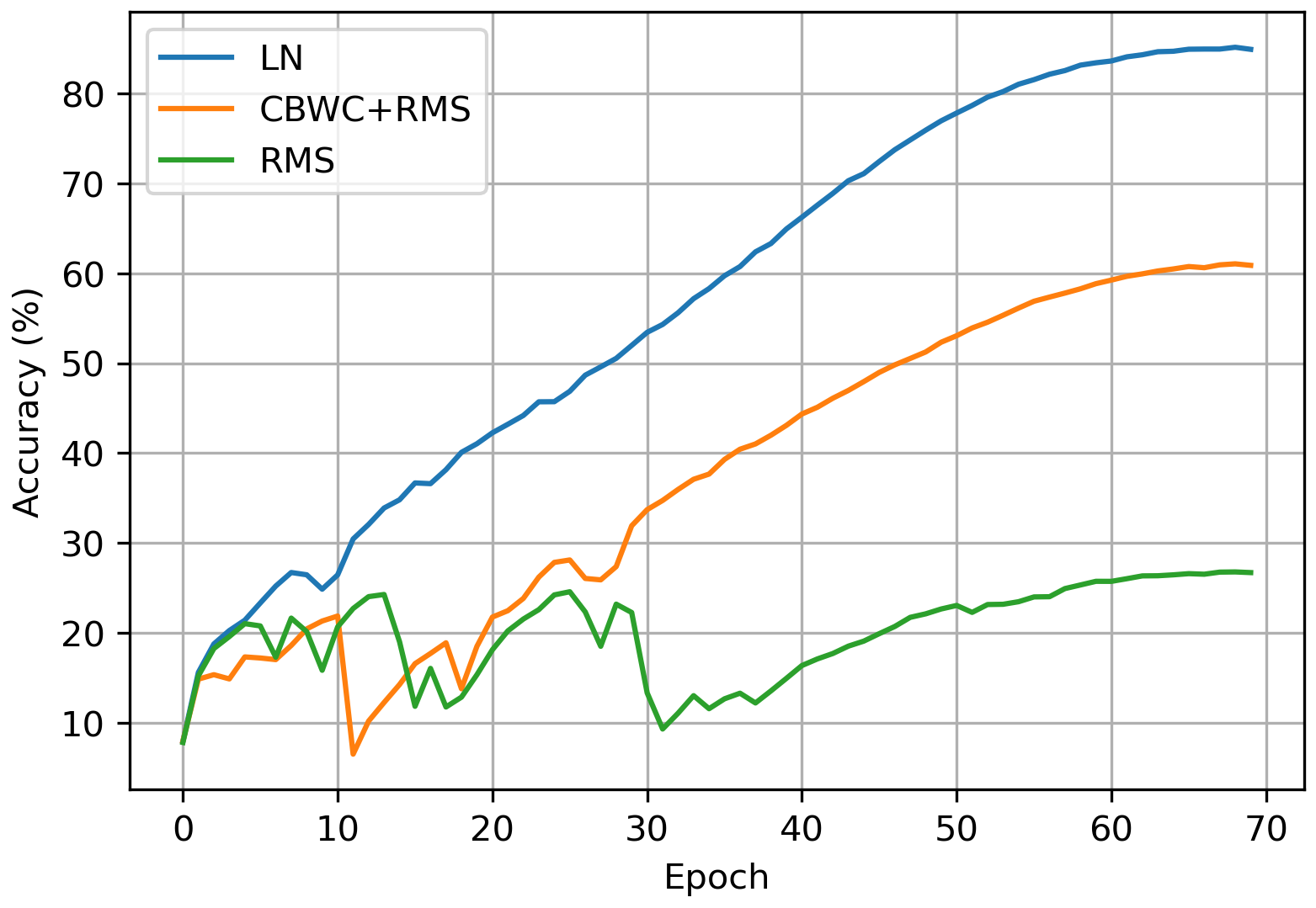}
        \caption{Train Accuracy.}
    \end{subfigure}   
    \begin{subfigure}[t]{0.30\textwidth}
        \centering
        \includegraphics[height=25ex]{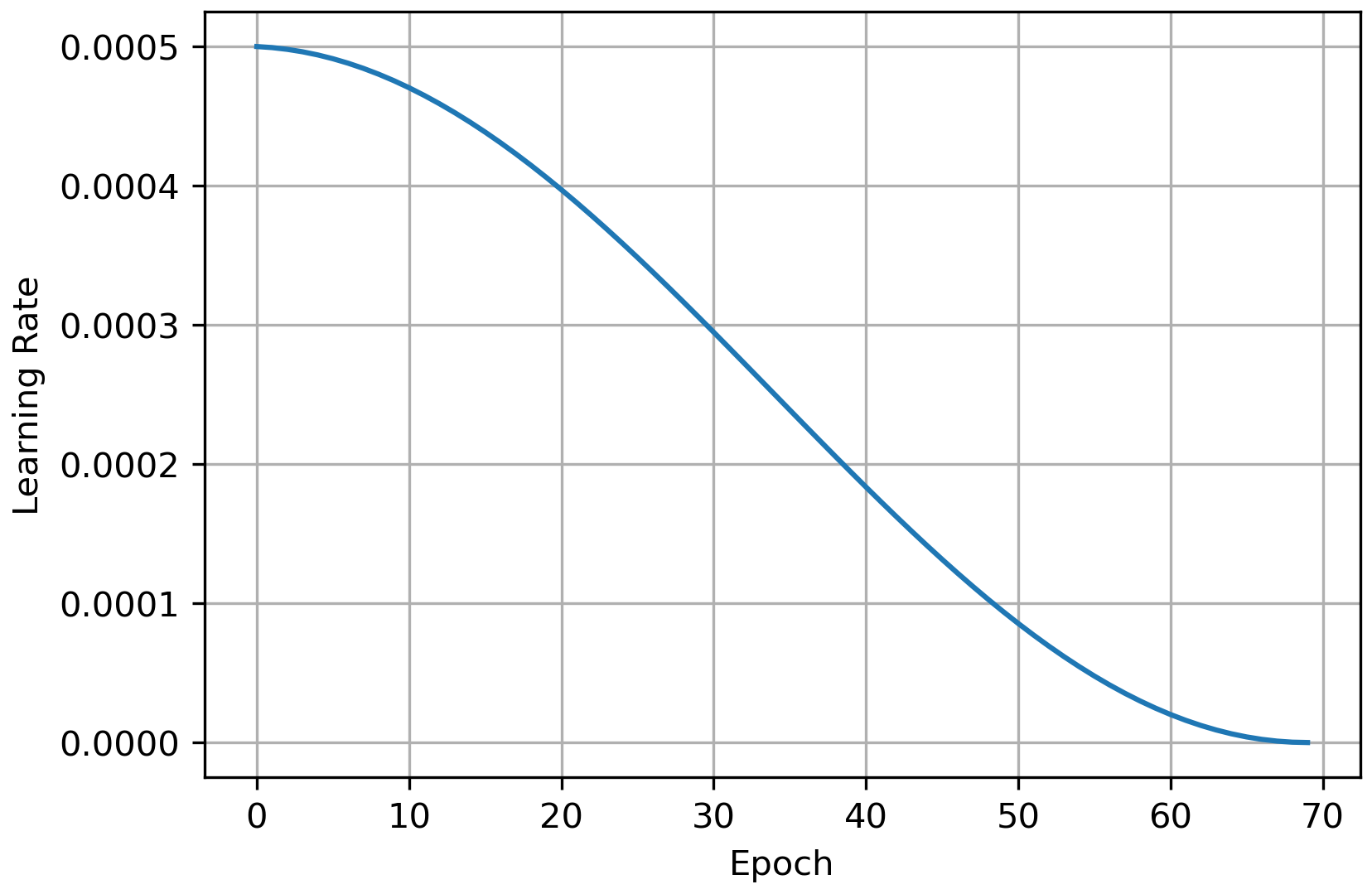}
        \caption{Learning Rate.}
    \end{subfigure}   
\vspace{-1ex}
\caption{Train accuracy of three models and learning rate setting of SWIN transformer on Imagenet100 under high learning rate setting.}
\label{fig:stability-high}
\end{figure}

Under a higher learning rate ($10^{-3}$), CBWC+RMS performs between the LN and RMSNorm variants. As shown in Figure~\ref{fig:stability-high}, with the decline of the learning rate, our method shows better stability as it can adapt to a higher learning rate. We also notice a more inclined accuracy curve, which is the evidence of faster convergence.

\begin{figure}[hbp]
    \centering   
    \begin{subfigure}[t]{0.48\textwidth}
        \centering
        \includegraphics[height=25ex]{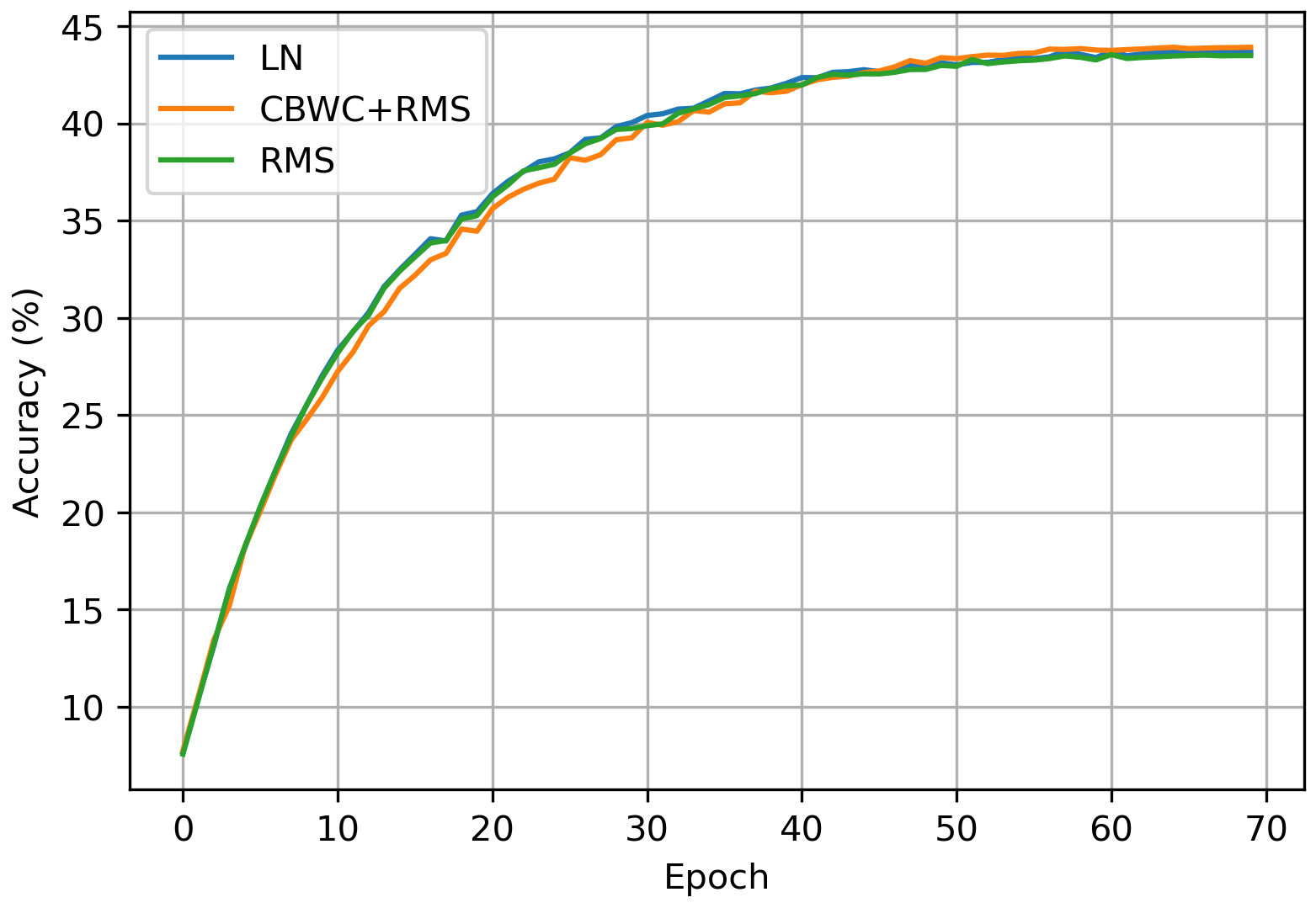}
        \caption{Test Accuracy.}
            \label{fig:stable-small-test-acc}
    \end{subfigure}   
    \begin{subfigure}[t]{0.30\textwidth}
        \centering
        \includegraphics[height=25ex]{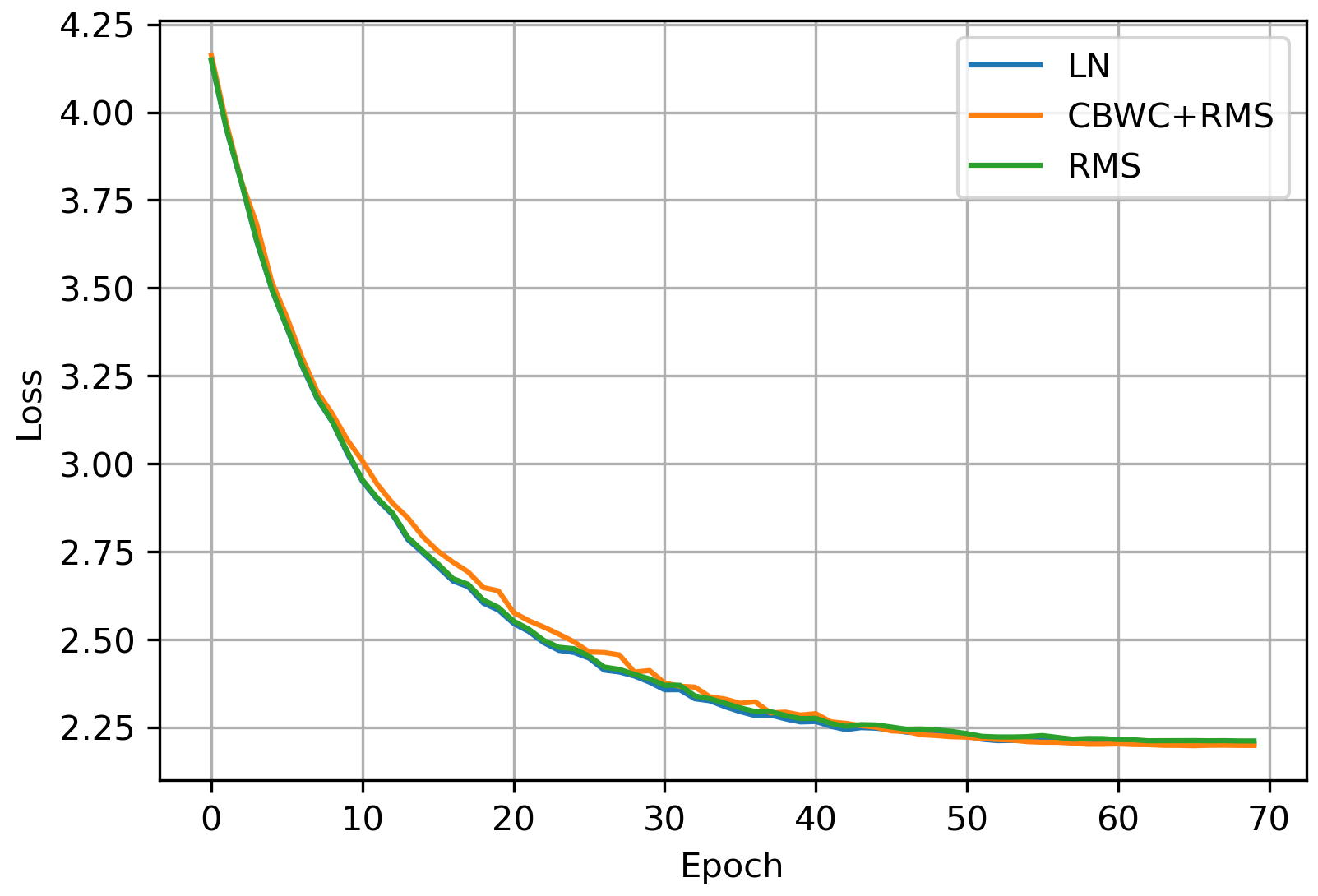}
        \caption{Test Loss.}
                    \label{fig:stable-small-test-loss}
    \end{subfigure}
    \vspace{-1ex}
    \caption{Performance of the three models in SWIN transformer on Imagenet100 under small learning rate.}
\end{figure}

Under a smaller learning rate ($10^{-5}$), the accuracy and loss of LN variant and RMS\-Norm  are almost the same, while our method has different performance. Our method shows slightly better generalization performance than the variants using only RMSNorm or LN, with slightly higher test accuracy and lower test loss (as shown in Figure~\ref{fig:stable-small-test-acc} and Figure~\ref{fig:stable-small-test-loss}).
\subsection{Text Generation}
We also conduct a text generation experiment using GPT-3, and report the results in Table~\ref{table:performace-gpt2}.

\begin{table}[htbp]
\centering
\caption{\centering Performance of GPT-3 on WikiText-103.}
\label{table:performace-gpt2}
\begin{tabular}{lcccc}
\toprule
Method   & Train Loss    & Train PPL  & Test Loss & Test PPL   \\ \hline
LN       & 6.24          & 513        & 6.0598    & 428.30   \\
RMS      & \textcolor[gray]{0.6}{6.32}          & \textcolor[gray]{0.6}{557} & \textcolor[gray]{0.6}{6.1311}    & \textcolor[gray]{0.6}{459.94}          \\
CBWC+RMS & \textbf{6.18} & \textbf{484}  & \textbf{6.0144}    & \textbf{409.28} \\ \bottomrule
\end{tabular}
\end{table}

\subsection{Verification Experiment for Fine-tuning on Pre-trained Model}

\subsubsection{Experiment Setting for MLP}
\label{sec:mlp-verify}
We verify that the CBWC+RMS\-Norm and original LN training scheme are identical in engineering by the following experiment.

We perform a classification task on CIFAR-10 \cite{Krizhevsky2009CIFAR10}. The MLPs have a depth of 6 and a width of 256. We train the models for 40 epochs with a learning rate of 0.01 and a batch size of 256, using a constant seed. The proxy parameter $W_A$ and the original weight matrix $W_B$ differ by less than $10^{-5}$. As shown in Figure~\ref{fig:comparison-of-w}, the difference between the two parameter sets is negligible and can be attributed to numerical error. We therefore conclude that the two models follow indistinguishable optimization processes and can be inter-converted at any point during training without affecting the outcome.

\begin{figure}[h]
    \vspace{-1ex}
    \centering
    \begin{subfigure}[t]{0.30\textwidth}
        \centering
         \includegraphics[height=25.25ex]{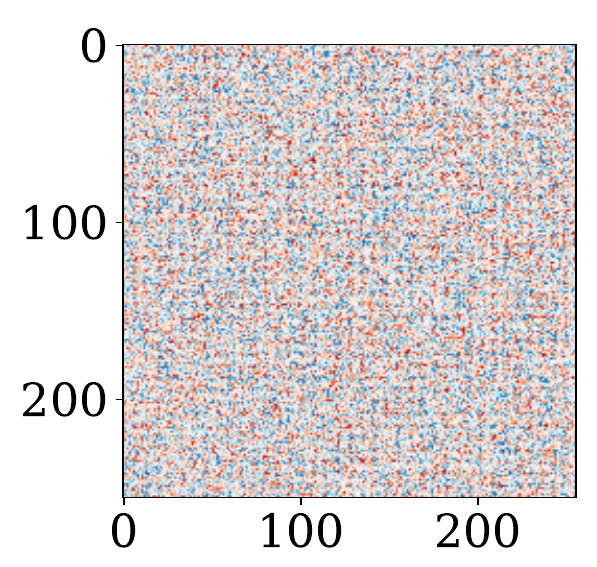}
        \caption{$\mW_A$}
    \end{subfigure}   
    \begin{subfigure}[t]{0.30\textwidth}
        \centering
        \includegraphics[height=25ex]{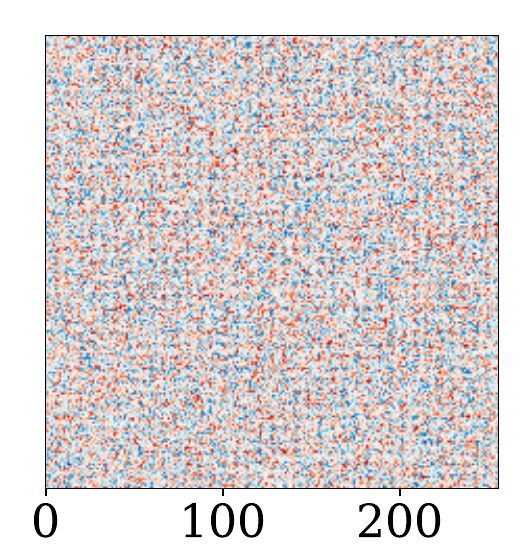}
        \caption{$\mW_B$}
    \end{subfigure}   
    \begin{subfigure}[t]{0.30\textwidth}
        \centering
        \includegraphics[height=25ex]{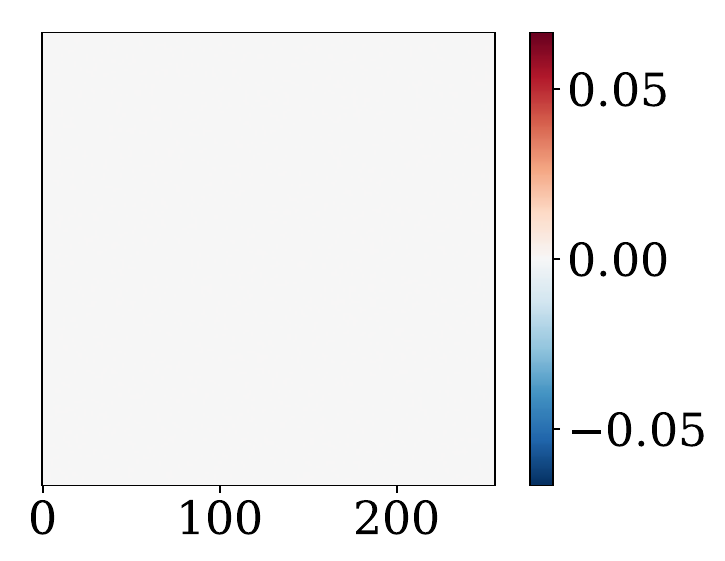}
        \caption{$\mW_A-\mW_B$}
    \end{subfigure}   
    \vspace{-1ex}
    \caption{Comparison of $\mW_A$ and $\mW_B$.}
    \label{fig:comparison-of-w}
    \vspace{-2ex}
\end{figure}

\subsubsection{Experiment Setting for Transformer}
\label{app:expcom}
We pre-train the model with a learning rate of $5\times10^{-4}$ and an effective batch size of 128 (batch size is 2 and gradient accumulation is 64). 
For fine-tuning, we have a learning rate of $5\times10^{-5}$ and an effective batch size of 16 (batch size is 2 and gradient accumulation is 8).

\section{Open-source Code}
To facilitate verification and further experimentation, we release the complete open-source implementation, including the fully automatic folding tool and optimized RMS\-Norm kernels: \url{https://github.com/BobYue-01/Enjoy-LN}.

\end{document}